%% file: structure.tex
\documentclass{article}
\usepackage{etoolbox}       %

\usepackage{graphicx}  %

\usepackage[parfill]{parskip}
\usepackage{amsmath,amsthm}
\usepackage{mathtools}  %
\usepackage[dvipsnames]{xcolor}         %

\newtoggle{arxiv}
\toggletrue{arxiv}

  \usepackage[parfill]{parskip}
  \usepackage[citestyle=authoryear-comp,sorting=nyt,maxbibnames=99,backend=biber]{biblatex}
  \newcommand{\citep}{\parencite}
  \newcommand{\citet}{\textcite}
  \newlength{\defbaselineskip}
  \RequirePackage[T1]{fontenc}
  \RequirePackage[tt=false, type1=true]{libertine}
  \RequirePackage[varqu]{zi4}
  \RequirePackage[libertine]{newtxmath}

\usepackage{bm}

\usepackage[inline]{enumitem}
\usepackage{booktabs}
\usepackage[dvipsnames]{xcolor}  %
\usepackage{subcaption}
\usepackage{multirow}
\usepackage{wrapfig}
\usepackage{algorithm,algorithmicx,algpseudocode}
\usepackage{nicematrix}

\usepackage[frozencache,cachedir=.]{minted} %

\usepackage[colorlinks=true,linkcolor=blue,citecolor=magenta,urlcolor=red]{hyperref}
\usepackage[capitalise,noabbrev]{cleveref}

\usepackage{import}

\newtheorem{theorem}{Theorem}[section]  %
\newtheorem{corollary}[theorem]{Corollary}  %

\newtheorem{lemma}[theorem]{Lemma}
\newtheorem{proposition}[theorem]{Proposition}
\newtheorem{definition}[theorem]{Definition}

\newtheorem{remark}{Remark}

\usepackage{pifont}%
\newcommand{\cmark}{\ding{51}}%
\newcommand{\xmark}{\ding{55}}%

\newcommand{\R}{\mathbb{R}}
\newcommand{\dt}{\Delta}

\newcommand{\para}[1]{\iftoggle{arxiv}{\paragraph{#1}}{\textbf{#1}}}

\newcommand*\samethanks[1][\value{footnote}]{\footnotemark[#1]}

  \title{Transformers are SSMs: Generalized Models and Efficient Algorithms \\ Through Structured State Space Duality}
  \usepackage{authblk}
  \author[$^1$]{Tri Dao\thanks{Alphabetical by last name.}}
  \author[$^2$]{Albert Gu\samethanks}
  \affil[$^1$]{Department of Computer Science, Princeton University}
  \affil[$^2$]{Machine Learning Department, Carnegie Mellon University}
  \affil[ ]{{\texttt{tri@tridao.me}}, {\texttt{agu@cs.cmu.edu}}}
  \date{}

\begin{document}

  \maketitle

\input{structure/abstract}
\input{structure/intro}

\input{structure/background}
\input{structure/ssm}

\input{structure/attention}

\input{structure/ssd}

\input{structure/efficient}
\input{structure/architecture}
\input{structure/systems}

\input{structure/experiments}
\input{structure/related}

\input{structure/conclusion}

\subsubsection*{Acknowledgments}
We thank Angela Wu for the suggestion on how to efficiently compute
the gradient of $\Delta$ in a numerically stable manner.
We thank Sukjun Hwang and Aakash Lahoti for assistance with the MQAR experiments.

\printbibliography

\newpage

\appendix

\onecolumn

  \input{structure/glossary}
\input{structure/scan}

\input{structure/theory_details}
\input{structure/experiment_details}

\end{document}

%% file: structure/abstract.tex
\begin{abstract}

  \noindent
  While Transformers have been the main architecture behind deep learning's
  success in language modeling, state-space models (SSMs) such as Mamba have recently been
  shown to match or outperform Transformers at small to medium scale.
  We show that these families of models are actually quite closely related,
  and develop a rich framework of theoretical connections between SSMs and variants of
  attention,
  connected through various decompositions of a well-studied class of structured \emph{semiseparable matrices}.
  Our state space duality (SSD) framework 
  allows us to design a new architecture (\textbf{Mamba-2}) whose core layer is an a refinement of Mamba's selective SSM
  that is 2-8$\times$ faster, while continuing to be competitive with
  Transformers on language modeling.

\end{abstract}

%% file: structure/intro.tex
\section{Introduction}
\label{sec:intro}

Transformers, in particular decoder-only models (e.g.\ GPT~\citep{brown2020language}, Llama~\citep{touvron2023llama}) which process input sequences in a causal fashion, are one of the main drivers of modern deep learning's success.
Numerous approaches attempt to approximate the core attention layer to address its efficiency issues~\citep{tay2022efficient}, such as scaling quadratically in sequence length during training and requiring a cache of size linear in sequence length during autoregressive generation.
In parallel, a class of alternative sequence models, structured state-space models (SSMs), have emerged with linear scaling in sequence length during training and constant state size during generation.
They show strong performance on long-range tasks (e.g. S4~\citep{gu2022efficiently}) and recently matched or beat Transformers on language modeling (e.g. Mamba \citep{gu2023mamba}) at small to moderate scale.
However, the development of SSMs have appeared disjoint from the community's collective effort to improve Transformers, such as understanding them theoretically as well as optimizing them on modern hardware.
As a result, it is more difficult to understand and experiment with SSMs compared to Transformers, and it remains challenging to train SSMs as efficiently as Transformers from both an algorithmic and systems perspective.

Our main goal is to develop a rich body of theoretical connections between structured SSMs and variants of attention.
This will allow us to transfer algorithmic and systems optimizations originally developed for Transformers to SSMs, towards the goal of building foundation models that perform better than Transformers while scaling more efficiently in sequence length.
A milestone contribution in this direction was the \textbf{Linear Attention (LA)} framework \citep{katharopoulos2020transformers},
which derived a connection between autoregressive attention and linear RNNs
by showing the equivalence between ``dual forms'' of quadratic kernelized attention and a particular linear recurrence.
This duality allows new capabilities such as the ability to have both efficient parallelizable training and efficient autoregressive inference.
In the same spirit, this paper provides multiple viewpoints connecting linear-complexity SSMs with quadratic-complexity forms to combine the strengths of SSMs and attention.%
\footnote{Technically speaking, these connections only relate to certain flavors of attention; the title of this paper is an homage to \citet{katharopoulos2020transformers} which first showed that ``Transformers are RNNs''.}

\iftoggle{arxiv}{
\begin{wrapfigure}{R}{0.48\linewidth}
  \begin{center}
    \includegraphics[width=\linewidth]{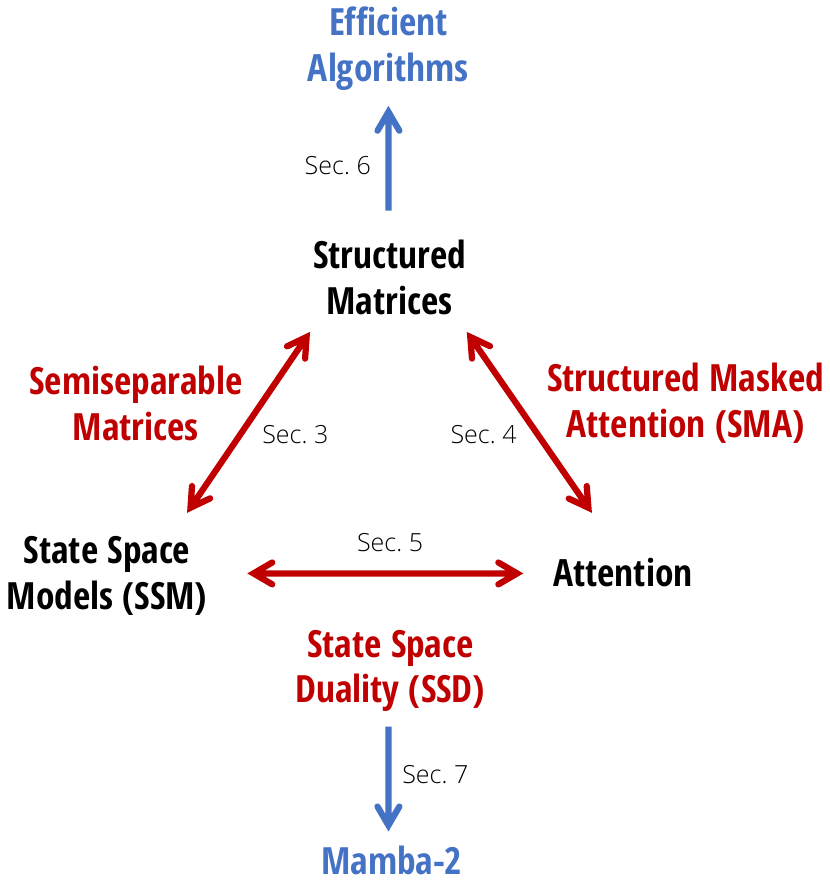}
  \end{center}
  \caption{
    (\textbf{Structured State-Space Duality}.)
    This paper fleshes out the relationship between state space models and attention through the bridge of structured matrices.
  }
  \label{fig:roadmap}
\end{wrapfigure}
}{}

\para{State Space Duality.}
Our framework connecting structured SSMs and variants of attention, which we call \textbf{structured state space duality} (SSD),
is made through the abstractions of \textbf{structured matrices}:
matrices with subquadratic parameters and multiplication complexity.
We develop two broad frameworks for representing sequence models, one as matrix transformations and one as tensor contractions, which each reveal different perspectives of the duality.
Our technical contributions include:
\begin{itemize}[leftmargin=*,itemsep=0pt,topsep=0pt]
  \item We show an equivalence between state space models and a well-studied family of structured matrices called \textbf{semiseparable matrices}\iftoggle{arxiv}{ (\cref{sec:ssm})}{}.
    This connection is at the heart our framework, revealing new properties and algorithms for SSMs. A central message of this paper is that \emph{different methods of computing state space models can be reframed as various matrix multiplication algorithms on structured matrices}.
  \item We significantly improve the theory of linear attention~\citep{katharopoulos2020transformers}.
    We first provide an incisive proof of its recurrent form through the language of tensor contractions, and then generalize it to a new family of \textbf{structured masked attention (SMA)}\iftoggle{arxiv}{ (\cref{sec:attention})}{}.
  \item We connect SSMs and SMA, showing that they have a large intersection that are duals of each other, possessing both SSM-like linear and attention-like quadratic forms\iftoggle{arxiv}{ (\cref{sec:ssd})}{}.
    \iftoggle{arxiv}{We also prove that any kernel attention method possessing a fast recurrent form must be an SSM.}{}
\end{itemize}

Beyond its intrinsic theoretical value, our framework opens up a broad set of directions for understanding and improving sequence models.

\para{Efficient Algorithms.}
First and most importantly, our framework exposes new efficient and easily-implementable algorithms for computing SSMs\iftoggle{arxiv}{ (\cref{sec:efficient})}{}.
We introduce a new \textbf{SSD algorithm}, based on block decompositions of semiseparable matrices, that takes advantage of both the linear SSM recurrence and quadratic dual form, obtaining optimal tradeoffs on all main efficiency axes (e.g. training and inference compute, memory usage, and ability to leverage matrix multiplication units on modern hardware).
A dedicated implementation of SSD is $2-8\times$ faster than the optimized selective scan implementation of Mamba, while simultaneously allowing for much larger recurrent state sizes ($8\times$ the size of Mamba or even higher, with minimal slowdown).
SSD is highly competitive with optimized implementations of softmax attention (FlashAttention-2~\citep{dao2023flashattention2}), crossing over at sequence length 2K and 6$\times$ faster at sequence length 16K.

\iftoggle{arxiv}{
\para{Architecture Design.}
One major obstacle to adopting new architectures such as SSMs is the ecosystem tailored to Transformers, such as hardware-efficient optimization and parallelism techniques for large-scale training.
Our framework allows using established conventions and techniques for attention to build a vocabulary of architecture design choices for SSMs, and further improve them (\cref{sec:architecture}).
For example, we introduce the analog of heads from multi-head attention (MHA) to SSMs.
We show that the Mamba architecture is a \textbf{multi-input SSM (MIS)} that turns out to be analogous to \textbf{multi-value attention (MVA)}, and compare other variants of Mamba with different head structures.

We also use these ideas to make slight modifications to the Mamba block, which allows tensor parallelism to be implemented (e.g. in the style of Megatron~\citep{shoeybi2019megatron}).
The main ideas include introducing grouped-value attention (GVA) head structure, and moving all data-dependent projections to occur in parallel at the beginning of the block.

}{
  \para{Mamba-2.}
  Additionally, inspired by the connection between SSMs and Transformers, we slightly modify the neural network architecture of Mamba by moving all data-dependent projections to occur in parallel at the beginning of the block. %
}
The combination of the modified parallel Mamba block, together with using SSD as the inner SSM layer, results in the \textbf{Mamba-2} architecture.
We investigate Chinchilla scaling laws for Mamba-2 in the same setting as Mamba, finding that it Pareto dominates Mamba and Transformer++ in both perplexity and wall-clock time.
We additionally train a family of Mamba-2 models at varying sizes on the Pile, showing that it matches or outperforms Mamba and open source Transformers on standard downstream evaluations.
For example, Mamba-2 with 2.7B parameters trained on 300B tokens on the Pile outperforms Mamba-2.8B, Pythia-2.8B and even Pythia-6.9B trained on the same dataset.

\iftoggle{arxiv}{
\paragraph{Systems Optimizations.}
The SSD framework connects SSMs and Transformers, allowing us to leverage a rich body of work on systems optimizations developed for Transformers~(\cref{sec:systems}).
\begin{itemize}[leftmargin=*,itemsep=0pt,topsep=0pt]
  \item For example, Tensor Parallelism (TP) is an important model parallelism technique to train large Transformer models by splitting each layer across GPUs on the same node.
    We design Mamba-2 to be TP-friendly, reducing the number of synchronization point per block by half.
  \item For very long sequences whose activations do not fit on one device, sequence parallelism has been developed for the attention blocks.
    We describe how to train SSMs in general and Mamba-2 in particular with sequence parallelism, by passing the recurrent states between devices.
  \item For finetuning with examples of different lengths, for best efficiency, Transformer requires sophisticated techniques to remove padding tokens and perform attention on variable length sequences.
    We show how Mamba-2 can be trained with variable sequence lengths efficiently, requiring no padding tokens.
\end{itemize}
}{}

\cref{sec:experiments} empirically validates Mamba-2 on language modeling, training efficiency, and a difficult multi-query associative recall task~\citep{arora2024simple}.
Finally, in \cref{sec:related}, we provide an extended related work and discuss potential research directions opened up by our framework.

Model code and pre-trained checkpoints are open-sourced at \url{https://github.com/state-spaces/mamba}.

%% file: structure/background.tex
\section{Background and Overview}
\label{sec:background}

\subsection{Structured State Space Models}

Structured state space sequence models (S4) are a recent class of sequence models for deep learning that are broadly related to RNNs, CNNs,
and classical state space models.
They are inspired by a particular continuous system \eqref{eq:ssm}
that maps a 1-dimensional sequence $x \in \R^\mathtt{T} \mapsto y \in \R^\mathtt{T}$ through an implicit latent state \( h \in \R^{\mathtt{(T, N)}} \). %
\iftoggle{arxiv}{

}{}
A general discrete form of structured SSMs takes the form of equation \eqref{eq:ssm}.
\begin{center}
  \vspace*{-1em}
\begin{minipage}[t]{.45\linewidth}
\begin{subequations}
  \label{eq:ssm}
  \begin{align}
  \label{eq:ssm:1}
    h_{t} &= A h_{t-1} + B x_t \\
  \label{eq:ssm:2}
    y_t &= C^{\top} h_t
  \end{align}
\end{subequations}
\end{minipage}
\qquad
\begin{minipage}[t]{.45\linewidth}
\begin{subequations}
  \label{eq:s6}
  \begin{align}
  \label{eq:s6:1}
    h_{t} &= A_t h_{t-1} + B_t x_t \\
  \label{eq:s6:2}
    y_t &= C_t^{\top} h_t
  \end{align}
\end{subequations}
\end{minipage}
\end{center}

where $A \in \R^{\mathtt{(N,N)}}, B \in \R^{\mathtt{(N,1)}}, C \in \R^{\mathtt{(N,1)}}$.
Structured SSMs are so named because the $A$ matrix controlling the temporal dynamics must be \emph{structured} in order to compute this sequence-to-sequence transformation efficiently enough to be used in deep neural networks.
The original structures introduced were diagonal plus low-rank (DPLR) \citep{gu2022efficiently} and diagonal~\citep{gupta2022diagonal,gu2022parameterization,smith2023s5}, which remains the most popular structure.

In this work, we use the term state space model (SSM) to refer to structured SSMs.
There are many flavors of such SSMs, with deep ties to several major paradigms of neural sequence models such as continuous-time, recurrent, and convolutional models~\citep{gu2021combining}.
\iftoggle{arxiv}{
  We provide a brief overview below, and refer to prior work for more context and details~\citep{gu2023thesis,gu2023mamba}.

\para{Continuous-time Models.}
The original structured SSMs originated as continuous-time maps on functions $x(t) \in \R \mapsto y(t) \in \R$, rather than operating directly on sequences. %
In the continuous-time perspective, in equation \eqref{eq:ssm:1} the matrices $(A, B)$ are not directly learned but generated from underlying
parameters $(\mathring{A}, \mathring{B})$, along with a parameterized step size $\dt$.
The ``continuous parameters'' $(\dt, \mathring{A}, \mathring{B})$ are converted to ``discrete parameters'' $(A, B)$ through fixed formulas $A = f_A(\dt, \mathring{A})$ and $B = f_B(\dt, \mathring{B})$,
where the pair $(f_A, f_B)$ is called a \emph{discretization rule}.

\begin{remark}
  While our main models adopt the same parameterization and discretization step as prior work (see \citet{gu2023mamba} for details), for simplifying exposition and notation we omit it in the rest of this paper.
  We note that prior work on structured SSMs referred to the continuous parameters $(\mathring{A}, \mathring{B})$ and discrete parameters $(A, B)$ as
  $(A, B)$ and $(\bar{A}, \bar{B})$ instead; we have changed notation to simplify the presentation and focus directly on the discrete parameters, which govern the main SSM recurrence.
\end{remark}

\para{Recurrent Models.}

Equations \eqref{eq:ssm} and \eqref{eq:s6} take the form of a recurrence which is linear in its input $x$.
Structured SSMs can therefore be viewed as types of recurrent neural networks (RNNs), where the linearity endows them with additional properties and
allows them to avoid the sequential computation of traditional RNNs.
Conversely, despite this simplification, SSMs are still fully expressive as sequence transformations (in the sense of universal approximation)~\citep{kaul2020linear,orvieto2023resurrecting,wang2023state}.

\para{Convolutional Models.}
When the SSM's dynamics are constant through time as in equation \eqref{eq:ssm}, the model is called \textbf{linear time-invariant (LTI)}.
In this case, they are equivalent to convolutions.
Thus, SSMs can also be viewed as types of CNNs, but where (i) the convolution kernels are implicitly parameterized through the SSM parameters $(A, B, C)$ and
(ii) the convolution kernels are generally global instead of local.
Conversely, through classical signal processing theory all sufficiently well-behaved convolutions can be represented as SSMs.

Commonly, previous LTI SSMs would use the convolutional mode for efficient parallelizable training (where the whole input sequence is seen ahead of time), %
and switched into recurrent mode \eqref{eq:ssm} for efficient autoregressive inference (where the inputs are seen one step at a time).

}{}

\para{Selective State Space Models.}
The form \eqref{eq:s6} where the parameters $(A, B, C)$ can also vary in time was introduced in Mamba as the \textbf{selective SSM}.
Compared to the standard LTI formulation \eqref{eq:ssm}, this model can selectively choose to focus on or ignore inputs at every timestep.
It was shown to perform much better than LTI SSMs on information-dense data such as language,
especially as its state size $\mathtt{N}$ increases allowing for more information capacity.
However, it can only be computed in recurrent instead of convolutional mode, and requires a careful hardware-aware implementation to be efficient.
Even so, it is still less efficient than hardware-friendly models such as CNNs and Transformers because it does not leverage matrix multiplication units, which modern accelerators such as GPUs and TPUs are specialized for.

While \emph{time-invariant} SSMs are closely related to continuous, recurrent, and convolutional sequence models,
they are not directly related to attention.
In this paper, we show a deeper relationship between \emph{selective} SSMs and attention,
and use it to significantly improve the training speed of SSMs while simultaneously allowing for much larger state sizes $\mathtt{N}$.

\para{Structured SSMs as Sequence Transformations.}

\begin{definition}
  \label{def:sequence-transformation}
We use the term \textbf{sequence transformation} to refer to a parameterized map on sequences $Y = f_{\theta}(X)$ where $X, Y \in \R^{\mathtt{(T,P)}}$ and $\theta$ is an arbitrary collection of parameters.
$\mathtt{T}$ represents the sequence or \emph{time} axis; subscripts index into the first dimension, e.g.\ $X_t, Y_t \in \R^\mathtt{P}$.
\end{definition}
Sequence transformations (e.g.\ SSMs, or self-attention) are the cornerstone of deep sequence models, where they are incorporated into neural network architectures (e.g.\ Transformers).
The SSM in \eqref{eq:ssm} or \eqref{eq:s6} is a sequence transformation with $\mathtt{P}=1$; it can be generalized to $\mathtt{P} > 1$ by simply broadcasting across this dimension (in other words, viewing the input as $\mathtt{P}$ independent sequences and applying the SSM to each).
One can think of $\mathtt{P}$ as a \textbf{head dimension}\iftoggle{arxiv}{, which we will elaborate on in \cref{sec:architecture}}{}.

\begin{definition}
  \label{def:ssm}
  We define the \textbf{SSM operator}
    $\mathsf{SSM}(A, B, C) = \mathsf{SSM}(A_{0:T}, B_{0:T}, C_{0:T})$ as the sequence transformation $X \in \R^{\mathtt{(T,P)}} \mapsto Y \in \R^{\mathtt{(T,P)}}$
    defined by equation \eqref{eq:s6}.
\end{definition}

In SSMs, the $\mathtt{N}$ dimension is a free parameter called the \textbf{state size} or state dimension.
We also call it the \textbf{state expansion factor}, because it expands the size of the input/output by a factor of $N$, with implications for the computational efficiency of these models.

Finally, we remark that many types of sequence transformations, such as attention, can be represented as a single matrix multiplication across the sequence dimension.
\begin{definition}
  \label{def:matrix-transformation}
  We call a sequence transformation $Y = f_\theta(X)$ a \textbf{matrix transformation} if it can be written in the form $Y = M_\theta X$ where $M$ is a matrix depending on the parameters $\theta$.
  We identify the sequence transformation with the matrix $M$, and often drop the dependence on $\theta$ when clear from context.
\end{definition}

\subsection{Attention}
\label{sec:overview:attention}

Attention broadly refers to a type of computation that assigns scores to every pair of positions in a sequence, allowing each element to ``attend'' to the rest.
By far the most common and important variant of attention is softmax self-attention, which can be defined as
\begin{align*}%
  Y = \operatorname*{softmax}(QK^{\top}) \cdot V
\end{align*}
for $Q, K, V \in \R^{\mathtt{(T,P)}}$.
The mechanism of pairwise comparisons (induced by materializing $QK^{\top}$) leads to the characteristic quadratic training cost of attention.

Many variants of attention have been proposed, but all share the underlying core of these attention scores, with various approximations~\citep{tay2022efficient}.
The most important variant for this work is \textbf{linear attention}~\citep{katharopoulos2020transformers}.
Roughly speaking, this family of methods drops the softmax by folding it into a kernel feature map, and uses associativity of matrix multiplication to rewrite
$(QK^{\top}) \cdot V = Q \cdot (K^{\top} V)$.
Moreover, in the important case of causal (autoregressive) attention,
they show that when the causal mask is incorporated into the left-hand side as
$(L \circ QK^{\top}) \cdot V$, where $L$ is the lower-triangular 1's matrix,
then the right-hand side can be expanded as a recurrence.
Several recent and concurrent works
such as RetNet~\citep{sun2023retentive} and GateLoop~\citep{katsch2023gateloop} strengthen this to more general forms of $L$ (\cref{sec:related}).
In this work, our formulation of structured masked attention will strongly generalize these ideas.

\subsection{Structured Matrices}
\label{sec:overview:structured-matrix}

General matrices $M \in \R^{\mathtt{(T,T)}}$ require $\mathtt{T}^2$ parameters to represent and $O(\mathtt{T}^2)$ time to perform basic operations such as matrix-vector multiplication.
\textbf{Structured matrices} are those that
\begin{enumerate}[label=(\roman*)]
  \item can be represented in subquadratic (ideally linear) parameters through a compressed representation, and
  \item have fast algorithms (most importantly matrix multiplication) by operating directly on this compressed representation.
\end{enumerate}
Perhaps the most canonical families of structured matrices are sparse and low-rank matrices.
However, there exist many other families, such as Toeplitz, Cauchy, Vandermonde, and butterfly matrices, which have all been used in machine learning for efficient models~\citep{thomas2018learning,dao2019learning,gu2022parameterization,fu2024monarch}.
Structured matrices are a powerful abstraction for efficient representations and algorithms.
In this work, we will show that SSMs are equivalent to another class of structured matrices that have not previously been used in deep learning, and use this connection to derive efficient methods and algorithms.

\subsection{Overview: Structured State Space Duality}

While this paper develops a much richer framework of connections between SSMs, attention, and structured matrices, we provide a brief summary of the main method, which is actually quite self-contained and simple algorithmically.

\paragraph{Recurrent (Linear) Form.}
The state space dual (SSD) layer can be defined as a special case of the selective SSM \eqref{eq:s6}.
The standard computation of an SSM as a recurrence (or parallel scan) can be applied, which has linear complexity in sequence length.
Compared to the version used in Mamba, SSD has two minor differences:
\begin{itemize}
  \item The structure on $A$ is further simplified from diagonal to \emph{scalar times identity} structure.
    Each $A_t$ can also be identified with just a scalar in this case.
  \item We use a larger head dimension $\mathtt{P}$, compared to $\mathtt{P}=1$ used in Mamba. Typically $\mathtt{P}=\{64,128\}$ is chosen which is similar to conventions for modern Transformers. %
\end{itemize}
Compared to the original selective SSM, these changes can be viewed as slightly decreasing the expressive power in return for significant training efficiency improvements.
In particular, our new algorithms will allow the use of matrix multiplication units on modern accelerators.

\paragraph{Dual (Quadratic) Form.}

The dual form of SSD is a quadratic computation closely related to attention, defined as
\begin{align*}%
  (L \circ QK^{\top}) \cdot V \qquad
  L_{ij} = \begin{cases}
    a_i \times \dots \times a_{j+1} & i \ge j \\
    0 & i < j
  \end{cases}
\end{align*}
where $a_i$ are input-dependent scalars bounded in $[0, 1]$.

Compared to standard softmax attention, there are two main differences
\begin{itemize}
  \item The softmax is dropped.
  \item The attention matrix is multiplied elementwise-wise by an additional mask matrix $L$.
\end{itemize}
Both of these changes can be viewed as addressing problems in vanilla attention.
For example, the softmax has been recently observed to cause problems in attention scores, such as the ``attention sink'' phenomenon~\citep{xiao2024efficient,darcet2024vision}.
More importantly, the mask matrix $L$ can be viewed as replacing the heuristic positional embeddings of Transformers with a different \emph{data-dependent positional mask} that controls how much information is transfered across time.

More broadly, this form is an instance of our \textbf{structured masked attention} generalization of linear attention, defined in \cref{sec:attention}.

\paragraph{Matrix Form and SSD Algorithm.}

The various forms of SSD are connected through a unified matrix representation,
by showing that SSMs have a matrix transformation form
$Y = MX$ for a matrix $M_\theta \in \R^{\mathtt{(T,T)}}$ that depends on $\theta = (A, B, C)$.
In particular, the dual form of SSD is equivalent to naive (quadratic-time) multiplication by the matrix $M$,
and the recurrent form is a particular efficient (linear-time) algorithm that leverages the structure in $M$.

Going beyond these, \emph{any} algorithm for multiplication by $M$ can be applied.
Our proposed hardware-efficient SSD algorithm (\cref{sec:efficient}) is a new structured matrix multiplication method that involves block decompositions of $M$, which obtains better efficiency tradeoffs than either the pure linear or quadratic forms.
It is relatively simple and easy-to-implement compared to general selective SSMs \citep{gu2023mamba};
\cref{listing} provides a complete implementation in a few lines of code.

\iftoggle{arxiv}{
\cref{fig:roadmap} provides a simple roadmap of the relationships between the concepts presented in this paper.
}{}

\subsection{Notation}

Throughout this paper, we prefer using precise notation that can be mapped to code.

\para{Matrices and Vectors.}
We generally use lower case to denote vectors (i.e.\ tensors with a single axis)
and upper case to denote matrices (i.e.\ tensors with more than one axes).
We do not bold matrices in this work.
Sometimes, if a matrix is tied or repeated along one axis (and hence can also be viewed as a vector),
we may use either upper or lower case for it.\footnote{In this work, this happens only with the $A$ parameter of SSMs.}
$\cdot$ denotes scalar or matrix multiplication while $\circ$ denotes Hadamard (elementwise) multiplication.

\para{Indexing.}
We use Python-style indexing, e.g. $i:j$ refers to the range $(i, i+1, \dots, j-1)$ when $i < j$ and $(i, i-1, \dots, j+1)$ when $i > j$.
For example, for any symbol $v$ we let $v_{j:i}$ for $j \ge i$ denote the sequence $(v_j, \dots, v_{i+1})$.
$[i]$ is equivalent to $0:i = (0, \dots, i-1)$.
For shorthand, we also let $v_{j:i}^\times$ denote the product $v_j \times \dots \times v_{i+1}$.%
\footnote{In some contexts, it is always clear that the notation $a_{i:j}$ or $A_{i:j}$ means $a_{i:j}^\times$, and the superscript is omitted. }

\para{Dimensions.}
To distinguish from matrices and tensors, we often use capital letters in typewriter fonts (e.g. $\mathtt{D}, \mathtt{N}, \mathtt{T})$ to denote dimensions and tensor shapes.
Instead of the traditional notation $M \in \mathbb{R}^{T \times T}$ we frequently use $M \in \mathbb{R}^{\mathtt{(T,T)}}$ to reflect tensor shapes in code.

\para{Tensor Contractions.}
We will heavily rely on \textbf{tensor contraction} or \textbf{einsum} notation both for clarity and as a central tool in stating and proving our results.
We assume the reader to be familiar with this notation,
which is commonly used in modern tensor libraries such as \texttt{numpy}.
For example, we can use $\mathsf{contract}(\mathtt{MN,NK \to MK})$ to denote the matrix-matrix multiplication operator, and in our notation $\mathsf{contract}(\mathtt{MN,NK \to MK})(X, Y)$ (which is equivalent to $X \cdot Y$) can be translated to code as $\mathtt{numpy.einsum('mn,nk\to mk', X, Y)}$.

\iftoggle{arxiv}{
A large glossary of notation is included in \cref{sec:glossary}.
}{}

%% file: structure/ssm.tex
\section{State Space Models are Structured Matrices}
\label{sec:ssm}

This section explores different perspectives of the state space model as a sequence transformation, and outlines properties and algorithms of such maps.
The main results of this section are about the equivalence between state space models and a family of structured matrices called semiseparable matrices,
which imply new efficiency results (\cref{thm:ssm-sss,thm:ssm-efficiency}).

\subsection{The Matrix Transformation Form of State Space Models}

Recall that our definition of an SSM is defined as a parameterized map
defined through \eqref{eq:s6}.
Our theoretical framework starts by simply writing this transformation as a matrix multiplication mapping the vectors $x \in \R^\mathtt{T} \mapsto y \in \R^\mathtt{T}$.

By definition, $h_0 = B_0 x_0$.
By induction,
\begin{align*}
  h_t &= A_t \dots A_1 B_0 x_0 + A_t \dots A_2 B_1 x_1 + \dots + A_t A_{t-1} B_{t-2} x_{t-2} + A_t B_{t-1} x_{t-1} + B_t x_t
    \\&= \sum_{s=0}^t A_{t:s}^\times B_s x_s
    .
\end{align*}

Multiplying by $C_t$ to produce $y_t$ and vectorizing the equation over $t \in [\mathtt{T}]$,
we derive the matrix transformation form of SSMs.
\begin{equation}
  \label{eq:ssm-matrix}
  \begin{aligned}
    y_t &= \sum_{s=0}^t C_t^{\top} A_{t:s}^\times B_s x_s
    \\
    y &= \mathsf{SSM}(A, B, C)(x) = Mx
    \\
    M_{ji} &\coloneqq C_j^{\top} A_{j} \cdots A_{i+1} B_{i}
  \end{aligned}
\end{equation}

\subsection{Semiseparable Matrices}

$M$ in equation \eqref{eq:ssm-matrix} is a particular representation of a class of matrices known as semiseparable matrices.
Semiseparable matrices are a fundamental matrix structure.
We first define these matrices and their properties.

\begin{definition}
  \label{def:semiseparable-rank}
  A (lower triangular) matrix $M$ is $\mathtt{N}$-semiseparable if every submatrix contained in the lower triangular portion (i.e.\ on or below the diagonal) has rank at most $\mathtt{N}$.
  We call $\mathtt{N}$ the \emph{order} or \emph{rank} of the semiseparable matrix.
\end{definition}

\cref{def:semiseparable-rank}, and other forms of related ``separable'' structure (e.g.\ quasiseparable matrices and other definitions of semiseparable matrices) are sometimes called \textbf{structured rank matrices} (or rank-structured matrices) because they are characterized by rank conditions on their submatrices.
Semiseparable matrices have many structured representations including the hierarchical semiseparable (HSS),
sequential semiseparable (SSS), and Bruhat forms~\citep{pernet2018time}.
We will primarily use the SSS form.

\subsubsection{The Sequentially Semiseparable (SSS) Representation}

\begin{definition}
  \label{def:sss}
  A lower triangular matrix $M \in \R^{\mathtt{(T,T)}}$ has a $\mathtt{N}$-\textbf{sequentially semiseparable (SSS)} representation if it can be written in the form
  \begin{equation}%
    \label{eq:sss}
    M_{ji} = C_j^{\top} A_{j} \cdots A_{i+1} B_{i}
  \end{equation}
  for vectors $B_{0}, \dots, B_{\mathtt{T}-1}, C_{0}, \dots, C_{\mathtt{T}-1} \in \R^{\mathtt{N}}$
  and matrices $A_{0}, \dots, A_{\mathtt{T}-1} \in \R^{\mathtt{(N,N)}}$.

  We define the operator $\mathsf{SSS}$ so that $M = \mathsf{SSS}(A_{0:\mathtt{T}}, B_{0:\mathtt{T}}, C_{0:\mathtt{T}})$.
\end{definition}

A fundamental result of semiseparable matrices is that they are exactly equivalent to matrices with SSS representations.
One direction can be deduced with a simple constructive proof.

\begin{lemma}
  \label{lmm:sss-rank-factor}
  An $\mathtt{N}$-SSS matrix $M$ with representation \eqref{eq:sss} is $\mathtt{N}$-semiseparable.
\end{lemma}
\begin{proof}
  Consider any off-diagonal block $M_{j:j', i':i}$ where $j' > j \ge i > i'$.
  This has an explicit rank-$\mathtt{N}$ factorization as
  \begin{equation}
    \label{eq:sss-rank-factor}
    \begin{bmatrix}
      C_j^{\top} A_{j:i'}^\times B_{i'}         & \dots & C_j^{\top} A_{j:i-1}^\times B_{i-1}   \\
      \vdots                                    &       &       \vdots      \\
      C_{j'-1}^{\top} A_{j'-1:i'}^\times B_{i'} & \dots & C_{j'-1}^{\top} A_{j'-1:i-1}^\times B_{i-1} \\
    \end{bmatrix}
    =
    \begin{bmatrix} C_j^{\top} A_{j:j}^\times \\ \vdots \\ C_{j'-1}^{\top} A_{j'-1:j}^\times \end{bmatrix}
    A_{j:i-1}^\times
    \begin{bmatrix} A_{i-1:i'}^\times B_{i'} & \cdots & A_{i-1:i-1}^\times B_{i-1} \end{bmatrix}
    .
  \end{equation}
\end{proof}
Equation \eqref{eq:sss-rank-factor} will be used extensively in deriving our fast algorithms for sequence models.
The other direction is well-established in the literature on semiseparable matrices.

\begin{proposition}
  \label{prop:sss}
  Every $\mathtt{N}$-semiseparable matrix has a $\mathtt{N}$-SSS representation.
\end{proposition}
Furthermore, note that although \cref{def:sss} involves $O(\mathtt{N}^2\mathtt{T})$ parameters for the representation (in particular to store the $A$ matrices),
it can actually be compressed down to $O(\mathtt{NT})$ parameters, which is asymptotically tight~\citep{pernet2023exact}.
Therefore in the rest of this paper we will conflate the structured matrix class (\cref{def:semiseparable-rank}) and a particular representation of it (\cref{def:sss}); we will always use this representation instead of other candidates.
In turn we will use $\mathtt{N}$-SS to refer to an $\mathtt{N}$-semiseparable matrix in SSS form.

Semiseparable matrices are a fundamental matrix structure and have many important properties.
They are deeply related to recurrences at large, and can be defined by multiple characterizations (e.g.\ \cref{def:semiseparable-rank,def:sss}) which reveal different connections and efficient algorithms for them.
We mention some of their other properties in \cref{sec:ssm:properties}.

\begin{remark}
  The notion of semiseparability is very broad and many similar but subtlely different definitions appear in the literature;
  our definitions may differ slightly from other conventions.
  First, because we are primarily concerned with causal or autoregressive settings in this paper,
  we have restricted the definition of semiseparability to the triangular case;
  \cref{def:semiseparable-rank} more formally might be called $(\mathtt{N},0)$-semiseparability by some authors.
  Some authors may also instead refer to it as a form of quasiseparability~\citep{eidelman1999new,pernet2016computing}.
  See \citet{vandebril2005bibliography} for a brief survey.
\end{remark}

\subsubsection{1-Semiseparable Matrices: the Scalar SSM Recurrence}
\label{sec:ssm:1-ss}

We will single out the special case of $1$-SS matrices.
Note that in this case, the $C_j$ and $B_i$ are scalars, and can be factored out of the SSS representation \eqref{eq:sss}
(we also use lower-case to emphasize that the parameters are scalars in this case)
\begin{align*}%
  \mathsf{SSS}(a, b, c) = \mathsf{diag}(c) \cdot M \cdot \mathsf{diag}(b) \qquad \text{where} \qquad M_{ji} = a_{j:i}^\times
  .
\end{align*}

Since diagonal matrices are easy to handle (e.g.\ multiplication by a diagonal matrix is the same as elementwise scalar multiplication),
we can ignore these terms.
Thus our basic representation of a 1-SS matrix is $M_{ji} = a_{j:i}$ or
\begin{equation}%
  \label{eq:1ss}
  M =
  \mathsf{1SS}(a_{0:T}) \coloneqq
  \begin{bmatrix}
    1 & \\
    a_1 & 1 & \\
    a_2a_1 & a_2 & 1 \\
    \vdots & \vdots & \ddots & \ddots \\
    a_{T-1}\dots a_1 & a_{T-1}\dots a_2 & \dots & a_{T-1} & 1 \\
  \end{bmatrix}
  .
\end{equation}

The importance of 1-SS matrices lies in their equivalence to the minimal form of a scalar recurrence -- the case of a degenerate SSM with state dimension $\mathtt{N}=1$ and no $(B, C)$ projections.
Note that multiplication $y = Mx$ can be computed by the recurrence
\begin{equation}
  \label{eq:1ss-recurrence}
  \begin{aligned}
    y_t &= a_{t:0}x_0 + \dots + a_{t:t}x_t \\
        &= a_t \left(a_{t-1:0}x_0 + \dots + a_{t-1:t-1}x_{t-1}\right) + a_{t:t}x_t \\
        &= a_t y_{t-1} + x_t
        .
  \end{aligned}
\end{equation}
We thus also refer to matrix multiplication by $1$-SS matrices as the \textbf{scalar SSM recurrence} or the \texttt{cumprodsum} (cumulative product sum; a generalization of cumulative product and cumulative sum) operator.
As the fundamental form of recurrence, multiplication by 1-SS matrices is important
as a building block for our main algorithms.

We emphasize that one of the central themes of this paper is that \emph{many algorithms on sequence models can be reduced to structured matrix multiplication algorithms}.
1-SS matrices exemplify this connection: there are many fast algorithms for computing the primitive scalar recurrence or \texttt{cumprodsum} operator,
and all of them turn out to be equivalent to different structured factorization of 1-SS matrices.
We dedicate \cref{sec:scan} to these algorithms for 1-SS matrix multiplication.

\subsection{State Space Models are Semiseparable Matrices}
Recall that our definition of an SSM is defined as a parameterized map
defined through \cref{def:sequence-transformation}.
The connection between SSMs and semiseparable matrices follows from simply writing this transformation as a matrix multiplication mapping the vectors $x \mapsto y \in \R^\mathtt{T}$.

Equation \eqref{eq:ssm-matrix} directly establishes the link between state space models and the sequentially semiseparable representation, which in turn are equivalent to semiseparable matrices in general (\cref{lmm:sss-rank-factor,prop:sss}).
\begin{theorem}
  \label{thm:ssm-sss}
  The state space model transformation $y = \mathsf{SSM}(A, B, C)(x)$ with state size $\mathtt{N}$ is identical to matrix multiplication by an $\mathtt{N}$-SS matrix in sequentially semiseparable representation $y = \mathsf{SSS}(A, B, C) \cdot x$.
\end{theorem}

In other words the sequence transformation operator $\mathsf{SSM}$ (\cref{def:ssm}) coincides with the matrix construction operator $\mathsf{SSS}$ (\cref{def:sss}),
and we use them interchangeably (or sometimes $\mathsf{SS}$ as shorthand).
Furthermore---by a twist of fate---structured state space models and sequentially semiseparable matrices have the same acronyms, underscoring their equivalence!
Conveniently we can use any of these acronyms SSM (state space model or semiseparable matrix), SSS (structured state space or sequentially semiseparable), or SS (state space or semiseparable) interchangeably to unambiguously refer to either concept.
However, we will generally use the convention that SSM refers to state space model, SS refers to semiseparable, and SSS refers to sequentially semiseparable.

\cref{fig:ssm-semiseparable} illustrates the sequence transformation perspective of state space models as semiseparable matrices.
\begin{figure}[!t]
  \centering
  \includegraphics[width=\linewidth]{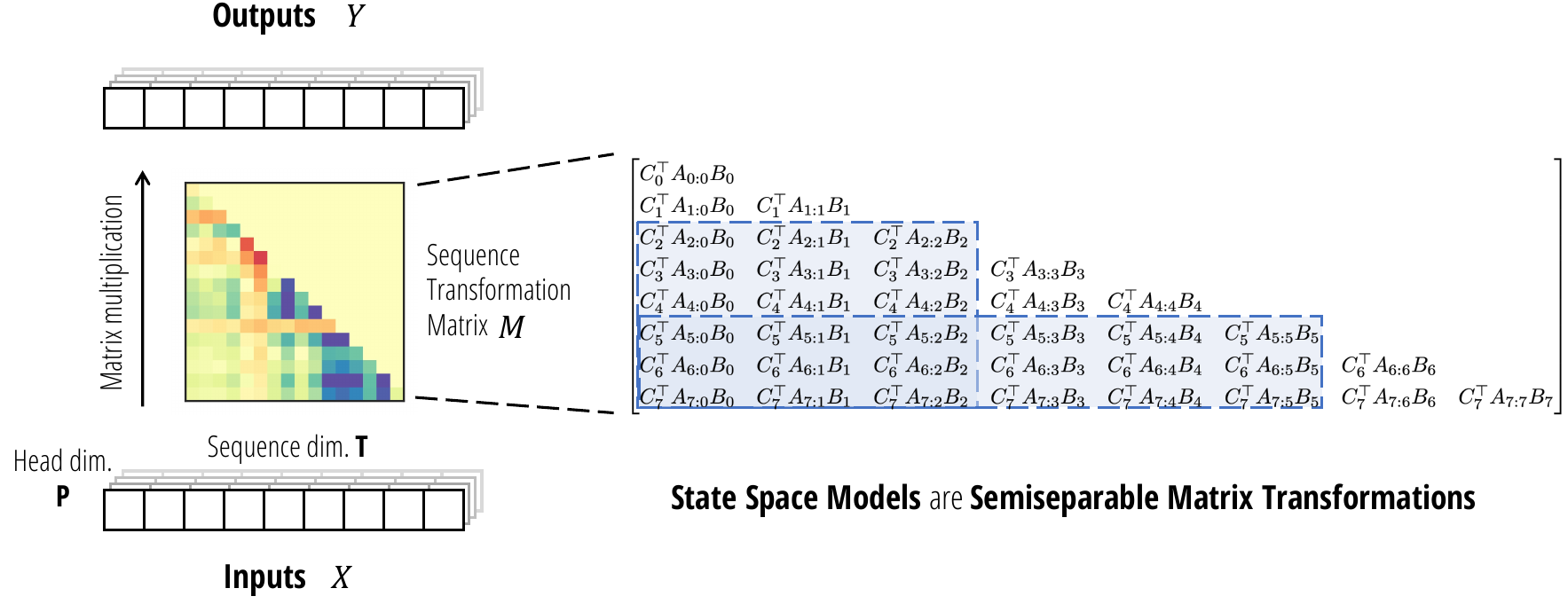}
  \caption{
    (\textbf{State Space Models are Semiseparable Matrices}.)
    As sequence transformations, state space models can be represented as a matrix transformation $M \in \mathbb{R}^{\mathtt{(T,T)}}$ acting on the sequence dimension $\mathtt{T}$,
    sharing the same matrix for each channel in a head (\emph{Left}).
    This matrix is a semiseparable matrix (\emph{Right}), which is a rank-structured matrix where every submatrix contained on-and-below the diagonal (\emph{Blue}) has rank at most $\mathtt{N}$,
    equal to the SSM's state dimension.
  }
  \label{fig:ssm-semiseparable}
\end{figure}

\subsection{Computing State Space Models through Structured Matrix Algorithms}
\label{sec:ssm:algorithms}

The reason \cref{thm:ssm-sss} is important is that it will allow us to \emph{reduce the problem of efficient computation of SSMs (and other sequence models) into efficient algorithms for structured matrix multiplication}.
We briefly provide an overview and defer our main new algorithm to \cref{sec:efficient}, after showing the equivalence of SSMs to other sequence models in
\cref{sec:attention,sec:ssd}.

As previously defined, semiseparable matrices (i.e.\ rank-structured matrices) are a classical type of structured matrix:
\begin{enumerate}[label=(\roman*)]
  \item They have compressed representations such as the SSS form which has only $O(\mathtt{T})$ instead of $O(\mathtt{T}^2)$ parameters.
  \item They have fast algorithms operating directly on the compressed representation.
\end{enumerate}
Furthermore, the parameterization and matrix multiplication cost can be tight in the semiseparable order.
\begin{proposition}[\citet{pernet2023exact}]
  \label{prop:ss-mvm}
  An $\mathtt{N}$-SS matrix of size $\mathtt{T}$ can be represented in $O(\mathtt{NT})$ parameters and has matrix-vector multiplication in time and space $O(\mathtt{NT})$.
\end{proposition}

For example, 1-SS matrices illustrate the essence of this connection.
The matrix $M = \mathsf{1SS}(a)$ is defined by exactly $\mathtt{T}-1$ parameters $a_{0:\mathtt{T}-1} = a_1, \dots, a_{\mathtt{T}-1}$,
and can be computed in $O(\mathtt{T})$ time by following the scalar recurrence \eqref{eq:1ss-recurrence}.

\subsubsection{The Linear (Recurrent) Mode}
\label{sec:ssm:algorithms:linear}

\cref{prop:ss-mvm} can be easily seen in the case of diagonal structured SSMs (S4D~\citep{gu2022parameterization}),
simply by leveraging the state space model formulation \eqref{eq:s6} and unrolling the recurrence.
We provide the formal tensor-contraction algorithm in \eqref{eq:ssm-diagonal}, where the dimension $\mathtt{S}$ is equal to $\mathtt{T}$%
\footnote{A different symbol is required for the contraction notation.}.
\begin{subequations}
  \label{eq:ssm-diagonal}
  \begin{align}
    \label{eq:ssm-diagonal:1}
    Z &= \mathsf{contract}(\mathtt{SP},\mathtt{SN} \to \mathtt{SPN})(X, B) & \mathtt{(S,P,N)} \\
    \label{eq:ssm-diagonal:2}
    H &= \mathsf{contract}(\mathtt{TSN},\mathtt{SPN} \to \mathtt{TPN})(L, Z) & \mathtt{(T,P,N)} \\
    \label{eq:ssm-diagonal:3}
    Y &= \mathsf{contract}(\mathtt{TN},\mathtt{TPN} \to \mathtt{TP})(C, H) & \mathtt{(T,P)}
  \end{align}
\end{subequations}
Here, $L \in \R^{(\mathtt{T},\mathtt{T})}$ is defined as $\mathsf{1SS}(A)$, or in other words $L_{0:\mathtt{T},0:\mathtt{T}} = \mathsf{1SS}(A_{0:\mathtt{T}})$ for $i \in [\mathtt{N}]$.
This algorithm involves three steps corresponding to \eqref{eq:s6}:
\begin{enumerate}[label=(\roman*)]
  \item \emph{expanding} the input $X$ by the input matrix $B$ \eqref{eq:ssm-diagonal:1},
  \item unrolling independent scalar SSM recurrences \eqref{eq:ssm-diagonal:2}, and
  \item \emph{contracting} the hidden state $H$ by the output matrix $C$ \eqref{eq:ssm-diagonal:3}.
\end{enumerate}
Note that we have used the equivalence between scalar SSMs and 1-SS matrices in step \eqref{eq:ssm-diagonal:2}.

\begin{remark}
  We note that \eqref{eq:ssm-diagonal} is a special case of the Mamba (S6) model.
  however, a naive implementation is slow because of the expanded tensors $Z$ and $H$ of size $\mathtt{(T,P,N)}$;
  \citet{gu2023mamba} introduced a hardware-aware implementation to avoid materializing these tensors.
\end{remark}

Surprisingly, \cref{thm:ssm-sss} and \cref{prop:ss-mvm} immediately imply that all SSMs have the same asymptotic efficiency as algorithm \eqref{eq:ssm-diagonal}.
\begin{theorem}
  \label{thm:ssm-efficiency}
  Any state space model (\cref{def:ssm}) of state size $\mathtt{N}$ on sequence length $\mathtt{T}$ can be computed in time $O(\mathtt{TN})$ (not accounting for potential preprocessing).
\end{theorem}
We note that this result is new to the structured SSM literature.
In particular, given dense unstructured $A_t$ matrices, the total representation alone seems to be of size $O(\mathtt{TN}^2)$.
Thus \cref{thm:ssm-efficiency} states the non-trivial result that with a pre-processing step, even an unstructured SSM can be computed optimally efficiently,
with upper bound matching the lower bound $O(\mathtt{TN})$ given by the size of $B$ and $C$.

\begin{remark}
  \cref{thm:ssm-efficiency} is perhaps not too surprising in light of the fact that almost all dense matrices over $\R^{\mathtt{(N,N)}}$ are diagonalizable over $\mathbb{C}$,
  leading to the result that \emph{almost all} dense real SSMs are equivalent to a diagonal complex SSM.
  This fact underlies the reason why diagonal SSMs are the most popular form of structured SSM~\citep{gupta2022diagonal,gu2022parameterization,smith2023s5}.
  However, \cref{thm:ssm-efficiency} implies the much stronger result for \emph{all} real SSMs (not just the diagonalizable ones), as well as dense SSMs over other fields (including $\mathbb{C}$ itself).
\end{remark}

In practice, efficiently computable SSMs still require additional structure on $A$,
particularly to avoid the expensive preprocessing step (which both has order $\mathtt{N}$ extra FLOPs and involves hardware-inefficient operations such as singular value decompositions).
These structures are the focus of past work on structured SSMs (e.g.\ S4(D) and Mamba) as well as our new algorithms.
In particular, when slightly stronger structure is imposed on $A$, we will design very hardware-efficient algorithms through block decompositions of the SSM matrix $M = \mathsf{SSS}(A, B, C)$ in \cref{sec:efficient}.

\subsubsection{The Quadratic (Naive) Mode}

We note that there is another way to compute an SSM exposed by our new matrix point of view.
A naive computation of the matrix SSM representation \eqref{eq:ssm-matrix} involves simply materializing the sequence transformation matrix $M=\mathsf{SSS}(A, B, C)$.
This is a $\mathtt{(T,T)}$ matrix, and therefore this naive algorithm will scale quadratically in sequence length.
However, when the sequence length $\mathtt{T}$ is short, this can actually be more efficient than the linear algorithm due to constant factors and the hardware-friendliness of the computation pattern (e.g. leveraging matrix-matrix multiplications).
In fact, for a particular case of structured SSMs, this looks very similar to a quadratic attention computation (\cref{sec:ssd}).

\subsubsection{Summary}

Many sequence models are explicitly motivated or defined as matrix sequence transformations --
most notably Transformers, where the matrix mixer is the attention matrix.
On the other hand, RNNs and SSMs have not previously been described in this way.
By providing an explicit \emph{matrix transformation} form of state space models,
we reveal new ways of understanding and using them.
From a computational perspective, any method of computing the forward pass of a state space model can be
viewed as a matrix multiplication algorithm on semiseparable matrices.
The semiseparable matrix perspective provides one lens into state space duality (SSD),
where the dual modes respectively refer to a linear-time semiseparable matrix multiplication algorithm and quadratic-time naive matrix multiplication.

Moreover, leveraging the rich structure of semiseparable matrices can lead to even better algorithms and more insights (e.g. \cref{sec:efficient,sec:scan}).
In \cref{sec:ssm:properties}, we describe some additional properties of semiseparable matrices.

%% file: structure/attention.tex
\section{Structured Masked Attention: Generalizing Linear Attention \texorpdfstring{\\}{} with Structured Matrices}
\label{sec:attention}

In this section we revisit the linear attention framework from first principles.
The main results in this section are a simple tensor-contraction-based proof of linear attention (\cref{prop:linear-attention}),
and our generalized abstraction of structured masked attention in \cref{def:sma}.
\iftoggle{arxiv}{
We note that this section derives the main duality results from a different direction than state space models and can be read completely independently of \cref{sec:ssm}.
}{}

\begin{itemize}
  \item \cref{sec:masked-attention} sets up our framework for variants of attention, with a particular focus on kernel attention and masked kernel attention.
  \item \cref{sec:linear-attention} provides our first main attention result, a simple proof of linear attention through the lens of tensor contractions.
  \item \cref{sec:structured-attention} defines structured masked attention, our generalization of prior attention variants through structured matrices.
\end{itemize}

\subsection{The Attention Framework}
\label{sec:masked-attention}

\subsubsection{Attention}
The basic form of (single-head) attention is a map on three sequences of vectors $(Q, K, V) \mapsto Y$.
\begin{equation}
  \label{eq:kernel-attention}
  \begin{aligned}%
    Q &= \mathsf{input} & \mathtt{(T,N)} \\
    K &= \mathsf{input} & \mathtt{(S,N)} \\
    V &= \mathsf{input} & \mathtt{(S,P)} \\
    G &= QK^\top        & \mathtt{(T,S)} \\
    M &= f(G)           & \mathtt{(T,S)} \\
    Y &= GV             & \mathtt{(T,P)} \\
  \end{aligned}
\end{equation}
We use ``shape annotations'' to indicate the dimensions of tensors, e.g.\ $Q \in \R^{\mathtt{(T,N)}}$.
In this general form, $\mathtt{S}$ and $\mathtt{T}$ represent \emph{source} and \emph{target} sequence lengths,
$\mathtt{N}$ represents the \emph{feature dimension}, and $\mathtt{P}$ represents the \emph{head dimension}.

The most common variant of \textbf{softmax attention} uses a softmax activation $f=\mathsf{softmax}$ to normalize the rows of the $G$ matrix.

\subsubsection{Self-Attention}
Our treatment is motivated by the most important case of self-attention, where
\begin{enumerate}[label=(\roman*)]
  \item the source and target sequences are the same (i.e.\ $\mathtt{S}=\mathtt{T}$),
  \item usually the feature and head dimensions are the same (i.e.\ $\mathtt{N}=\mathtt{P}$),
  \item and $Q, K, V$ are generated by linear projections on the same input vector ($Q = W_Q \cdot X, K = W_K \cdot X, V = W_V \cdot X$).
\end{enumerate}
However, our presentation abstracts away these choices and begins from the $Q, K, V$ matrices.

\begin{remark}
  Our focus is on the self-attention case with equal head and feature dimensions (i.e.\ $\mathtt{S}=\mathtt{T}$ and $\mathtt{N}=\mathtt{P}$),
  which should be used as the running example.
  We define the general formulation of attention not only so that our framework captures variants such as cross-attention,
  but also because separating the notation for dimensions (e.g.\ $\mathtt{S}$ and $\mathtt{T}$) makes the contraction notation proofs
  of our main results in this section more clear.
\end{remark}

\begin{remark}
  \label{rmk:attention-input}
  While attention is usually framed as an operation on these three inputs $Q, K, V$ which are viewed symmetrically,
  the input and output dimensions in \eqref{eq:kernel-attention} indicate otherwise.
  In particular, the feature dimension $\mathtt{N}$ is not present in the output;
  therefore in the case when $\mathtt{S}=\mathtt{T}$ (e.g.\ self-attention),
  we view $V$ as the main input, so that \eqref{eq:kernel-attention} defines a proper sequence transformation $V \mapsto Y$ (\cref{def:sequence-transformation}).
\end{remark}

\subsubsection{Kernel Attention}
\label{sec:attention:kernel}

The step where the softmax function is applied to the Gram matrix $G$ can be decomposed into two parts:
\begin{enumerate}
  \item Exponentiating the $G$ matrix.
  \item Normalizing the $G$ matrix on the $\mathtt{S}$ axis.
\end{enumerate}
We can ignore the normalization term for now, as it amounts to simply passing in $V=1$ and dividing\iftoggle{arxiv}{ (we revisit this in \cref{sec:architecture:kernels})}{}.
The exponentiation term can be viewed as a kernel transformation:
there is an (infinite-dimensional) feature map $\varphi$ such that $\exp(QK^{\top}) = \varphi(Q)\varphi(K)^{\top}$.
By abstracting away the feature map into the definition of $Q$ and $K$ itself (i.e.\ define $Q, K$ as the post-transformed versions),
we can ignore the softmax transformation, and assume that $Q, K$ are arbitrarily generated by kernel feature maps and potentially $\mathtt{N} \neq \mathtt{P}$.

Many instantiations of kernel attention have been proposed, including:
\begin{itemize}
  \item The original Linear Attention \citep{katharopoulos2020transformers} defines the kernel feature map as an arbitrary pointwise activation function, such as $x \mapsto 1+\mathsf{elu}(x)$.
  \item Random Feature Attention (RFA)~\citep{peng2021random} chooses the kernel feature map to approximate softmax attention (i.e. the $\exp$ feature map) using the random Fourier feature approximation of Gaussian kernels~\citep{rahimi2007random}. This involves random projections (i.e.\ multiplying $Q$ and $K$ by a random projection $W$ and applying the activation $x \mapsto (\cos(x), \sin(x))$.
  \item Performer~\citep{choromanski2021rethinking} proposes the fast attention via positive orthogonal random features (FAVOR+).
    The positive random features (PRF) part chooses the kernel feature map to be a random projection followed by the feature map $x \mapsto 2^{-1/2}(\exp(x), \exp(-x))$.
    This choice is motivated so that the kernel elements are positive-valued and provably approximates the softmax attention. [It also proposes choosing the random projections in orthogonal directions, which we do not consider.]
  \item cosFormer~\citep{qin2022cosformer} augment RFA with a cosine reweighting mechanism that incorporates positional information to emphasize locality.
    This effectively passes $Q_t,K_t$ through the feature map $x \mapsto (x \cos(\pi t / 2T), \sin(\pi t / 2T))$.
  \item Linear Randomized Attention~\citep{zheng2022linear} generalize RFA from the perspective of importance sampling, and generalize it to provide better estimates of the full softmax kernel (rather than just the $\exp$-transformed numerator).
\end{itemize}

Other related attention variants include Linformer~\citep{wang2020linformer} and Nystr\"{o}former~\citep{xiong2021nystromformer}, which both use low-rank approximations of the attention matrix $M$ (and are thus compatible with equation \eqref{eq:kernel-attention}), through random projections (Johnson-Lindenstrauss) and kernel approximation (the Nystr\"{o}m method) respectively.

\subsubsection{Masked (Kernel) Attention}

Let $L$ be a mask of shape $\mathtt{(T,S)}$.
Most commonly, in the \emph{autoregressive} self-attention case when $\mathtt{S}=\mathtt{T}$,
$L$ may be a lower-triangular matrix of $1$'s representing a \emph{causal mask}.
Besides enforcing causality, many other types of masks can be applied -- in particular various sparsity patterns such as banded, dilated, or block diagonal -- which are motivated by reducing the complexity of dense attention. %

Masked attention is usually written in matrix notation as
\begin{equation}%
  \label{eq:sha-quad-matrix}
  y = (L \circ (QK^\top)) \cdot V
  .
\end{equation}
More precisely, with shape annotations and breaking this down into the precise sequence of computations:
\begin{equation}
  \label{eq:sha-quad-0}
  \begin{aligned}%
    G &= QK^\top & \mathtt{(T,S)} \\
    M &= G \circ L & \mathtt{(T,S)} \\
    Y &= M V & \mathtt{(T,P)}
  \end{aligned}
\end{equation}

Our improved derivation of attention variants in this section starts by noticing that this formula can be written as a \emph{single contraction}:
\begin{equation}
  \label{eq:sha}
  Y = \mathsf{contract}(\mathtt{TN},\mathtt{SN},\mathtt{SP},\mathtt{TS} \to \mathtt{TP})(Q, K, V, L)
\end{equation}

and the algorithm in \eqref{eq:sha-quad-0} can be reframed as computing \eqref{eq:sha} by a particular ordering of pairwise contractions
\begin{subequations}
  \label{eq:sha-quad}
  \begin{align}%
    \label{eq:sha-quad:1}
    G &= \mathsf{contract}(\mathtt{TN, SN} \to \mathtt{TS})(Q, K) && \qquad \mathtt{(T,S)} \\
    \label{eq:sha-quad:2}
    M &= \mathsf{contract}(\mathtt{TS, TS} \to \mathtt{TS})(G, L) && \qquad \mathtt{(T,S)} \\
    \label{eq:sha-quad:3}
    Y &= \mathsf{contract}(\mathtt{TS, SP} \to \mathtt{TP})(M, V) && \qquad \mathtt{(T,P)}
  \end{align}
\end{subequations}

\subsection{Linear Attention}
\label{sec:linear-attention}

Linear attention, and many other variants of efficient attention, is often motivated by changing the order of matrix associativity in the core attention computation $(QK^\top)V = Q(K^\top V)$.
However when the mask is added, the derivation is somewhat less straightforward (for example, the original paper~\citep{katharopoulos2020transformers} and variants~\citep{sun2023retentive} state the formula without proof).

Roughly, the linear attention method claims that the following formula
is equivalent to \eqref{eq:sha-quad-matrix}, which must be verified by expanding the sum and tracking indices carefully.
\begin{equation}
  \label{eq:sha-lin-matrix}
  Y = Q \cdot \mathsf{cumsum}(K^\top V)
\end{equation}

\begin{proposition}[\citep{katharopoulos2020transformers}]
  \label{prop:linear-attention}
  Autoregressive kernel attention, i.e.\ masked kernel attention with the causal mask, can be computed in $O(T)$ time by a recurrence taking constant time per step.
\end{proposition}

\subsubsection{A Tensor Contraction Proof of Linear Attention}

We present a simple and rigorous derivation of linear attention that will also immediately reveal how to generalize it.
The main idea is to perform the contraction \eqref{eq:sha} in an alternate order.
We avoid ambiguous matrix notation and work directly with contraction notation:
\begin{subequations}
  \label{eq:sha-lin}
  \begin{align}
    \label{eq:sha-lin:1}
    Z &= \mathsf{contract}(\mathtt{SP},\mathtt{SN} \to \mathtt{SPN})(V, K) & \mathtt{(S,P,N)} \\
    \label{eq:sha-lin:2}
    H &= \mathsf{contract}(\mathtt{TS},\mathtt{SPN} \to \mathtt{TPN})(L, Z) & \mathtt{(T,P,N)} \\
    \label{eq:sha-lin:3}
    Y &= \mathsf{contract}(\mathtt{TN},\mathtt{TPN} \to \mathtt{TP})(Q, H) & \mathtt{(T,P)}
  \end{align}
\end{subequations}

Intuitively, we interpret this contraction order as follows.

The first step \eqref{eq:sha-lin:1} performs an ``expansion'' into more features, by a factor of the feature dimension $\mathtt{N}$.
The third step \eqref{eq:sha-lin:3} contracts the expanded feature dimension away.
If $K$ is viewed as the input (\cref{rmk:attention-input}),
then $V$ and $Q$ perform the expansion and contraction, respectively.

The second step is the most critical, and explains the \emph{linear} part of linear attention.
First notice that \eqref{eq:sha-lin:2} is just a direct matrix multiplication by $L$ (since the $\mathtt{(P,N)}$ axes can be flattened).
Also note that this is the only term that involves both $\mathtt{T}$ and $\mathtt{S}$ axes,
hence should have $\Omega(\mathtt{TS})$ complexity (i.e.\ quadratic in sequence length).
However, when the mask $L$ is the standard causal attention mask (lower triangular $1$'s),
matrix-vector multiplication by $L$ is identical to a feature-wise cumulative sum
\begin{align*}%
  y = \begin{bmatrix} 1 \\ \vdots & \ddots \\ 1 & \dots & 1 \end{bmatrix} x
  \quad\iff\quad
  \begin{aligned}
    y_0 &= x_0 \\
    y_t &= y_{t-1} + x_t
  \end{aligned}
  .
\end{align*}

\subsection{Structured Masked Attention}
\label{sec:structured-attention}

With the tensor contraction perspective of masked attention \eqref{eq:sha-lin},
we can immediately see that the crux of the original linear attention is the fact that \emph{matrix-vector multiplication by the causal mask is equivalent to the cumulative sum operator}.

However, we observe that there is no reason the attention mask has to be all $1$'s.
All that is necessary for linear attention to be fast is for $L$ to be a \emph{structured matrix},
which by definition are those that have fast matrix multiplication\iftoggle{arxiv}{ (\cref{sec:overview:structured-matrix})}{}.
In particular, we can use \emph{any mask matrix} $L$ that has sub-quadratic (ideally linear) matrix-vector multiplication,
which would have the same complexity as standard linear attention by speeding up the bottleneck equation \eqref{eq:sha-lin:2}.

\begin{definition}%
  \label{def:sma}
  \textbf{Structured masked attention (SMA)} (or \textbf{structured attention} for short) is defined as a \emph{function} on queries/keys/values $Q, K, V$ as well as any \emph{structured matrix} $L$ (i.e.\ has sub-quadratic matrix multiplication), through the 4-way tensor contraction
  \begin{align*}
    Y = \mathsf{contract}(\mathtt{TN},\mathtt{SN},\mathtt{SP},\mathtt{TS} \to \mathtt{TP})(Q, K, V, L)
    .
  \end{align*}
  The SMA \textbf{quadratic mode algorithm} is the sequence of pairwise contractions defined by \eqref{eq:sha-quad},
  which corresponds to the standard (masked) attention computation.

  The SMA \textbf{linear mode algorithm} is the sequence of pairwise contractions defined by \eqref{eq:sha-lin},
  where step \eqref{eq:sha-lin:2} is optimized through the subquadratic structured matrix multiplication.
\end{definition}
We can instantiate structured masked attention to any given class of matrix structure.
Some examples include (\cref{fig:sma}):
\begin{itemize}
  \item Linear attention uses a causal mask.
  \item RetNet~\citep{sun2023retentive} uses a decay mask $L_{ij} = \gamma^{i-j} \cdot \mathbb{I}[j \ge i]$ for some decay factor $\gamma \in [0, 1]$.
  \item The decay mask could be generalized to a Toeplitz matrix $L_{ij} = \alpha_{i-j}$ for some learnable (or input-dependent) set of parameters $\alpha \in \mathbb{R}^\mathtt{T}$. This can be interpreted as a form of relative positional encoding, reminiscent of other methods such as AliBi~\citep{press2022train} but multiplicative instead of additive.
  \item Another variant could use a Fourier matrix $L_{ij} = \omega^{ij / \mathtt{T}}$ to encode positional structure a different way.
\end{itemize}
In \cref{sec:ssd}, we consider semiseparable SMA, which defines our main SSD model.

\iftoggle{arxiv}{
\begin{figure}[!t]
  \centering
  \includegraphics[width=0.9\linewidth]{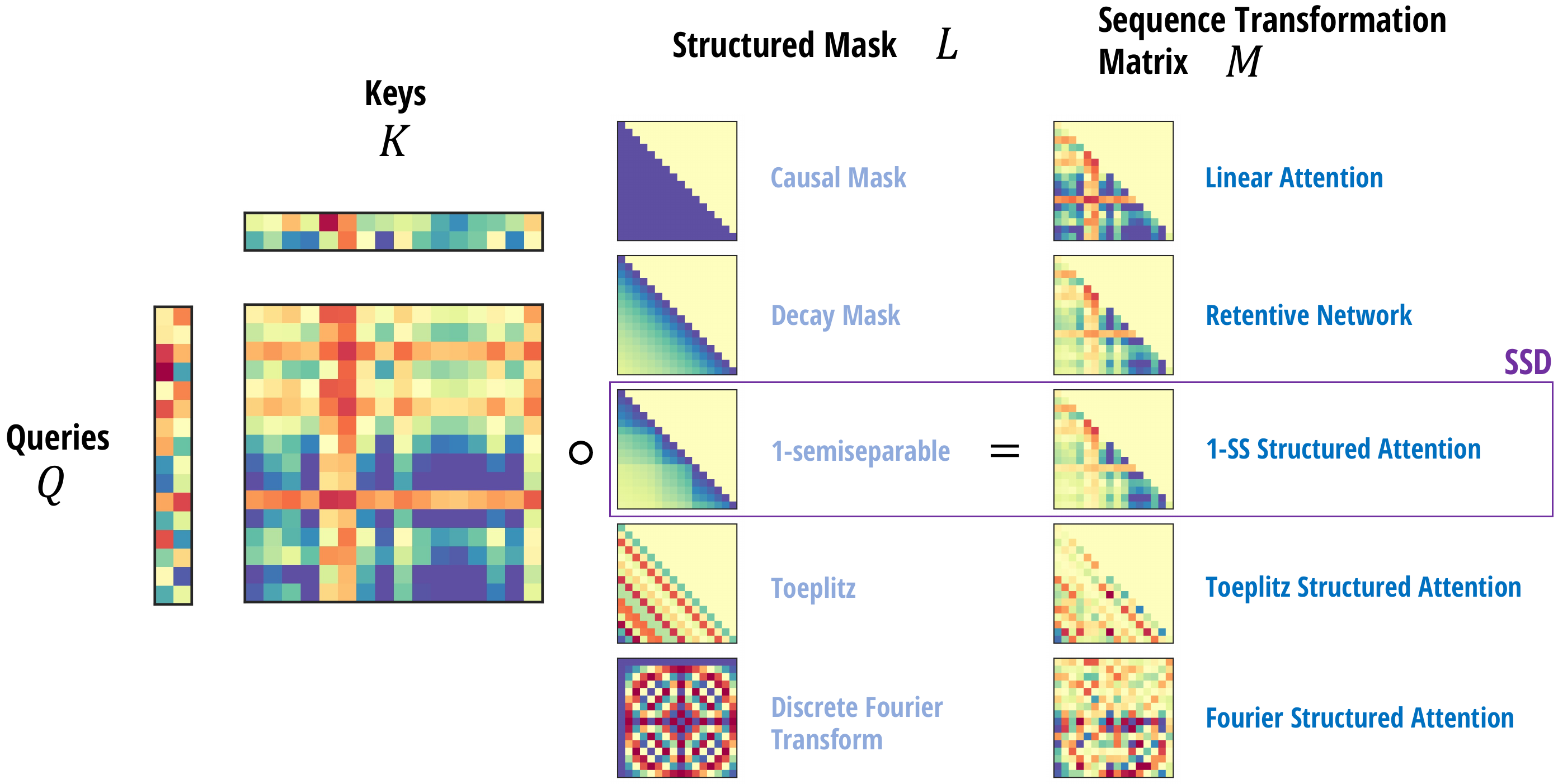}
  \caption{
    (\textbf{Structured Masked Attention}.)
    SMA constructs a masked attention matrix $M = QK^\top \circ L$ for any structured matrix $L$, which defines a matrix sequence transformation $Y = MV$.
    All instances of SMA have a dual subquadratic form induced by a different contraction ordering, combined with the efficient structured matrix multiplication by $L$.
    Previous examples include Linear Attention~\citep{katharopoulos2020transformers} and RetNet~\citep{sun2023retentive}.
    Beyond SSD (1-semiseparable SMA), the focus of this paper, many other potential instantiations of structured attention are possible.
  }
  \label{fig:sma}
\end{figure}
}{}

\subsubsection{Summary: The Dual Forms of Masked Attention}

Standard (masked kernel) attention is often conflated between a function and an algorithm.
Separating this distinction presents a clear way to understand different variants of attention.
\begin{itemize}
  \item We view \textbf{masked attention} as a particular \emph{function}~\eqref{eq:sha}.
  \item The standard \textbf{quadratic attention} computation \eqref{eq:sha-quad} can be viewed as an \emph{algorithm} to compute the function.
  \item \textbf{Linear attention} \eqref{eq:sha-lin} is an alternate algorithm to compute the same function.
\end{itemize}

Moreover, in this case
\begin{itemize}
  \item The masked attention function is simply a particular \emph{contraction on four terms}.
  \item The quadratic and linear attention algorithms are simply \emph{two different orders to perform the contractions}.
\end{itemize}
It is known that contraction orderings can make large differences in computation complexity, leading to the quadratic vs.\ linear split.
Just as state space models are a transformation that can be computed in multiple ways, with dual quadratic vs.\ linear forms (\cref{sec:ssm:algorithms}),
linear attention has a similar duality that results from two contraction orders.
\iftoggle{arxiv}{
In fact, these turn out to be different perspectives on the same underlying duality, which we make explicit in \cref{sec:ssd}.
}{}

%% file: structure/ssd.tex
\section{State Space Duality}
\label{sec:ssd}

In \cref{sec:ssm,sec:attention}, we defined structured state space models and structured attention,
discussed their properties, and showed that they both have a quadratic algorithm and a linear algorithm.
This section connects them together.
Our main result is showing that a particular case of structured state space models coincides with a particular case of structured attention,
and that the linear-time SSM algorithm and quadratic-time kernel attention algorithm are dual forms of each other.
\begin{itemize}
  \item \cref{sec:ssd:quadratic-ssm} specializes state space models to scalar structure, where the naive quadratic computation can be seen as an instance of kernel attention.
  \item \cref{sec:ssd:1ss-sma} specializes structured masked attention to semiseparable SMA, which characterizes masked attention with efficient autoregression.
  \item \cref{sec:ssd:ssd} summarizes the connection between structured masked attention and structured state space models, termed structured state space duality.
\end{itemize}

\subsection{Scalar-Identity Structured State Space Models}
\label{sec:ssd:quadratic-ssm}

In \cref{sec:ssm} we showed that state space models are equivalent to semiseparable matrix transformations,
resulting in both a linear recurrent form and quadratic naive form.

Recall that SSMs are defined by $y = \mathsf{SSM}(A, B, C)(x)$, and the matrix form of SSMs uses the SSS (sequentially semiseparable) representation
$M = \mathsf{SSS}(A, B, C)$ where
$M_{ji} = C_j^{\top} A_{j:i} B_i$ (equation \eqref{eq:ssm-matrix}).

Now let us consider the case where $A_j$ is simply a scalar;
in other words, an instantiation of a structured SSM where the $A$ matrices are \emph{extremely} structured: $A = aI$ for scalar $a$ and identity matrix $I$.
Then we can rearrange
\begin{align*}
  M_{ji} = A_{j:i} \cdot (C_j^{\top}B_i)
.
\end{align*}
And this can be vectorized into
\begin{align*}%
  L &\coloneqq \mathsf{1SS}(a) \\
  M &= L \circ (C B^{\top}) \\
\end{align*}
where $B, C \in \R^{\mathtt{(T,N)}}$.

Using this formulation, the full output $Y=MX$ is computed precisely as
\begin{equation}
  \label{eq:ssm-quad}
  \begin{aligned}%
    G     & = \mathsf{contract}(\mathtt{TN, SN \to TS})(C, B) & & \qquad \mathtt{(T,S)} \\
    M     & = \mathsf{contract}(\mathtt{TS, TS \to TS})(G, L) & & \qquad \mathtt{(T,S)} \\
    Y     & = \mathsf{contract}(\mathtt{TS, SP \to TP})(M, X) & & \qquad \mathtt{(T,P)}
  \end{aligned}
\end{equation}
where $\mathtt{S}=\mathtt{T}$.
But this is exactly the same as original definition of masked kernel attention definition \eqref{eq:sha-quad}!

Therefore, as alluded to in \cref{sec:ssm:algorithms},
\emph{naively computing the scalar structured SSM---by materializing the semiseparable matrix $M$ and performing quadratic matrix-vector multiplication---is exactly the same as quadratic masked kernel attention.}

\subsection{1-Semiseparable Structured Masked Attention}
\label{sec:ssd:1ss-sma}

Structured masked attention allows for the use of any structured mask $L$.
When $L$ is the causal mask, it is standard linear attention.
Note that the causal mask is $L = \mathsf{SS}(1_T)$, i.e.\ the $1$-SS mask is generated by $a_t=1$ in definition~\eqref{eq:1ss}.
This motivates generalizing $L$ to the class of 1-semiseparable masks, or \textbf{1-semiseparable structured masked attention (1-SS SMA)},
where the $\mathsf{cumsum}$ in linear attention's recurrence is replaced by a more general recurrence -- the scalar SSM scan, i.e.\ 1-semiseparable matrix multiplication~(\cref{sec:ssm:1-ss}).

Finally, the most important reason we consider 1-semiseparable SMA is because the linear form for computing it is a special case of diagonal state space model.
The linear form of SMA is algorithm \eqref{eq:sha-lin}, where the bottleneck step \eqref{eq:sha-lin:2} can be viewed as matrix multiplication by the 1-SS mask.
In \cref{sec:ssm}, we also wrote out the computation for a diagonal SSM \eqref{eq:ssm-diagonal}, where the bottleneck step \eqref{eq:ssm-diagonal:2} is a scalar SSM recurrence which is equivalent to 1-SS multiplication.
The only difference is that \eqref{eq:ssm-diagonal:2} has an extra $\mathtt{N}$ dimension in $L$, because the matrix $A$ is a diagonal matrix of size $\mathtt{N}$.
This $\mathtt{N}$ dimension would disappear if all diagonal entries of $A$ are the same,
which results in \cref{cor:1ss-sma}.

\begin{corollary}
  \label{cor:1ss-sma}
  1-SS SMA (masked attention with 1-semiseparable structured matrices $L$)~\eqref{eq:sha-lin} is a special case of a diagonal SSM~\eqref{eq:ssm-diagonal} where the diagonal matrix is a scalar multiple of the identity.
\end{corollary}

While \cref{cor:1ss-sma} says that 1-SS SMA has an efficient recurrent form,
we can also show a converse result that characterizes which instances of SMA has efficient autoregression.
\begin{theorem}
  \label{thm:ss-sma}
  For any instantiation of structured masked attention (\cref{def:sma}) that is an autoregressive process with bounded order,
  the structured mask $L$ must be a semiseparable matrix.
\end{theorem}
In other words, efficient autoregressive attention is general \emph{semiseparable SMA}.
\cref{thm:ss-sma} is proved in \cref{sec:theory-details:ssm-sma}.

\begin{remark}
  While 1-semiseparable SMA is a special case of a state space model,
  general semiseparable SMA is strictly more expressive than 1-SS SMA, and cannot be described by a standard SSM.
  However, the semiseparable multiplication by $L$ and the linear form of SMA (equation \eqref{eq:sha-lin:1}) each involve an expansion and contraction step, and can be absorbed into a similar instance of 1-SS SMA with a single (larger) expansion.
\end{remark}

In summary, 1-semiseparable structured attention is the most important case of SMA, because it is:
\begin{itemize}
  \item a natural generalization of linear attention with an input-dependent recurrence.
  \item the simplest case of general semiseparable attention, which is equivalent to efficient autoregressive attention.
  \item a special case of a diagonal state space model.
\end{itemize}

\subsection{Structured State-Space Duality (SSD)}
\label{sec:ssd:ssd}

To summarize our results:
\begin{itemize}
  \item Structured state-space models (\cref{sec:ssm}) are a model usually defined through a linear-time recurrence.
    However, by expanding the matrix formulation characterizing its linear sequence-to-sequence transformation, one can derive a quadratic form.
  \item Attention variants (\cref{sec:attention}) are a model defined through quadratic-time pairwise interactions. However, by viewing it as a four-way tensor contraction and reducing in a different order, one can derive a linear form.
  \item A natural special case of each one -- 
    more precisely, state space models with scalar-identity structure on the $A$ matrices, and structured masked attention with 1-semiseparable structure on its $L$ mask
    -- are duals of each other with the exact same linear and quadratic forms.
\end{itemize}
\cref{fig:ssd} summarizes the duality between these two representations.

\begin{figure*}
  \begin{minipage}[c]{.49\linewidth}
    \small
    \centering
    \begin{tabular}{@{}lll@{}}
      \toprule
      Structured State Space Model                                  & Structured Masked Attention                               \\
      \midrule
      $C$ \hfill (contraction matrix)                                    & $Q$ \hfill (queries)                                      \\
      $B$ \hfill (expansion matrix)                                     & $K$ \hfill (keys)                                         \\
      $X$ \hfill (input sequence)                                   & $V$ \hfill (values)                                       \\
      $A_{j:i}$ \hfill (state matrix)                               & $L_{ji}$ \hfill (mask)                                    \\
      $\mathtt{N}$ \hfill (state expansion dim.)                    & $\mathtt{N}$ \hfill (kernel feature dim.)              \\
                                                                    \midrule
      $H$ \hfill (hidden states \eqref{eq:ssm-diagonal:2})          & \qquad \multirow{2}{*}{SMA linear dual \eqref{eq:sha-lin}} \\
      $\quad = L \cdot XB$ \hfill (linear mode)                     &                                                          \\
      \midrule
      \qquad \multirow{2}{*}{SSM quadratic dual \eqref{eq:ssm-quad}} & $G$ \hfill (Gram matrix \eqref{eq:sha-quad:1})            \\
                                                                    & $\quad = Q \cdot K^{\top}$ \hfill (quadratic mode)        \\
      \bottomrule
    \end{tabular}
  \end{minipage}
  \hfill
  \begin{minipage}[c]{0.49\linewidth}
    \centering
    \includegraphics[width=\linewidth]{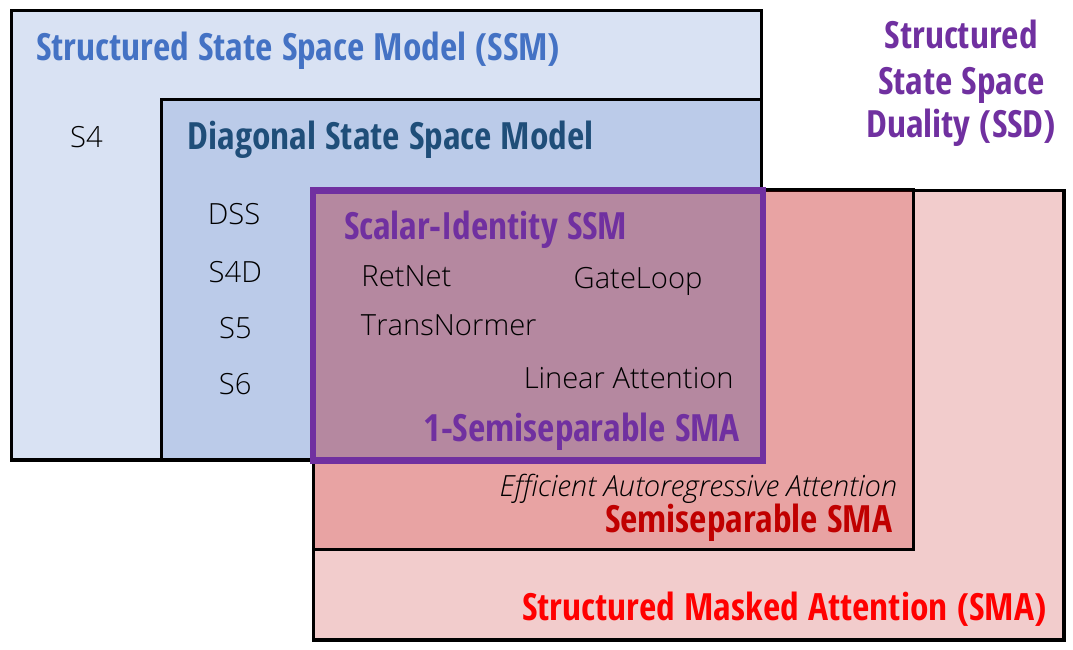}
  \end{minipage}
  \captionsetup{type=figure}
  \caption{
    (\textbf{Structured State Space Duality}.)
    State space duality describes the close relationship between state space models and masked attention.
    (\emph{Left}) General SSMs and SMA both possess linear and quadratic forms, with direct analogs in notation.
    (\emph{Right}) SSMs and SMA intersect at a large class of \emph{state space dual models} (SSD) which capture many sequence models as special cases.
  }
  \label{fig:ssd}
\end{figure*}

An extended related work and discussion (\cref{sec:related}) describes the relationship between SSD and general SSMs / attention in more detail.

%% file: structure/efficient.tex
\section{A Hardware-Efficient Algorithm for SSD Models}
\label{sec:efficient}

The benefits of developing the theoretical SSD framework between SSMs, attention, and structured matrices lies in using the connections to improve the models and algorithms.
In this section, we show how various algorithms for computing SSD models efficiently can be derived from various algorithms for computing structured matrix multiplication.

\iftoggle{arxiv}{
Our main computational result is an algorithm for computing SSD models that combines both the linear (recurrent) mode and quadratic (attention) mode.
This algorithm is as computation efficient as SSMs (linear scaling in sequence length) and as hardware-friendly as attention (primarily uses matrix multiplications).

\begin{theorem}
  \label{thm:algorithm}
  Consider an SSD model with state expansion factor $\mathtt{N}$ and head dimension $\mathtt{P}=\mathtt{N}$.
  There exists an algorithm for computing the model on any input $X \in \R^{\mathtt{(T, P)}}$
  which only requires $O(\mathtt{TN}^2)$ training FLOPs, $O(\mathtt{TN})$ inference FLOPs,
  $O(\mathtt{N}^2)$ inference memory,
  and whose work is dominated by matrix multiplications.
\end{theorem}
Note that all of these bounds are tight, because
a state space model with state expansion $\mathtt{N}$ operating on a head size of $\mathtt{N}$ has total state size $\mathtt{N}^2$ (yielding the lower bounds for training and inference FLOPs of $O(\mathtt{TN}^2)$ and $O(\mathtt{N}^2)$ respectively).
Furthermore the input $X$ itself has $\mathtt{TN}$ elements, yielding the memory lower bound.

The main idea behind \cref{thm:algorithm} is once again viewing the problem of computing a state space model as a semiseparable matrix multiplication, but leveraging its structure in a new way.
Instead of computing the whole matrix in either recurrent or attention mode,
we perform a \emph{block decomposition} of the matrix.
The diagonal blocks can be computed using the dual attention mode, which can be efficiently done with matrix multiplications,
while the off-diagonal blocks can be factored by the rank-structure of semiseparable matrices and reduced to a smaller recurrence.
We highlight that \cref{listing} provides a self-contained implementation of the SSD algorithm.
Compared to the general selective SSM of \citet{gu2023mamba},
this implementation is much simpler, and relatively efficient even in native PyTorch without requiring special low-level kernels.
}{
}

To begin, we partition the matrix $M$ into a $\frac{\mathtt{T}}{\mathtt{Q}} \times \frac{\mathtt{T}}{\mathtt{Q}}$ grid of submatrices of size $\mathtt{Q} \times \mathtt{Q}$,
for some block size $\mathtt{Q}$.
Note that the off-diagonal blocks are low-rank by the defining property of semiseparable matrices (\cref{def:semiseparable-rank}).%
\footnote{Note that the block decomposition is valid even with partitions of varying size, e.g.\ if $\mathtt{Q} \not\mid \mathtt{T}$, but we assume even divisibility for simplicity.}
\begin{align*}%
  \text{(Block Decomposition)} \quad M &=
  \begin{bmatrix}
    M^{(0,0)} \\
    M^{(1,0)} & M^{(1,1)} \\
    \vdots & \vdots & \ddots \\
    M^{(\mathtt{T}/\mathtt{Q}-1,0)} & M^{(\mathtt{T}/\mathtt{Q}-1,1)} & \dots & M^{(\mathtt{T}/\mathtt{Q}-1,\mathtt{T}/\mathtt{Q}-1)} \\
  \end{bmatrix}
  \\
  \text{(Diagonal Block)} \quad M^{(j,j)} &= \mathsf{SSM}(A_{j\mathtt{Q}:(j+1)\mathtt{Q}}, B_{j\mathtt{Q}:(j+1)\mathtt{Q}}, C_{j\mathtt{Q}:(j+1)\mathtt{Q}}) \\
  \text{(Low-Rank Block)} \quad M^{(j,i)} &= 
  \begin{bmatrix}C_{j\mathtt{Q}}^{\top} A_{j\mathtt{Q}:j\mathtt{Q}-1} \\[1pt] \vdots \\[1pt] C_{(j+1)\mathtt{Q}-1}^{\top} A_{(j+1)\mathtt{Q}-1:j\mathtt{Q}-1}\end{bmatrix}
  A_{j\mathtt{Q}-1:(i+1)\mathtt{Q}-1}
  \begin{bmatrix}B_{i\mathtt{Q}}^{\top} A_{(i+1)\mathtt{Q}-1:i\mathtt{Q}} \\[1pt] \vdots \\[1pt] B_{(i+1)\mathtt{Q}-1}^{\top} A_{(i+1)\mathtt{Q}-1:(i+1)\mathtt{Q}-1}\end{bmatrix}^{\top}
\end{align*}

This is easiest illustrated through an example, e.g.\ for $\mathtt{T}=9$ and decomposing into chunks of length $\mathtt{Q}=3$.
The shaded cells are low-rank factorizations of the off-diagonal blocks of the semiseparable matrix.
\begin{align*}%
M &=
\begin{bNiceArray}{ccc|ccc|ccc}[cell-space-limits=2pt]
    C_0^{\top} A_{0:0} B_0 & \\
    C_1^{\top} A_{1:0} B_0 & C_1^{\top} A_{1:1} B_1 & \\
    C_2^{\top} A_{2:0} B_0 & C_2^{\top} A_{2:1} B_1 & C_2^{\top} A_{2:2} B_2 \\
    \hline
    C_3^{\top} A_{3:0} B_0 & C_3^{\top} A_{3:1} B_1 & C_3^{\top} A_{3:2} B_2 & C_3^{\top} A_{3:3} B_3 \\
    C_4^{\top} A_{4:0} B_0 & C_4^{\top} A_{4:1} B_1 & C_4^{\top} A_{4:2} B_2 & C_4^{\top} A_{4:3} B_3 & C_4^{\top} A_{4:4} B_4 \\
    C_5^{\top} A_{5:0} B_0 & C_5^{\top} A_{5:1} B_1 & C_5^{\top} A_{5:2} B_2 & C_5^{\top} A_{5:3} B_3 & C_5^{\top} A_{5:4} B_4 & C_5^{\top} A_{5:5} B_5 \\
    \hline
    C_6^{\top} A_{6:0} B_0 & C_6^{\top} A_{6:1} B_1 & C_6^{\top} A_{6:2} B_2 & C_6^{\top} A_{6:3} B_3 & C_6^{\top} A_{6:4} B_4 & C_6^{\top} A_{6:5} B_5 & C_6^{\top} A_{6:6} B_6 \\
    C_7^{\top} A_{7:0} B_0 & C_7^{\top} A_{7:1} B_1 & C_7^{\top} A_{7:2} B_2 & C_7^{\top} A_{7:3} B_3 & C_7^{\top} A_{7:4} B_4 & C_7^{\top} A_{7:5} B_5 & C_7^{\top} A_{7:6} B_6 & C_7^{\top} A_{7:7} B_7 \\
    C_8^{\top} A_{8:0} B_0 & C_8^{\top} A_{8:1} B_1 & C_8^{\top} A_{8:2} B_2 & C_8^{\top} A_{8:3} B_3 & C_8^{\top} A_{8:4} B_4 & C_8^{\top} A_{8:5} B_5 & C_8^{\top} A_{8:6} B_6 & C_8^{\top} A_{8:7} B_7 & C_8^{\top} A_{8:8} B_8 \\
\end{bNiceArray}
\\
\\&=
\begin{bNiceArray}{ccc|ccc|ccc}[cell-space-limits=2pt]
    C_0^{\top} A_{0:0} B_0 & \\
    C_1^{\top} A_{1:0} B_0 & C_1^{\top} A_{1:1} B_1 & \\
    C_2^{\top} A_{2:0} B_0 & C_2^{\top} A_{2:1} B_1 & C_2^{\top} A_{2:2} B_2 \\
    \hline
    \Block[fill=[RGB]{243,240,255}]{3-3}{\begin{bmatrix}C_3^{\top} A_{3:2} \\[1pt] C_4^{\top} A_{4:2} \\[1pt] C_5^{\top} A_{5:2}\end{bmatrix}A_{2:2}\begin{bmatrix}B_0^{\top} A_{2:0} \\[1pt] B_1^{\top} A_{2:1} \\[1pt] B_2^{\top} A_{2:2}\end{bmatrix}^{\top}}
                           &&& C_3^{\top} A_{3:3} B_3 \\
                           &&& C_4^{\top} A_{4:3} B_3 & C_4^{\top} A_{4:4} B_4 \\
                           &&& C_5^{\top} A_{5:3} B_3 & C_5^{\top} A_{5:4} B_4 & C_5^{\top} A_{5:5} B_5 \\
    \hline
    \Block[fill=[RGB]{243,240,255}]{3-3}{\begin{bmatrix}C_6^{\top} A_{6:5} \\[1pt] C_7^{\top} A_{7:5} \\[1pt] C_8^{\top} A_{8:5}\end{bmatrix}A_{5:2}\begin{bmatrix}B_0^{\top} A_{2:0} \\[1pt] B_1^{\top}A_{2:1} \\[1pt] B_2^{\top}A_{2:2}\end{bmatrix}^{\top}}
                           &&&
                           \Block[fill=[RGB]{243,240,255}]{3-3}{\begin{bmatrix}C_6^{\top} A_{6:5} \\[1pt] C_7^{\top} A_{7:5} \\[1pt] C_8^{\top} A_{8:5}\end{bmatrix}A_{5:5}\begin{bmatrix}B_3^{\top}A_{5:3} \\[1pt] B_4^{\top}A_{5:4} \\[1pt] B_5^{\top}A_{5:5}\end{bmatrix}^{\top}}
                           &&& C_6^{\top} A_{6:6} B_6 \\
                           &&&&&& C_7^{\top} A_{7:6} B_6 & C_7^{\top} A_{7:7} B_7 \\
                           &&&&&& C_8^{\top} A_{8:6} B_6 & C_8^{\top} A_{8:7} B_7 & C_8^{\top} A_{8:8} B_8 \\
  \end{bNiceArray}
\end{align*}

From here we can reduce the problem into these two parts.
These can also be interpreted as dividing the output of a ``chunk'' $y_{j\mathtt{Q}:(j+1)\mathtt{Q}}$ into two components:
the effect of inputs within the chunk $x_{j\mathtt{Q}:(j+1)\mathtt{Q}}$, and the effect of inputs before the chunk $x_{0:j\mathtt{Q}}$.

\subsection{Diagonal Blocks}
The diagonal blocks are easy to handle, because they are simply self-similar problems of a smaller size.
The $j$-th block represents computing the answer $\mathsf{SSM}(A_R, B_R, C_R)(x_R)$ for the range $R = j\mathtt{Q}:(j+1)\mathtt{Q} = (j\mathtt{Q}, j\mathtt{Q}+1, \dots, j\mathtt{Q}+\mathtt{Q}-1)$.
The key is that this block can be computed using any desired method.
In particular, for small chunk lengths $\mathtt{Q}$, this problem is computed more efficiently using the dual quadratic SMA form.
Additionally, the chunks can be computed in parallel.

These subproblems can be interpreted as: what is the output per chunk \emph{supposing that the initial state (to the chunk) is $0$}.
In other words for chunk $j$, this computes the correct outputs taking into account only the chunk inputs $x_{j\mathtt{Q}:(j+1)\mathtt{Q}}$.

\subsection{Low-Rank Blocks}

The low-rank factorizations consist of 3 terms, and there are correspondingly three pieces of the computation.
In this factorization, we will use the terminology
\begin{itemize}
  \item The terms like $\begin{bmatrix}B_0^{\top} A_{2:0} \\[1pt] B_1^{\top}A_{2:1} \\[1pt] B_2^{\top}A_{2:2}\end{bmatrix}^{\top}$ are called the right factors or $B$-block-factors.
  \item The terms like $A_{5:2}$ are called the center factors or $A$-block-factors.
  \item The terms like $\begin{bmatrix}C_6^{\top} A_{6:5} \\[1pt] C_7^{\top} A_{7:5} \\[1pt] C_8^{\top} A_{8:5}\end{bmatrix}$ are called the left factors or $C$-block-factors.
\end{itemize}

\paragraph{Right Factors.}

This step computes the multiplication by the right $B$-block-factors of the low-rank factorization.
Note that for each chunk, this is a $(\mathtt{N}, \mathtt{Q})$ by $(\mathtt{Q}, \mathtt{P})$ matrix multiplication, where $\mathtt{N}$ is the state dimension and $P$ is the head dimension.
The result is a $(\mathtt{N}, \mathtt{P})$ tensor for each chunk, which has the same dimensionality as the expanded hidden state $h$.

This can be interpreted as: what is the final state per chunk \emph{supposing that the initial state (to the chunk) is $0$}.
In other words this computes $h_{j\mathtt{Q}+\mathtt{Q}-1}$ assuming that $x_{0:j\mathtt{Q}} = 0$.

\paragraph{Center Factors.}

This step computes the effect of the center $A$-block-factors terms in the low-rank factorization.
In the previous step, the final states per chunk have total shape $\mathtt{(\mathtt{T}/\mathtt{Q},\mathtt{N},\mathtt{P})}$.
This is now multiplied by a 1-SS matrix generated by $A_{2\mathtt{Q}-1:\mathtt{Q}-1}^\times, A_{3\mathtt{Q}-1:2\mathtt{Q}-1}^\times, \dots, A_{\mathtt{T}-1:\mathtt{T}-\mathtt{Q}-1}^\times$.

This step can be computed by any algorithm for computing 1-SS multiplication (also known as the scalar SSM scan or \texttt{cumprodsum} operator).

This can be interpreted as: what is the actual final state per chunk \emph{taking into account all previous inputs};
in other words, this computes the true hidden state $h_{j\mathtt{Q}}$ taking into account all of $x_{0:(j+1)\mathtt{Q}}$.

\paragraph{Left Factors.}

This step computes the multiplication by the left $C$-block-factors of the low-rank factorization.
For each chunk, this can be represented by a matrix multiplication
$\mathsf{contract}(\mathtt{QN,NP \to QP})$.

This can be interpreted as: what is the output per chunk \emph{taking into account the correct initial state $h_{j\mathtt{Q}-1}$, and supposing the inputs $x_{j\mathtt{Q}:(j+1)\mathtt{Q}}$ are $0$}.
In other words for chunk $j$, this computes the correct outputs taking into account only the prior inputs $x_{0:j\mathtt{Q}}$.

\begin{listing*}[!t]
\begin{minted}{python}
def segsum(x):
    """Naive segment sum calculation. exp(segsum(A)) produces a 1-SS matrix,
       which is equivalent to a scalar SSM."""
    T = x.size(-1)
    x_cumsum = torch.cumsum(x, dim=-1)
    x_segsum = x_cumsum[..., :, None] - x_cumsum[..., None, :]
    mask = torch.tril(torch.ones(T, T, device=x.device, dtype=bool), diagonal=0)
    x_segsum = x_segsum.masked_fill(~mask, -torch.inf)
    return x_segsum

def ssd(X, A, B, C, block_len=64, initial_states=None):
    """
    Arguments:
        X: (batch, length, n_heads, d_head)
        A: (batch, length, n_heads)
        B: (batch, length, n_heads, d_state)
        C: (batch, length, n_heads, d_state)
    Return:
        Y: (batch, length, n_heads, d_head)
    """
    assert X.dtype == A.dtype == B.dtype == C.dtype
    assert X.shape[1] % block_len == 0

    # Rearrange into blocks/chunks
    X, A, B, C = [rearrange(x, "b (c l) ... -> b c l ...", l=block_len) for x in (X, A, B, C)]

    A = rearrange(A, "b c l h -> b h c l")
    A_cumsum = torch.cumsum(A, dim=-1)

    # 1. Compute the output for each intra-chunk (diagonal blocks)
    L = torch.exp(segsum(A))
    Y_diag  = torch.einsum("bclhn,bcshn,bhcls,bcshp->bclhp", C, B, L, X)

    # 2. Compute the state for each intra-chunk
    # (right term of low-rank factorization of off-diagonal blocks; B terms)
    decay_states = torch.exp((A_cumsum[:, :, :, -1:] - A_cumsum))
    states = torch.einsum("bclhn,bhcl,bclhp->bchpn", B, decay_states, X)

    # 3. Compute the inter-chunk SSM recurrence; produces correct SSM states at chunk boundaries
    # (middle term of factorization of off-diag blocks; A terms)
    if initial_states is None:
        initial_states = torch.zeros_like(states[:, :1])
    states = torch.cat([initial_states, states], dim=1)
    decay_chunk = torch.exp(segsum(F.pad(A_cumsum[:, :, :, -1], (1, 0))))
    new_states = torch.einsum("bhzc,bchpn->bzhpn", decay_chunk, states)
    states, final_state = new_states[:, :-1], new_states[:, -1]

    # 4. Compute state -> output conversion per chunk
    # (left term of low-rank factorization of off-diagonal blocks; C terms)
    state_decay_out = torch.exp(A_cumsum)
    Y_off = torch.einsum('bclhn,bchpn,bhcl->bclhp', C, states, state_decay_out)

    # Add output of intra-chunk and inter-chunk terms (diagonal and off-diagonal blocks)
    Y = rearrange(Y_diag+Y_off, "b c l h p -> b (c l) h p")
    return Y, final_state
\end{minted}
\caption{Full PyTorch example of the state space dual (SSD) model.}
\label{listing}
\end{listing*}

\begin{figure}[!t]
\centering
\includegraphics[width=\linewidth]{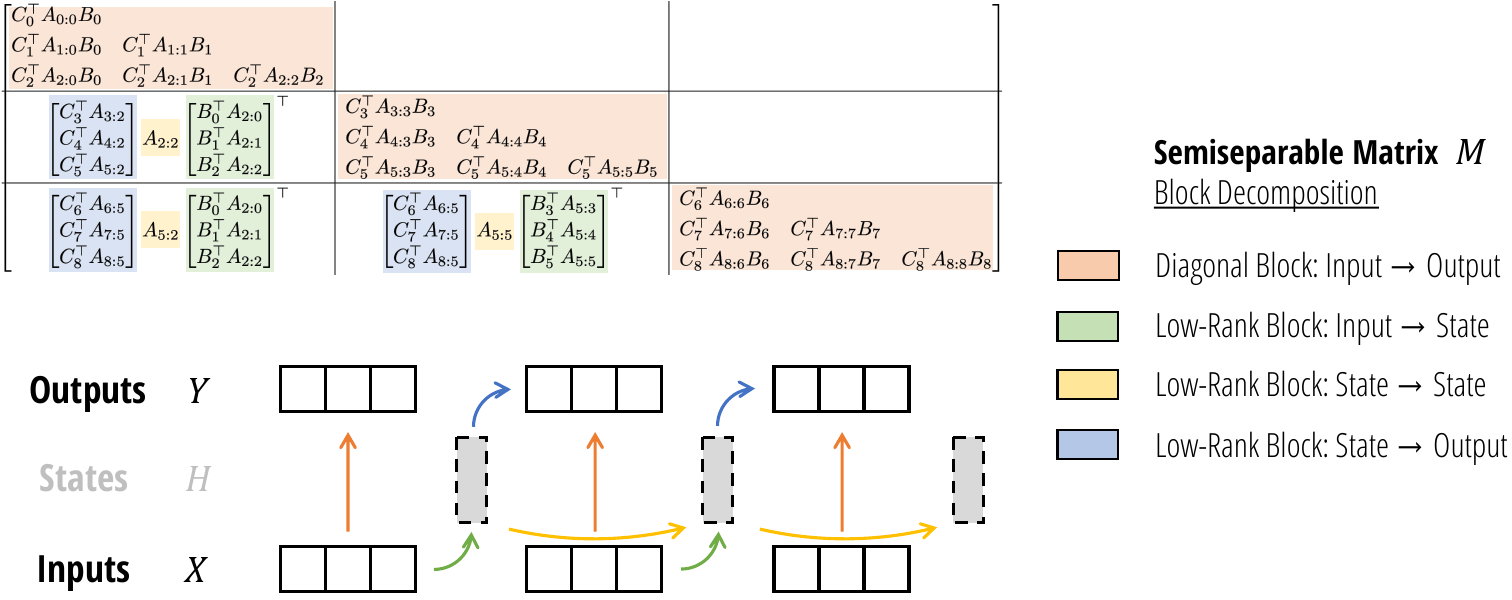}
\caption{
  (\textbf{SSD Algorithm}.)
  By using the matrix transformation viewpoint of state space models to write them as semiseparable matrices (\cref{sec:ssm}), we develop a more hardware-efficient computation of the SSD model through a block-decomposition matrix multiplication algorithm.
  The matrix multiplication also has an interpretation as a state space model,
  where blocks represent chunking the input and output sequence.
  Diagonal blocks represent intra-chunk computations and the off-diagonal blocks represent inter-chunk computations, factored through the SSM's hidden state.
}
\label{fig:ssd-algorithm}
\end{figure}

\subsection{Computational Cost}

We define the notation $\mathsf{BMM}(\mathtt{B}, \mathtt{M}, \mathtt{N}, \mathtt{K})$ to define a batched matrix multiplication $\mathsf{contract}(\mathtt{\mathtt{MK},\mathtt{KN} \to \mathtt{MN}})$ with batch dimension $\mathtt{B}$.
From this notation we can infer three aspects of the efficiency:
\begin{itemize}
  \item \emph{Computation cost}: total of $O(\mathtt{BMNK})$ FLOPs.
  \item \emph{Memory cost:} total of $O(\mathtt{B}(\mathtt{MK}+\mathtt{KN}+\mathtt{MN}))$ space.
  \item \emph{Parallelization:} larger $\mathtt{M}, \mathtt{N}, \mathtt{K}$ terms can leverage specialized matrix multiplication units on modern accelerators.
\end{itemize}

\paragraph{Center Blocks.}
The cost of the quadratic SMA computation consists of three steps (equation \eqref{eq:ssm-quad}):
\begin{itemize}
  \item Computing the kernel matrix $C^{\top} B$, which has cost $\mathsf{BMM}(\mathtt{T}/\mathtt{Q}, \mathtt{Q}, \mathtt{Q}, \mathtt{N})$.
  \item Multiplying by the mask matrix, which is an elementwise operation on tensors of shape $(\mathtt{T}/\mathtt{Q}, \mathtt{Q}, \mathtt{Q})$.
  \item Multiplying by the $X$ values, which has cost $\mathsf{BMM}(\mathtt{T}/\mathtt{Q}, \mathtt{Q}, \mathtt{P}, \mathtt{N})$
\end{itemize}

\paragraph{Low-Rank Blocks: Right Factors.}
This step is a single matrix multiplication with cost $\mathsf{BMM}(\mathtt{T}/\mathtt{Q}, \mathtt{N}, \mathtt{P}, \mathtt{Q})$.

\paragraph{Low-Rank Blocks: Center Factors.}
This step is a scalar SSM scan (or 1-SS multiplication) of length $\mathtt{T}/\mathtt{Q}$ on $(\mathtt{N}, \mathtt{P})$ independent channels.
The work of this scan is $\mathtt{TNP}/\mathtt{Q}$, which is negligible compared to the other factors.

Note that because of the blocking which reduces the length of the sequence from $\mathtt{T}$ to $\mathtt{T}/\mathtt{Q}$,
this scan has $\mathtt{Q}$ times smaller cost than a pure SSM scan (e.g. the selective scan of Mamba).
Thus we observe that on most problem lengths,
other algorithms (\cref{sec:scan}) may be more efficient or much easier to implement without a significant slowdown.
For example, a naive implementation of this via 1-SS matrix multiplication has cost $\mathsf{BMM}(1, \mathtt{T}/\mathtt{Q}, \mathtt{NP}, \mathtt{T}/\mathtt{Q})$,
which is much easier to implement and can be more efficient than a naive recurrence/scan implementation.

\paragraph{Low-Rank Blocks: Left Factors.}
This step is a single matrix multiplication with cost $\mathsf{BMM}(\mathtt{T}/\mathtt{Q}, \mathtt{Q}, \mathtt{P}, \mathtt{N})$.

\paragraph{Total Cost.}
If we set $\mathtt{N} = \mathtt{P} = \mathtt{Q}$ (in other words the state dimension, head dimension, and chunk length are equal),
then all BMM terms above become $\mathsf{BMM}(\mathtt{T}/\mathtt{N}, \mathtt{N}, \mathtt{N}, \mathtt{N})$.
The computational chacteristics of this are:
\begin{itemize}
  \item Total FLOP count of $O(\mathtt{TN}^2)$.
  \item Total memory of $O(\mathtt{TN})$.
  \item The work \emph{consists primarily of matrix multiplications} on matrices of shape $(\mathtt{N}, \mathtt{N})$.
\end{itemize}
Notice that the memory consumption is tight; the inputs and outputs $x, y$ have shape $(\mathtt{T}, \mathtt{P}) = (\mathtt{T}, \mathtt{N})$.
Meanwhile the flop count reflects an extra factor of $\mathtt{N}$, which is cost incurred by the autoregressive state size and is common to all models.

Aside from the matmuls, there is a scalar SSM scan on $\mathtt{NP} = \mathtt{N}^2$ features and sequence length $\mathtt{T}/\mathtt{Q}$.
This has cost $O(\mathtt{T}/\mathtt{Q} \mathtt{N}^2)$ FLOPs and $O(\log(\mathtt{T}/\mathtt{Q}))$ depth.
Although it does not use matrix multiplications, it is still parallelizable and the total work done is negligible compared to the other steps;
this has a negligible cost in our GPU implementation.

\paragraph{Comparison to Pure SSM and Attention Models.}

Quadratic attention is also very hardware efficient by only leveraging matrix multiplications, but has $\mathtt{T}^2 N$ total FLOPs.
Its slower computation speed at both training and inference can directly be seen as a consequence of having a larger state size -- standard attention has a state size scaling with sequence length $\mathtt{T}$ because it caches its history and does not compress its state.

Linear SSMs have $\mathtt{TNP} = \mathtt{TN}^2$ total FLOPs, which is the same as SSD.
However, a naive implementation requires a state expansion \eqref{eq:sha-lin:1} that materializes extra memory,
and a scalar operation \eqref{eq:sha-lin:2} that does not leverage matrix multiplications.

\begin{center}
  \begin{tabular}{@{}llll@{}}
    \toprule
                            & Attention                & SSM                                        & \textbf{SSD} \\
    \midrule
    State size              & $\mathtt{T}$             & $\mathtt{\mathbf{N}}$                      & $\mathtt{\mathbf{N}}$ \\
    Training FLOPs          & $\mathtt{T}^2\mathtt{N}$ & $\mathtt{\mathbf{T}}\mathtt{\mathbf{N}}^2$ & $\mathtt{\mathbf{T}}\mathtt{\mathbf{N}}^2$ \\
    Inference FLOPs         & $\mathtt{T}\mathtt{N}$   & $\mathtt{\mathbf{N}}^2$                    & $\mathtt{\mathbf{N}}^2$ \\
    \midrule
    (Naive) memory          & $\mathtt{T}^2$           & $\mathtt{T}\mathtt{N}^2$                   & $\mathtt{\mathbf{T}}\mathtt{\mathbf{N}}$ \\
    Matrix multiplication   & \cmark                   &                                            & \textbf{\cmark} \\
    \bottomrule
  \end{tabular}
  \label{tab:compute}
\end{center}

We note that many other matrix decompositions are possible (for example, see \cref{sec:scan} for a compendium of algorithms for 1-SS multiplication through different structured matrix decompositions)
which may lead to more algorithms for SSDs that could be better for other specialized settings.
Even more broadly, we note that semiseparable matrices have a rich literature and many more representations besides the SSS form that we use (\cref{def:sss}),
and even more efficient algorithms may be possible.

%% file: structure/architecture.tex
\section{The Mamba-2 Architecture}
\label{sec:architecture}

By connecting SSMs and attention, the SSD framework allows us to develop a shared vocabulary and library of techniques for both.
In this section we discuss some examples of understanding and modifying SSD layers using ideas originally developed for Transformers.
We discuss several design choices, resulting in the Mamba-2 architecture.
These axes of variation are ablated in \cref{sec:experiments:ablations}.

\begin{figure}[!t]
  \centering
  \includegraphics[width=\iftoggle{arxiv}{0.8\linewidth}{\linewidth}]{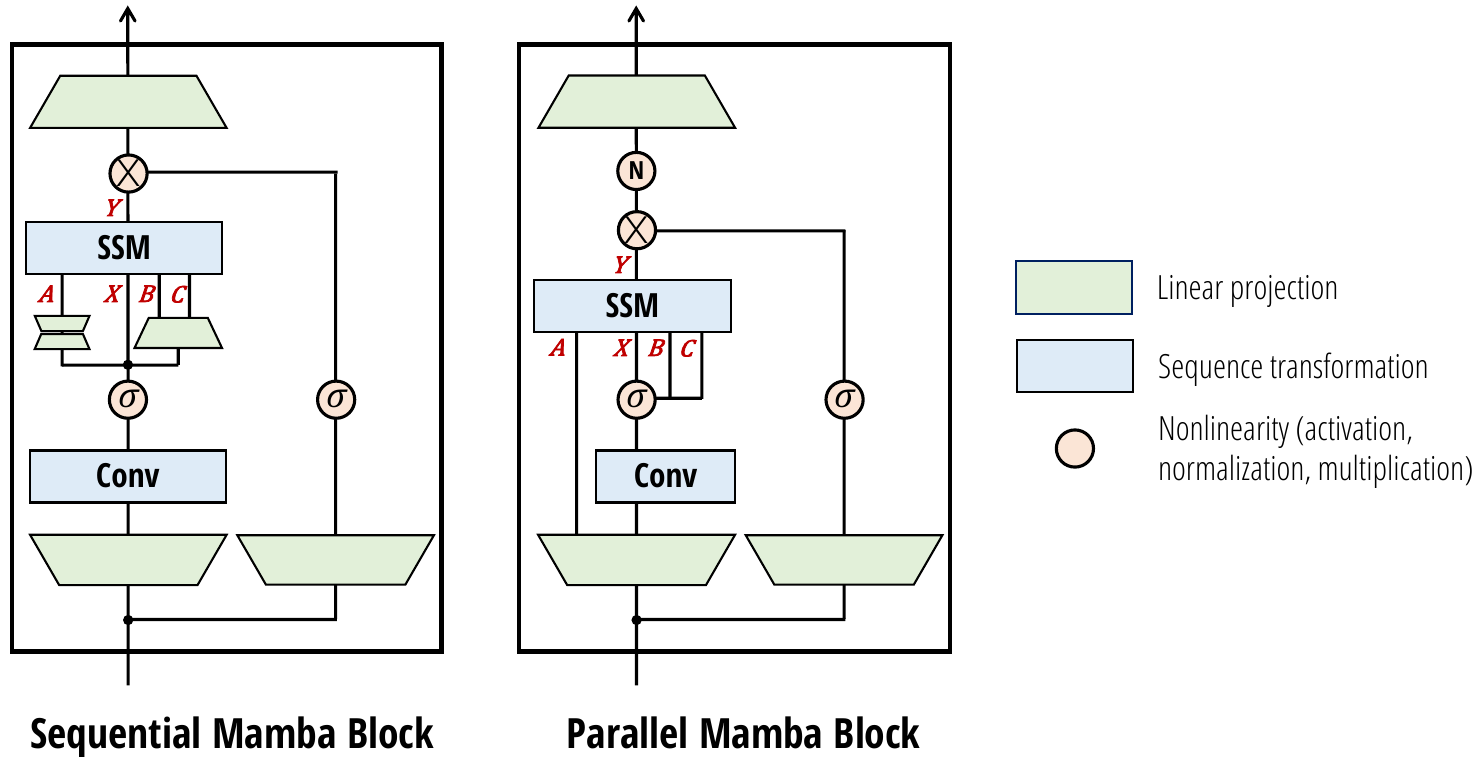}
  \caption{
    (\textbf{Mamba-2 Architecture}.) The Mamba-2 block simplifies the Mamba block by removing sequential linear projections; the SSM parameters $A, B, C$ are produced at the beginning of the block instead of as a function of the SSM input $X$.
    An additional normalization layer is added as in NormFormer~\citep{shleifer2021normformer}, improving stability.
    The $B$ and $C$ projections only have a single head shared across the $X$ heads, analogous to multi-value attention (MVA).
  }
  \label{fig:architecture}
  \iftoggle{arxiv}{}{\vspace{-1.25em}}
\end{figure}

\subsection{Block Design}
\label{sec:architecture:block}

We first discuss modifications to the neural network block that are independent of the inner sequence mixing layer (i.e. outside the core SSD layer).

\paragraph{Parallel Parameter Projections.}

Mamba-1 was motivated by an SSM-centric point of view
where the selective SSM layer is viewed as a map from $X \mapsto Y$.
The SSM parameters $A, B, C$ are viewed as
subsidiary and are functions of the SSM input $X$.
Thus the linear projections defining $(A, B, C)$ occur after the initial linear projection to create $X$.

In Mamba-2, the SSD layer is viewed as a map from $A, X, B, C \mapsto Y$.
It therefore makes sense to produce $A, X, B, C$ in parallel with a single projection at the beginning of the block.
Note the analogy to standard attention architectures,
where $X, B, C$ correspond to the $Q, K, V$ projections that are created in parallel.

Note that adopting parallel projections for the $A, B, C, X$ inputs to the SSM slightly reduces parameters and more importantly is more amenable to tensor parallelism for larger models, by using standard Megatron sharding patterns~\citep{shoeybi2019megatron}).

\paragraph{Extra Normalization.}

In preliminary experiments, we found that instabilities were prone to arising in larger models.
We were able to alleviate this by adding an extra normalization layer (e.g. LayerNorm, GroupNorm, or RMSNorm) to the block right before the final output projection.
This usage of a normalization is most directly related to the NormFormer architecture~\citep{shleifer2021normformer}, which also added normalization layers at the end of the MLP and MHA blocks.

We also note that this change is similar to other recent models related to Mamba-2 that were derived from a linear attention viewpoint.
The original linear attention formulation normalizes by a denominator term that emulates the normalization of the softmax function in standard attention.
TransNormerLLM~\citep{qin2023transnormerllm} and RetNet~\citep{sun2023retentive} find that this normalization is unstable and add an extra LayerNorm or GroupNorm after the linear attention layer.
Our extra normalization layer differs slightly from these, occuring after the multiplicative gate branch instead of before.

\subsection{Multihead Patterns for Sequence Transformations}
\label{sec:architecture:multihead}

Recall that SSMs are defined as a sequence transformation (\cref{def:sequence-transformation})
where:
\begin{itemize}
  \item $A, B, C$ parameters have a state dimension $\mathtt{N}$.
  \item They define a sequence transformation $\R^\mathtt{T} \to \R^\mathtt{T}$, which for example can be represented as a matrix $M \in \R^{\mathtt{(T, T)}}$.
  \item This transformation operates over an input sequence $X \in \R^{\mathtt{(T, P)}}$, independently over the $\mathtt{P}$ axis.
\end{itemize}
One can view this as defining one \emph{head} of the sequence transformation.
\begin{definition}[Multihead patterns]
  \label{def:head-pattern}
  A multihead sequence transformation consists of $\mathtt{H}$ independent heads, for a total model dimension of $\mathtt{D}=\mathtt{d\_model}$.
  The parameters may be tied across heads, leading to a \textbf{head pattern}.
\end{definition}
The state size $\mathtt{N}$ and head dimension $\mathtt{P}$ are analogous to the $QK$ head dimension and $V$ head dimension of attention, respectively.
Just as in modern Transformer architectures~\citep{chowdhery2022palm,touvron2023llama},
in Mamba-2 we generally choose these to be constants around $64$ or $128$; when the model dimension $\mathtt{D}$ increases, we increase the number of heads while keeping the head dimensions $\mathtt{N}$ and $\mathtt{P}$ fixed.
In order to describe how to do this,
we can transfer and generalize ideas from multihead attention to define similar patterns for SSMs, or any general sequence transformation.

\begin{minipage}[t]{.25\linewidth}
  \begin{equation}%
    \label{eq:multihead}
    \begin{aligned}%
        & \textbf{Multi-head SSM} \\
        & (\textrm{Multi-head Attn.}) \\
      X & \quad \mathtt{(T, H, P)} \\
      A & \quad \mathtt{(T, H)} \\
      B & \quad \mathtt{(T, H, N)} \\
      C & \quad \mathtt{(T, H, N)} \\
    \end{aligned}
  \end{equation}
\end{minipage}
\begin{minipage}[t]{.25\linewidth}
  \begin{equation}%
    \label{eq:multiquery}
    \begin{aligned}%
    & \textbf{Multi-contract SSM} \\
    & (\textrm{Multi-query Attn.}) \\
      X & \quad \mathtt{(T, 1, P)} \\
      A & \quad \mathtt{(T, H)} \\
      B & \quad \mathtt{(T, 1, N)} \\
      C & \quad \mathtt{(T, H, N)} \\
    \end{aligned}
  \end{equation}
\end{minipage}
\begin{minipage}[t]{.25\linewidth}
  \begin{equation}%
    \label{eq:multikey}
    \begin{aligned}%
    & \textbf{Multi-expand SSM} \\
    & (\textrm{Multi-key Attn.}) \\
      X & \quad \mathtt{(T, 1, P)} \\
      A & \quad \mathtt{(T, H)} \\
      B & \quad \mathtt{(T, H, N)} \\
      C & \quad \mathtt{(T, 1, N)} \\
    \end{aligned}
  \end{equation}
\end{minipage}
\begin{minipage}[t]{.25\linewidth}
  \begin{equation}%
    \label{eq:multivalue}
    \begin{aligned}%
    & \textbf{Multi-input SSM} \\
    & (\textrm{Multi-value Attn.}) \\
      X & \quad \mathtt{(T, H, P)} \\
      A & \quad \mathtt{(T, H)} \\
      B & \quad \mathtt{(T, 1, N)} \\
      C & \quad \mathtt{(T, 1, N)} \\
    \end{aligned}
  \end{equation}
\end{minipage}

\paragraph{Multihead SSM (MHS) / Multihead Attention (MHA) Pattern.}

The classic MHA pattern assumes that the head dimension $\mathtt{P}$ divides the model dimension $\mathtt{D}$.
The number of heads is defined as $\mathtt{H} = \mathtt{D} / \mathtt{P}$.
Then, $\mathtt{H}$ copies of the core sequence transformation are created by creating $\mathtt{H}$ independent copies of each parameter.
Note that while the MHA pattern was first described for the attention sequence transformation,
it can be applied to anything compatible with \cref{def:sequence-transformation}.
For example, a multi-head SSD layer would accept inputs with shapes according to equation \eqref{eq:multihead}
where the SSD algorithm is broadcasted over the $\mathtt{H} = \mathtt{n\_heads}$ dimension.

\paragraph{Multi-contract SSM (MCS) / Multi-query Attention (MQA) Pattern.}

Multi-query attention~\citep{shazeer2019fast} is a clever optimization for attention that can dramatically improve the speed of autoregressive inference,
which relies on caching the $K$ and $V$ tensors.
This technique simply avoids giving $K$ and $V$ the extra head dimension,
or in other words broadcasts a single head of $(K, V)$ across all the heads of $Q$.

Using the state space duality, we can define an equivalent SSM version of MQA as equation \eqref{eq:multiquery}.
Here, $X$ and $B$ (the SSM analogs of attention's $V$ and $K$) are shared across the $\mathtt{H}$ heads.
We also call this the \emph{multi-contract SSM (MCS)} head pattern, because the $C$ parameter which controls the SSM state contraction has independent copies per head.

We can similarly define a multi-key attention (MKA) or \emph{multi-expand SSM (MES)} head pattern,
where $B$ (which controls the SSM expansion) is independent per head while $C$ and $X$ are shared across heads.

\paragraph{Multi-input SSM (MIS) / Multi-value Attention (MVA) Pattern.}

While MQA makes sense for attention because of its KV cache, it is not the natural choice for SSMs.
In Mamba, instead, $X$ is viewed as the main input to the SSM,
and therefore $B$ and $C$ are parameters that are shared across the input channels.
We define a new multi-value attention (MVA) of \emph{multi-input SSM (MIS)} pattern in equation \eqref{eq:multivalue}, which can again be applied to any sequence transformation such as SSD.

Armed with this vocabulary, we can characterize the original Mamba architecture more precisely.
\begin{proposition}
  \label{prop:mamba-multihead}
  The selective SSM (S6) layer of the Mamba architecture~\citep{gu2023mamba} can be viewed as having
  \begin{itemize}
    \item Head dimension $P=1$: every channel has independent SSM dynamics $A$.
    \item \emph{Multi-input SSM} (MIS) or \emph{multi-value attention} (MVA) head structure: the $B, C$ matrices (corresponding to $K, Q$ in the attention duality) are shared across all channels of the input $X$ (corresponding to $V$ in attention).
  \end{itemize}
\end{proposition}

We can also ablate these head pattern variants when applied to SSD (\cref{sec:experiments:ablations:kernels}).
Interestingly, despite being controlled in parameter counts and total state dimension,
there is a noticeable difference in downstream performance.
We empirically find that the MVA pattern as originally used in Mamba performs best.

\paragraph{Grouped Head Patterns.}

The ideas of multi-query attention can be extended to \emph{grouped-query attention}~\citep{ainslie2023gqa}: instead of $1$ K and V head, one can create $\mathtt{G}$ independent K and V heads, where $1 < \mathtt{G}$ and $\mathtt{G}$ divides $\mathtt{H}$.
This is motivated both by bridging the performance difference between multi-query and multi-head attention,
and enabling more efficient tensor parallelism by setting $\mathtt{G}$ to be a multiple of the number of shards (\cref{sec:systems}).

Similarly, the multi-input SSM head pattern used in Mamba-2 can be easily extended to \textbf{grouped-input SSM (GIS)}, or synonymously \textbf{grouped-value attention (GVA)}.
The generalization is straightforward and we omit the details for simplicity.

\subsection{Other SSD Extensions from Linear Attention}
\label{sec:architecture:kernels}

We describe here an example of architectural modifications to SSD motivated by linear attention.
We ablate these in \cref{sec:experiments:ablations:kernels} as a form of negative result,
finding that they do not significantly improve performance enough to adopt them as default settings.
Nonetheless, these illustrate how the vast literature on attention can be incorporated to
define variants of SSD.
We treat the choice of kernel feature map as a hyperparameter in the Mamba-2 architecture,
and expect other simple modifications inspired by attention to be possible as well.

\paragraph{Kernel Attention Approximations to Softmax Attention.}
Many variants of linear attention or kernel attention are motivated by viewing the attention scores $\mathsf{softmax}(QK^{\top})$ as composed of
\begin{enumerate}
  \item An exponential kernel $Z = \exp(QK^{\top})$, which can be approximated by $Z = \psi(Q)\psi(K)^{\top}$ for some kernel feature map.
  \item Normalizing the kernel so that rows sum to $1$ via $M = G / G \bm{1} \bm{1}^{\top}$, where the division happens elementwise and $\bm{1}$ is the all 1's vector.
\end{enumerate}

\paragraph{Exponential Kernel Feature Maps.}

In Mamba-2, we incorporate a flexible kernel feature map, and apply it to the $B$ and $C$ branches (corresponding to the $K$ and $V$ branches in attention).
The feature map can also be optionally applied to the $X$ ($V$) branch, for simplicity and symmetry.
This is represented in \cref{fig:architecture} by an arbitrary nonlinearity.
By default, we simply choose $\psi$ to be an elementwise Swish / SiLU function~\citep{hendrycks2016gaussian,ramachandran2017swish}.
We explore other options in the ablations in \cref{sec:experiments:ablations:kernels},
including feature maps used by Linear Attention, Performer, Random Feature Attention, and cosFormer (\cref{sec:attention:kernel}).

\paragraph{Incorporating a Normalization (Denominator) Term.}

To find the denominator term, we simply have to compute $M \bm{1}$.
But recall that the final output of the model is just $Y = MX$ (equation \eqref{eq:ssm-quad}).
So the normalization terms can be found
simply by augmenting $X$ with an extra column $\bm{1}$, resulting in a tensor of shape $\mathtt{(T, P+1)}$.

Note that in this case, the kernel feature map $\psi$ must be positive so that the sum is positive.

%% file: structure/systems.tex
\section{Systems Optimization for SSMs}
\label{sec:systems}

We describe several systems optimizations for SSMs, in particular the Mamba-2 architecture, for large-scale efficient training and inference.
In particular, we focus on tensor parallel and sequence parallel for large-scale training, as a well variable-length sequences for efficient finetuning and inference.

\subsection{Tensor Parallel}
\label{subsec:tp}

Tensor parallelism (TP)~\citep{shoeybi2019megatron} is a model parallelism technique that splits each layer (e.g., attention, MLP) to run on multiple accelerators such as GPUs.
This technique is widely used to train most large models~\citep{brown2020language, chowdhery2022palm, touvron2023llama, touvron2023llama2}  on GPU clusters where each node typically has 4-8 GPUs with fast networking such as NVLink.
TP was originally developed for the Transformer architecture, and it is not straight-forward to adapt it other architecture.
We first show the challenge of using TP with the Mamba architecture, and the show how the Mamba-2 architecture is designed to make TP efficient.

Recall the Mamba architecture, with a single input $u \in \mathbb{R}^{L \times d}$ (no batching for simplicity), input projection matrices $W^{(x)}, W^{(z)} \in \mathbb{R}^{d \times ed}$ where $e$ is the expansion factor (typically 2), and output projection matrix $W^{(o)} \in \mathbb{R}^{ed \times d}$:
\begin{align*}
   x &= u {W^{(x)}}^\top \in \mathbb{R}^{L \times ed} \\
   z &= u {W^{(z)}}^\top \in \mathbb{R}^{L \times ed} \\
   x_c &= \mathrm{conv1d}(x) \in \mathbb{R}^{L \times ed} \quad \text{(depthwise, independent along $d$)} \\
   \Delta, B, C &= \text{low-rank projection}(x_c) \\
   y &= SSM_{A, B, C, \Delta}(x_c) \in \mathbb{R}^{L \times ed} \quad \text{(independent along $d$)} \\
   y_g &= y \cdot \phi(z)  \quad \text{(gating, e.g., with $\phi$ being SiLU)} \\
  \mathrm{out} &= y_g {W^{(o)}}^\top \in \mathbb{R}^{L \times d}.
\end{align*}
With TP, suppose that we want to split the computation along 2 GPUs. It is easy to split the input projection matrices $W^{(x)}$ and $W^{(z)}$ into two partitions each of size $d \times \frac{ed}{2}$.
Then each GPU would hold half of $x_c$ of size $L \times \frac{ed}{2}$.
However, we see that since $\Delta, B, C$ are functions are $x_c$, so we would need an extra all-reduce between the GPUs to get the whole of $x_c$ before computing $\Delta, B, C$.
After that the two GPUs can compute the SSM in parallel since they are independent along $d$.
At the end, we can split the output projection matrices $W^{(o)}$ into two partitions each of size $\frac{ed}{2} \times d$, and do an all-reduce at the end.
Compared to Transformers, we would incur two all-reduces instead of one, doubling the time spent in communication. For large-scale Transformers training, communication might already take a significant fraction of time (e.g. 10-20\%), and doubling communication would make Mamba not as efficient for large-scale training.

With Mamba-2, our goal is to have only one all-reduce per block, similar to attention or MLP blocks in Transformers.
As a result, we have the projection to get $\Delta, B, C$ directly from $u$ instead of from $x_c$, allowing us to split these projection matrices.
This implies that we have different sets of $\Delta, B, C$ on different GPUs, which is equivalent to having several ``groups'' of $\Delta, B, C$ on a larger ``logical GPU''.
Moreover, we use GroupNorm within each block, with number of groups divisible by the TP degree, so that the GPUs in a TP group do not have a communicate within the block:
\begin{align*}
   x &= u {W^{(x)}}^\top \in \mathbb{R}^{L \times ed} \\
   z &= u {W^{(z)}}^\top \in \mathbb{R}^{L \times ed} \\
   \Delta, B, C &= \text{projection}(u) \quad \text{(one or more groups of $\Delta, B, C$ per GPU)} \\
   x_c &= \mathrm{conv1d}(x) \in \mathbb{R}^{L \times ed} \quad \text{(depthwise, independent along $d$)} \\
   y &= SSM_{A, B, C, \Delta}(x_c) \in \mathbb{R}^{L \times ed} \quad \text{(independent along $d$)} \\
   y_g &= y \cdot \phi(z)  \quad \text{(gating, e.g., with $\phi$ being SiLU)} \\
   y_n &= \mathrm{groupnorm}(y_g) \quad \text{(number of groups divisible by degree of tensor parallel)} \\
   \mathrm{out} &= y_g {W^{(o)}}^\top \in \mathbb{R}^{L \times d}.
\end{align*}

We see that we only need to split the input projection matrices, and the output projection matrices, and only need to do all-reduce at the end of the block. This is similar to the design of TP for attention and MLP layers.
In particular, if we have TP degree 2, we would split $W^{(x)} = [W^{(x)}_1, W^{(x)}_2]$ with $W^{(x)}_i \in \mathbb{R}^{d \times ed/2}$,
$W^{(z)} = [W^{(z)}_1, W^{(z)}_2]$ with $W^{(z)}_i \in \mathbb{R}^{d \times ed/2}$,
and $W^{(o)} = \begin{bmatrix} W^{(o)}_1 \\ W^{(o)}_2 \end{bmatrix}$ with $W^{(o)}_i \in \mathbb{R}^{ed/2 \times d}$.
For $i = 1, 2$, the TP Mamba-2 layer can be written as:
\begin{align*}
   x^{(i)} &= u {W^{(x)}_i}^\top \in \mathbb{R}^{L \times ed / 2} \\
   z^{(i)} &= u {W^{(z)}_i}^\top \in \mathbb{R}^{L \times ed / 2} \\
   \Delta^{(i)}, B^{(i)}, C^{(i)} &= \text{projection}(u) \quad \text{(one or more groups of $\Delta, B, C$ per GPU)} \\
   x_c^{(i)} &= \mathrm{conv1d}(x^{(i)}) \in \mathbb{R}^{L \times ed / 2} \\
   y^{(i)} &= SSM_{A, B, C, \Delta}(x_c^{(i)}) \in \mathbb{R}^{L \times ed/2}  \\
   y_g^{(i)} &= y^{(i)} \cdot \phi(z^{(i)})  \\
   y_n^{(i)} &= \mathrm{groupnorm}(y_g^{(i)}) \quad \text{(number of groups divisible by degree of tensor parallel)} \\
   \mathrm{out}^{(i)} &= y_g^{(i)} {W^{(o)}_i}^\top \in \mathbb{R}^{L \times d / 2} \\
   \mathrm{out} &= \sum_i \mathrm{out}^{(i)}. \quad \text{(summing outputs from all GPUs with an all-reduce)}
\end{align*}
We illustrate tensor parallel with Mamba-2 in~\cref{fig:mamba2_parallelism} (\emph{Left}).
\begin{figure}[!t]
\centering
\begin{minipage}{.4\linewidth}%
  \centering
  \includegraphics[width=\linewidth]{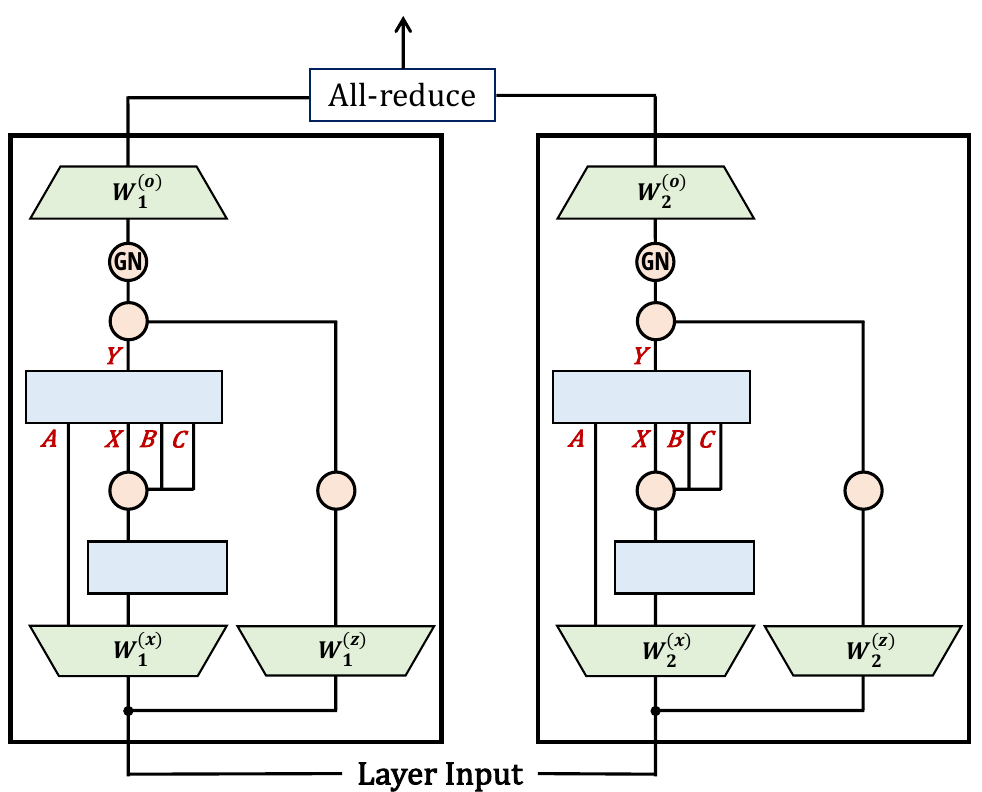}
\end{minipage}
\hfill
\begin{minipage}{.59\linewidth}%
  \centering
  \includegraphics[width=\linewidth]{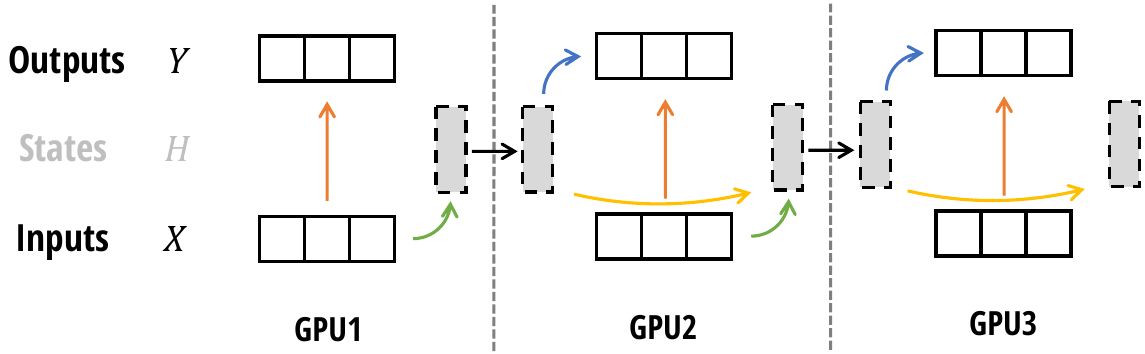}
\end{minipage}
\caption{
  (\textbf{Parallelism with the Mamba-2 Block}.)
  (\emph{Left}: \textbf{Tensor Parallelism})
  We split the input projection matrices $W^{(x)}, W^{(z)}$ and the output projection matrix $W^{(o)}$.
  Each SSM head $(A, B, C, X) \mapsto Y$ lives on a single device.
  Choosing GroupNorm for the final normalization layer avoids extra communication.
  We need one all-reduce per layer, just like the MLP or attention blocks in a Transformer.
  (\emph{Right}: \textbf{Sequence/Context Parallelism})
  Analogous to the SSD algorithm, with multiple devices, we can split along the sequence dimension. Each device computes the state of its sequence, then pass that state to the next GPU.
}
\label{fig:mamba2_parallelism}
\end{figure}

\subsection{Sequence Parallelism}
\label{subsec:sp}

For very long sequences, we might need to split the input and activation to different GPUs along the sequence length dimension.
There are two main techniques:
\begin{enumerate}
\item Sequence parallelism (SP) for the residual and normalization operations: first proposed by~\citet{korthikanti2023reducing}, this technique decomposes the all-reduce in TP as reduce-scatter and all-gather. Noticing that the residual and normalization operations are repeated on the same input for all GPUs in the same TP group, SP splits the activations along the sequence length dimension by performing: reduce-scatter, residual and normalization, then all-gather.

Since the Mamba-2 architecture uses the same residual and normalization structure, SP applies without modification.

\item Sequence parallelism for the token-mixing operations (attention or SSM), also known as ``context parallelism'' (CP).
Several techniques have been developed for attention layer (e.g., Ring attention~\citep{liu2023ring, liu2024world}), with sophisticated load-balancing technique~\citep{brandon2023striped}.
The difficulty with sequence parallelism in attention is that we can split queries and keys into block, but each query block needs to interact with key blocks, leading to communication bandwidth quadratic in the number of workers.

With SSMs, we can split the sequence in a simple manner: each worker takes an initial state, compute the SSM with respect to their inputs, return the final state, and pass that final state to the next worker.
The communication bandwidth is linear in the number of workers.
This decomposition is exactly the same as the block-decomposition in the SSD algorithm (\cref{fig:ssd-algorithm}) to split into blocks / chunks.
We illustrate this context parallelism in~\cref{fig:mamba2_parallelism} (\emph{Right}).

\end{enumerate}

\subsection{Variable Length}
\label{subsec:varlen}

While pretraining often uses the same sequence lengths for the batch, during finetuning or inference, the model might need to process different input sequences of different lengths.
One naive way to handle this case is to right-pad all sequences in the batch to the maximum length, but this can be inefficient if sequences are wildly different lengths.
For transformers, sophisticated techniques have been develop to avoid padding and do load-balancing between GPUs~\citep{zeng2022boosting, zhai2023bytetransformer}, or packing multiple sequences in the same batch and adjust the attention mask~\citep{ding2024fewer, pouransari2024dataset}.
With SSMs and Mamba in particular, we can handle variable sequence lengths by simply treating the whole batch as one long sequence, and avoid passing the states between individual sequences.
This is equivalent to simply setting $A_t = 0$ for tokens $t$ at the end of one sequence to prevent it from passing information to the token $t + 1$, which belongs to a different sequence.

%% file: structure/experiments.tex
\section{Empirical Validation}
\label{sec:experiments}

We empirically evaluate Mamba-2 on synthetic recall tasks that have been challenging for recurrent models (\cref{sec:experiments:mqar}), and standard language modeling pre-training and downstream evaluations (\cref{sec:experiments:lm}).
We validate that our SSD algorithm is much more efficient than Mamba-1 (\cref{sec:experiments:benchmark}) and comparable to optimized attention for moderate sequence lengths.
\iftoggle{arxiv}{
Finally, we ablate various design choices in the Mamba-2 architecture (\cref{sec:experiments:ablations}).
}{}

\begin{figure*}[!t]
  \centering
  \includegraphics[width=\linewidth]{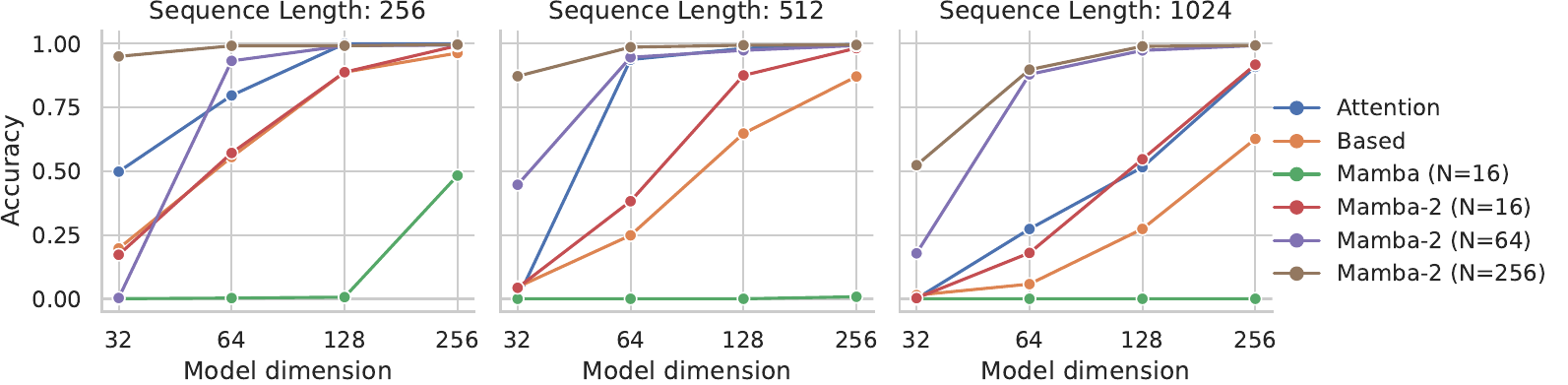}
  \caption{
    (\textbf{Multi-Query Associative Recall (MQAR)}).
    Associative recall tasks are challenging for SSMs, which must memorize all relevant information into their recurrent state.
    The SSD layer combined with improved architecture allows for much larger state sizes in Mamba-2,
    which performs significantly better than Mamba-1 and even vanilla attention.
  }
  \label{fig:mqar}
\end{figure*}

\begin{figure}[!t]
  \centering
    \centering
    \includegraphics[width=\iftoggle{arxiv}{0.7\linewidth}{\linewidth}]{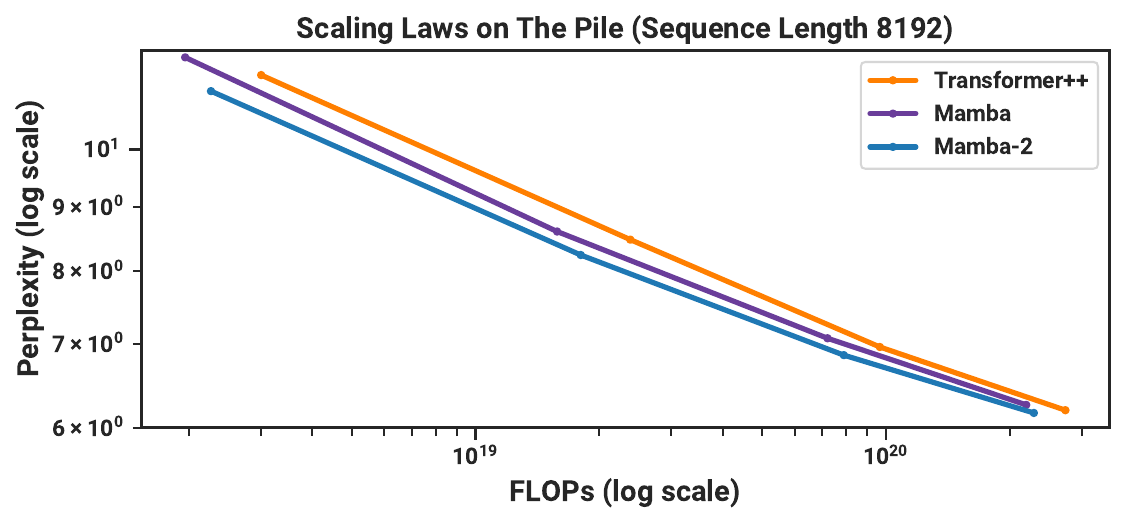}
  \caption{
    (\textbf{Scaling Laws}.) %
    Models of size $\approx 125M$ to $\approx 1.3B$ parameters, trained on the Pile.
    Mamba-2 matches or exceeds the performance of Mamba as well as a strong ``Transformer++'' recipe.
    Compared to our Transformer baseline, Mamba-2 is Pareto dominant on performance (perplexity), theoretical FLOPs, and actual wall-clock time.
  }
  \label{fig:lm-scaling}
  \iftoggle{arxiv}{}{\vspace{-1em}}
\end{figure}

\begin{table*}[!ht]
  \small
  \centering
  \captionsetup{font=small}
  \caption{
    (\textbf{Zero-shot Evaluations}.) Best results for each size in bold, second best unlined.
    We compare against open source LMs with various tokenizers, trained for up to 300B tokens.
    Pile refers to the validation split, comparing only against models trained on the same dataset and tokenizer (GPT-NeoX-20B).
    For each model size, Mamba-2 outperforms Mamba, and generally matches Pythia at twice the model size.
    Full results in \cref{table:downstream_zeroshot_full}.
  }
  \resizebox{0.99\linewidth}{!}
  {
    \begin{tabular}{@{}llllllllllll@{}}
      \toprule
      \sc{Model}                     & \sc{Token.} & \sc{Pile}          & \sc{LAMBADA}       & \sc{LAMBADA}       & \sc{HellaSwag}     & \sc{PIQA}          & \sc{Arc-E}         & \sc{Arc-C}           & \sc{WinoGrande}    & \sc{OpenbookQA}      & \sc{Average} \\
                                     &             & \sc{ppl $\downarrow$}       & \sc{ppl $\downarrow$}       & \sc{acc $\uparrow$}       & \sc{acc $\uparrow$}       & \sc{acc $\uparrow$}       & \sc{acc $\uparrow$}       & \sc{acc $\uparrow$}         & \sc{acc $\uparrow$}       & \sc{acc $\uparrow$}         & \sc{acc $\uparrow$} \\
                                        \midrule
      Pythia-1B                      & NeoX        & $7.82$             & $7.92$             & $56.1$             & $47.2$             & $70.7$             & $57.0$             & $27.1$               & $53.5$             & $31.4$               & $49.0$ \\
      Mamba-790M                     & NeoX        & $\underline{7.33}$ & $\underline{6.02}$ & $\mathbf{62.7}$    & $\mathbf{55.1}$    & $\mathbf{72.1}$    & $\mathbf{61.2}$    & $\mathbf{29.5}$      & $\underline{56.1}$ & $\underline{34.2}$   & $\underline{53.0}$ \\
      \textbf{Mamba-2-780M}          & NeoX        & $\mathbf{7.26}$    & $\mathbf{5.86}$    & $\underline{61.7}$ & $\underline{54.9}$ & $\underline{72.0}$ & $\underline{61.0}$ & $\underline{28.5}$   & $\mathbf{60.2}$    & $\mathbf{36.2}$      & $\mathbf{53.5}$ \\
      \midrule
      Hybrid H3-1.3B                 & GPT2        & ---                  & $11.25$            & $49.6$             & $52.6$             & $71.3$             & $59.2$             & $28.1$               & $56.9$             & $34.4$               & $50.3$ \\
      Pythia-1.4B                    & NeoX        & $7.51$             & $6.08$             & $61.7$             & $52.1$             & $71.0$             & $60.5$             & $28.5$               & $57.2$             & $30.8$               & $51.7$ \\
      RWKV4-1.5B                     & NeoX        & $7.70$             & $7.04$             & $56.4$             & $52.5$             & $72.4$             & $60.5$             & $29.4$               & $54.6$             & $34.0$               & $51.4$ \\
      Mamba-1.4B                     & NeoX        & $\underline{6.80}$ & $\underline{5.04}$ & $\underline{65.0}$ & $\underline{59.1}$ & $\mathbf{74.2}$    & $\mathbf{65.5}$    & $$\underline{32.8}$$ & $\mathbf{61.5}$    & $$\underline{36.4}$$ & $\mathbf{56.4}$ \\
      \textbf{Mamba-2-1.3B}          & NeoX        & $\mathbf{6.66}$    & $\mathbf{5.02}$    & $\mathbf{65.7}$    & $\mathbf{59.9}$    & $\underline{73.2}$ & $\underline{64.3}$ & $\mathbf{33.3}$      & $\underline{60.9}$ & $\mathbf{37.8}$      & $\mathbf{56.4}$ \\
      \midrule
      Hybrid H3-2.7B                 & GPT2        & ---                  & $7.92$             & $55.7$             & $59.7$             & $73.3$             & $65.6$             & $32.3$               & $61.4$             & $33.6$               & $54.5$ \\
      Pythia-2.8B                    & NeoX        & $6.73$             & $5.04$             & $64.7$             & $59.3$             & $74.0$             & $64.1$             & $32.9$               & $59.7$             & $35.2$               & $55.7$ \\
      RWKV4-3B                       & NeoX        & $7.00$             & $5.24$             & $63.9$             & $59.6$             & $73.7$             & $67.8$             & $33.1$               & $59.6$             & $37.0$               & $56.4$ \\
      Mamba-2.8B                     & NeoX        & $\underline{6.22}$ & $\underline{4.23}$ & $\underline{69.2}$ & $\underline{66.1}$ & $\underline{75.2}$ & $\mathbf{69.7}$    & $\underline{36.3}$   & $\underline{63.5}$ & $\mathbf{39.6}$      & $\underline{59.9}$ \\
      \textbf{Mamba-2-2.7B}          & NeoX        & $\mathbf{6.09}$    & $\mathbf{4.10}$    & $\mathbf{69.7}$    & $\mathbf{66.6}$    & $\mathbf{76.4}$    & $\underline{69.6}$ & $\mathbf{36.4}$      & $\mathbf{64.0}$    & $\underline{38.8}$   & $\mathbf{60.2}$ \\
      \bottomrule
    \end{tabular}
  }
  \label{table:downstream_zeroshot}
\end{table*}

\begin{figure*}[!ht]
  \centering
  \begin{subfigure}{.5\textwidth}
    \centering
    \includegraphics[width=.95\linewidth]{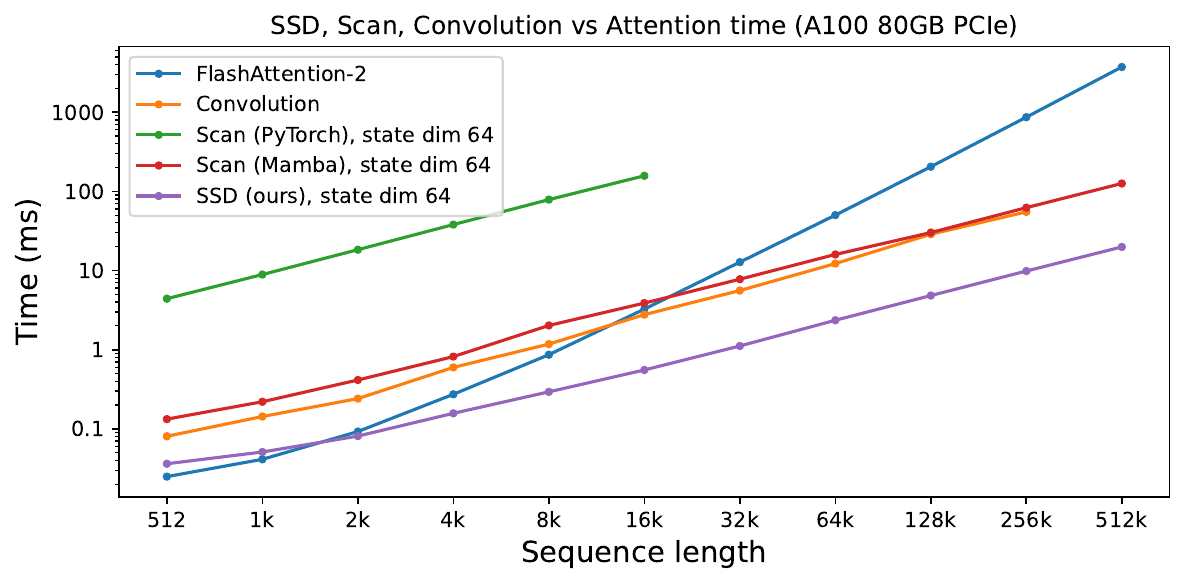}
  \end{subfigure}%
  \begin{subfigure}{.5\textwidth}
    \centering
    \includegraphics[width=.95\linewidth]{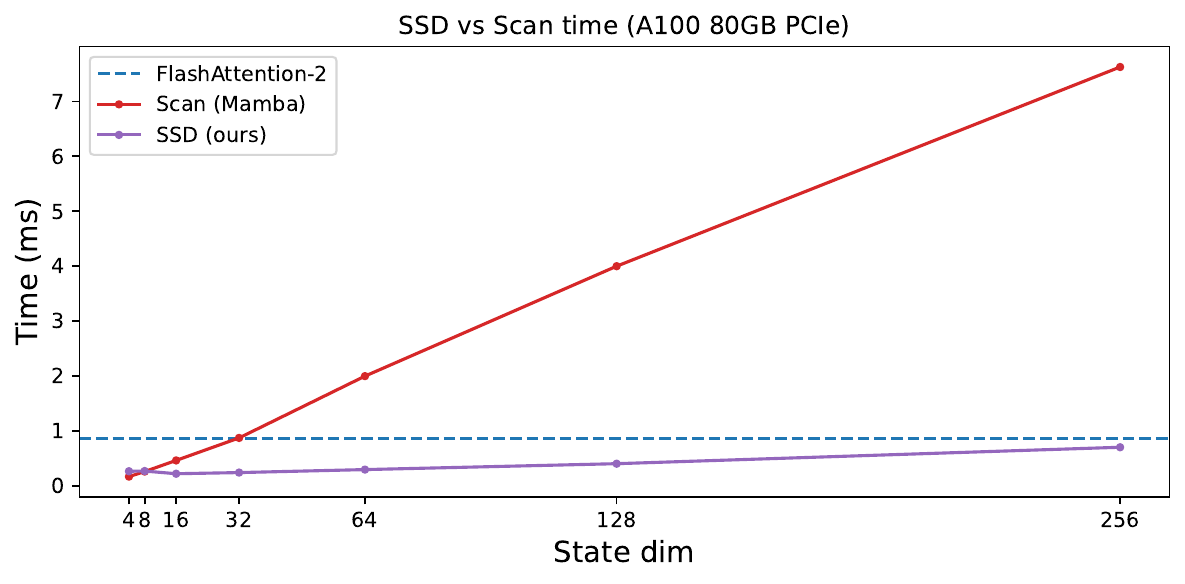}
  \end{subfigure}
  \caption{
    (\textbf{Efficiency Benchmarks}.)
    (\emph{Left}) Our SSD is $2-8\times$ faster than a Mamba fused scan for large state expansion ($N = 64$) and faster than FlashAttention-2 for sequence length 2k and above.
    (\emph{Right}) Sequence length 4K: Increasing state expansion slows down the Mamba optimized scan implementation linearly. SSD can handle much larger state expansion factors without much slowdown.
    \iftoggle{arxiv}{}{\vspace*{-0.75em}}
  }
  \label{fig:scan_benchmark}
\end{figure*}

\subsection{Synthetics: Associative Recall}
\label{sec:experiments:mqar}

Synthetic associative recall tasks have been popular for testing the ability of language models to look up information in their context.
Broadly, they involve feeding autoregressive models pairs of key-value associations, and then prompting the model to produce the correct completion upon being shown a previously-seen key.
The \textbf{multi-query associative recall (MQAR)} task is a particular formulation of this task that requires the model to memorize multiple associations~\citep{arora2024zoology}.
\iftoggle{arxiv}{
The original Mamba paper reported results on related synthetic tasks, in particular Selective Copying~\citep{gu2023mamba} and Induction Heads~\citep{olsson2022context}, which can be seen as easier associative recall tasks.
The MQAR task is also closely related to ``phonebook look-up'' tasks which has been shown to be challenging for recurrent models such as SSMs, due to their finite state capacity~\citep{jelassi2024repeat,de2024griffin}.
}{}

We compare on a challenging version of the MQAR setup from~\citep{arora2024simple}, using a harder task, longer sequences, and smaller models.
Our baselines include standard multi-head softmax attention as well as the Based architecture which combines convolutions, local attention, and a linear attention variant.

Results are shown in \cref{fig:mqar}.
While Mamba-1 struggles on this task, Mamba-2 performs well across all settings.
\iftoggle{arxiv}{
Surprisingly, it is significantly better than Mamba-1 even when the state sizes are controlled ($\mathtt{N}=16$).
(We are not sure which aspect of the architecture is the predominant factor, which remains a question to explore in future work.)
Additionally, this task validates the importance of state size: increasing from $\mathtt{N}=16$ to $\mathtt{N}=64$ and $\mathtt{N}=256$ consistently improves performance on MQAR, as the larger state allows more information (key-value pairs) to be memorized.
}{}

\subsection{Language Modeling}
\label{sec:experiments:lm}

Following standard protocols in LLMs, we train and evaluate the Mamba-2 architecture on standard autoregressive language modeling against other architectures.
We compare both pretraining metrics (perplexity) and zero-shot evaluations.
The model sizes (depth and width) follow GPT3 specifications, from 125m to 2.7B. %
We use the Pile dataset~\citep{pile}, and follow the training recipe described in~\citet{brown2020language}.
This follows the same setup as reported in Mamba~\citep{gu2023mamba};
training details are in~\cref{sec:exp-details}.

\subsubsection{Scaling Laws}
For baselines, we compare against both Mamba and its Transformer++ recipe~\citep{gu2023mamba}, which is based on the PaLM and LLaMa architectures (e.g.\ rotary embedding, SwiGLU MLP, \iftoggle{arxiv}{RMSNorm instead of LayerNorm, no linear bias, and higher learning rates}{etc.}).
As Mamba has already demonstrated that it outperforms the standard Transformer architecture (GPT3 architecture) as well as recent subquadratic architectures (H3~\citep{dao2023hungry}, Hyena~\citep{poli2023hyena}, RWKV-4~\citep{peng2023rwkv}, RetNet~\citep{sun2023retentive}), we omit those in the plot for clarity (see \citet{gu2023mamba} for comparisons).

\iftoggle{arxiv}{
\cref{fig:lm-scaling} shows scaling laws under the standard Chinchilla~\citep{hoffmann2022empirical} protocol,
on models from $\approx 125M$ to $\approx 1.3B$ parameters.

}{}

\subsubsection{Downstream Evaluations}

\cref{table:downstream_zeroshot} shows the performance of Mamba-2 on a range of popular downstream zero-shot evaluation tasks, compared to the most well-known open source models at these sizes,
most importantly Pythia~\citep{biderman2023pythia} which were trained with the same tokenizer, dataset, and training length (300B tokens) as our models.

\subsubsection{Hybrid Models: Combining SSD Layer with MLP and Attention}
\label{sec:hybrid}

Recent and concurrent work~\citep{dao2023hungry,de2024griffin,lieber2024jamba,glorioso2024zamba} suggests that a hybrid architecture with both SSM layers and
attention layers could improve the model quality over that of a Transformer, or a
pure SSM (e.g., Mamba) model, especially for in-context learning.
We explore the different ways that SSD layers can be combined with attention and
MLP to understand the benefits of each.
Empirically we find that having around 10\% of the total
number of layers being attention performs best.
Combining SSD layers, attention layers, and MLP also works better than either
pure Transformer++ or Mamba-2.

\paragraph{SSD and Attention}
We find that SSD and attention layers are complementary: by themselves (e.g.
in the Mamba-2 architecture vs. Transformer++) their performance (measured by
perplexity) is nearly the same, but a mixture of SSD and attention layers
outperforms the pure Mamba-2 or Transformer++ architecture. We show some results
(\cref{tab:ssd_attn}) for the 350M model (48 layers) trained to 7B tokens on the Pile with the GPT-2 tokenizer (same
number of parameters, same hyperparameters, same training and validation set).
Adding in just a few attention layers already yields notable improvement and
strikes the best balance between quality and efficiency. We hypothesize that
the SSM layers function well as a general sequence-to-sequence mapping, and
attention layers act as a retrieval mechanism to quickly refer to previous
tokens in the sequence instead of forcing the model to compress all the context
to its memory (SSM states).

\begin{table*}[ht]
  \small
  \caption{
    (\textbf{Combining SSD and Attention Blocks}.) Perplexity of a 350M model
    with 48 layers, with different number of attention layers.
    Having around a 10\% ratio of attention layers performs best.
  }
  \centering
  \begin{tabular}{@{}llllllllllllll@{}}
    \toprule
    \sc{Num.\ Attn Blocks} & 0 (Mamba-2) & 1 & 2 & 3 & 4 & 5 & 6 & 7 & 9 & 11 & 15 & 24 & Transformer++ \\
    \midrule
    \sc{Perplexity} $\downarrow$ & 8.60 & 8.38 & 8.32 & 8.29 & 8.29 & 8.28 & \textbf{8.26} & 8.27 & 8.28 & 8.30 & 8.34 & 8.50 & 8.68\\
    \bottomrule
  \end{tabular}
  \label{tab:ssd_attn}
\end{table*}

\paragraph{Hybrid Models with SSD, MLP, and Attention}

We compare different ways that SSD can be combined with the (gated) MLP and attention layers, and evaluate at the 2.7B scale (64 layers), trained to 300B tokens on the Pile (same
number of parameters, same hyperparameters, same training and validation set, same data order):
\begin{enumerate}
\item Transformer++: 32 attention layers and 32 gated MLP, interleaving.
\item Mamba-2: 64 SSD layers.
\item Mamba-2-MLP: 32 SSD and 32 gated MLP layers, interleaving.
\item Mamba-2-Attention: 58 SSD layers and 6 attention layers (at indices 9, 18, 27, 36, 45, 56)\footnote{In small-scale experiments, we find that as long as the attention layers are spaced out, not at the very beginning or at the very end, the model quality does not depend very much on the exact location of the attention layers.}.
\item Mamba-2-MLP-Attention: 28 SSD layers and 4 attention layers, interleaving with 32 gated MLP layers.
\end{enumerate}
We report the validation perplexity on the Pile, as well as zero-shot evaluation, in~\cref{table:hybrid_eval}.
In general, the quality of Transformer++ and Mamba-2 models are around the same.
We see that adding just 6 attention layers noticeably improves over the pure Mamba-2 model (and over Transformer++). Adding MLP layers reduces model quality, but can (i) speed up training and inference due to the simplicity and hardware-efficiency of the MLP layer (ii) be easier to up-cycle to MoE models by replacing MLP layers with mixture-of-experts.
\begin{table*}[!ht]
  \centering
  \captionsetup{font=small}
  \caption{
    (\textbf{Zero-shot Evaluations}.) Best results for each size in bold.
    We compare different ways SSD, MLP, and attention layers can be combined, evaluated at 2.7B scale trained to 300B tokens on the Pile.
  }
  \resizebox{0.99\linewidth}{!}
  {
    \begin{tabular}{@{}llllllllllll@{}}
      \toprule
      \sc{Model}                  & \sc{Token.} & \sc{Pile}             & \sc{LAMBADA}          & \sc{LAMBADA}        & \sc{HellaSwag}      & \sc{PIQA}           & \sc{Arc-E}          & \sc{Arc-C}          & \sc{WinoGrande}     & \sc{OpenbookQA}     & \sc{Average} \\
                                  &             & \sc{ppl $\downarrow$} & \sc{ppl $\downarrow$} & \sc{acc $\uparrow$} & \sc{acc $\uparrow$} & \sc{acc $\uparrow$} & \sc{acc $\uparrow$} & \sc{acc $\uparrow$} & \sc{acc $\uparrow$} & \sc{acc $\uparrow$} & \sc{acc $\uparrow$} \\
      \midrule
      Transformer++               & NeoX        & 6.13                  & 3.99                  & \underline{70.3}    & 66.4                & 75.2                & 67.7                & \underline{37.8}    & 63.9                & \textbf{40.4}       & 60.2 \\
      Mamba-2                     & NeoX        & 6.09                  & 4.10                  & 69.7                & \underline{66.6}    & \textbf{76.4}       & 69.6                & 36.4                & 64.0                & 38.8                & 60.2 \\
      Mamba-2-MLP                 & NeoX        & 6.13                  & 4.18                  & 69.3                & 65.0                & \textbf{76.4}       & 68.1                & 37.0                & 63.1                & 38.2                & 59.6 \\
      Mamba-2-Attention           & NeoX        & \textbf{5.95}         & \textbf{3.85}         & \textbf{71.1}       & \textbf{67.8}       & \underline{75.8}    & \underline{69.9}    & \underline{37.8}    & \textbf{65.3}       & 39.0                & \textbf{61.0} \\
      Mamba-2-MLP-Attention       & NeoX        & \underline{6.00}      & \underline{3.95}      & 70.0                & \underline{66.6}    & 75.4                & \textbf{70.6}       & \textbf{38.6}       & \underline{64.6}    & \underline{39.2}    & \underline{60.7} \\
      \bottomrule
    \end{tabular}
  }
  \label{table:hybrid_eval}
\end{table*}

\subsection{Speed Benchmarks}
\label{sec:experiments:benchmark}

We benchmark the speed of the SSD algorithm against Mamba's scan implementation and FlashAttention-2 (\cref{fig:scan_benchmark}).
SSD, thanks to its reformulation to use matrix multiplication as a subroutine, can exploit specialized matrix multiplication (matmul) units on GPUs, also known as tensor cores.
As a result, it is 2-8$\times$ faster than Mamba's fused associative scan, which does not leverage matmul units.
Due to its linear scaling in sequence length, SSD is faster than FlashAttention-2 starting at sequence length $2K$.

However, we note that the Mamba-2 model as a whole might not be as efficient to train as Transformer at short sequence length (e.g. at $2K$), since a Transformer with $L$ layers would have $\frac{L}{2}$ MLP layers and $\frac{L}{2}$ attention layers, while a Mamba-2 model would have $L$ SSD layers for the same number of parameters.
Generally the MLP layers are very hardware efficient since they consist of simple matrix multiplication and pointwise linearity.
As shown in~\cref{sec:hybrid}, one can also combine $\frac{L}{2}$ SSD layers and $\frac{L}{2}$ MLP layers to speed up training at short sequence length.

\iftoggle{arxiv}{
\subsection{Architecture Ablations}
\label{sec:experiments:ablations}

\subsubsection{Block Design}
\label{sec:experiments:ablations:block}

\cref{sec:architecture:block} introduces the Mamba-2 block, which has small modifications to the Mamba-1 block which are partly motivated by the connection to attention and also to improve the scalability of Mamba-2.
\cref{tab:ablations-architecture} ablates these architecture changes to the block, which occur outside of the core SSM layer.

The ablations validate that parallel projections to create $(A,B,C,X)$ saves parameters and performs slightly better than Mamba's sequential projections.
More importantly, this modification is amenable to tensor parallelism at larger model sizes (\cref{sec:systems}).
Additionally, the extra normalization layer also slightly improves performance.
More importantly, preliminary experiments at larger scales observed that it also helps with training stability.

\begin{table}
  \caption{
    (\textbf{Ablations: Mamba-2 block}.)
    We ablate the major differences between the Mamba-2 and Mamba-1 neural network blocks (\cref{fig:architecture}, \cref{sec:architecture:block}).
    Note that these components are independent of the inner sequence mixing layer; in these ablations, we use SSD for the inner SSM layer (differing from the S6 layer of Mamba-1).
  }
  \centering
  \begin{tabular}{@{}lllll@{}}
    \toprule
    Block & $ABCX$ \sc{Projections} & \sc{Extra Normalization} & \sc{Parameters} & \sc{Perplexity} \\
    \midrule
    Mamba-1 & Sequential              & \xmark                   & 129.3M          & 11.76 \\
            & Sequential              & \cmark                   & 129.3M          & 11.54 \\
            & Parallel                & \xmark                   & 126.5M          & 11.66 \\
    Mamba-2 & Parallel                & \cmark                   & 126.5M          & 11.49 \\
    \bottomrule
  \end{tabular}
  \label{tab:ablations-architecture}
\end{table}

\subsubsection{Head Structure}

\cref{sec:architecture:multihead} describes how the dimensions of the $B, C, X$ projections can be viewed as a hyperparameter analogous to notions of multi-head attention and multi-query attention.
We also showed how the original Mamba architecture is analogous to multi-value attention (\cref{prop:mamba-multihead}),
which was a choice that naturally developed from the state-space model point of view and was not previously ablated.

\cref{tab:ablations-heads} ablates choices of the multi-head structure for the Mamba-2 architecture.
Strikingly, we find a large difference between multi-value and multi-query or multi-key head patterns,
despite seeming very similar.
Note that this is not explained by the total state size, which is the same for all of them (equal to $\mathtt{HPN}$ or the product of the number of heads, head dimension, and state dimension).

We also compare to multi-head patterns where the number of $C, B, X$ (analogous to $Q, K, V$) heads is equal.
We compare against the standard multi-head pattern, as well as one with aggressive sharing where they all have only $1$ head.
Note that in the latter case, the model still has $\mathtt{H}$ different sequence mixers $M$, because each head still has a different $A$.
When parameter matched, these multi-head patterns perform similarly to each other, in between the MVA and MQA/MKA patterns.

\begin{table}[!th]
  \small
  \centering
  \captionsetup{type=table}
  \caption{
    (\textbf{Ablations: Multi-head structure}.)
    All models have state expansion factor $N=64$ and head size $P=64$ and are trained to Chinchilla scaling law token counts.
    The number of $A$ heads is always equal to the total heads $\mathtt{H}$, i.e. each head has a separate input-dependent $A$ decay factor.
    (\emph{Top}) 125M models, 2.5B tokens
    (\emph{Bottom}) 360M models, 7B tokens
  }
  \begin{tabular}{@{}lllllllll@{}}
    \toprule
    \sc{SSM Head Pattern} & \sc{Attn. Analog} & $A$ \sc{heads} & $B$ \sc{heads} & $C$ \sc{heads} & $X$ \sc{heads} & \sc{Layers} & \sc{Params} & \sc{Ppl.} \\
    \midrule
    Multi-input (MIS)     & Multi-value (MVA) & 24             & 1              & 1              & 24             & 24          & 126.5M      & $\mathbf{11.66}$ \\
    Multi-contract (MCS)  & Multi-query (MQA) & 24             & 1              & 24             & 1              & 24          & 126.5M      & $12.62$ \\
    Multi-expand (MES)    & Multi-key (MKA)   & 24             & 24             & 1              & 1              & 24          & 126.5M      & $12.59$ \\
    Multi-head (MHS)      & Multi-head (MHA)  & 24             & 24             & 24             & 24             & 15          & 127.6M      & $12.06$ \\
    Multi-state (MSS)     & -                 & 24             & 1              & 1              & 1              & 36          & 129.6M      & $12.00$ \\
    \midrule
    Multi-input (MIS)     & Multi-value (MVA) & 32             & 1              & 1              & 32             & 48          & 361.8M      & $\mathbf{8.73}$ \\
    Multi-contract (MCS)  & Multi-query (MQA) & 32             & 1              & 32             & 1              & 48          & 361.8M      & $9.33$ \\
    Multi-expand (MES)    & Multi-key (MKA)   & 32             & 32             & 1              & 1              & 48          & 361.8M      & $9.36$ \\
    Multi-head (MHS)      & Multi-head (MHA)  & 32             & 1              & 1              & 1              & 70          & 361.3M      & $9.01$ \\
    Multi-state (MSS)     & -         & 32             & 32             & 32             & 32             & 29          & 357.3M      & $9.04$ \\
    \bottomrule
  \end{tabular}
  \label{tab:ablations-heads}
\end{table}

\subsubsection{Attention Kernel Approximations}
\label{sec:experiments:ablations:kernels}

\cref{sec:architecture:kernels} noted how SSD can be combined with ideas from the linear attention literature,
such as various forms of kernel approximations.
We ablate several variants of these suggested by previous works in \cref{tab:ablations-kernel}.
These include the cosFormer~\citep{qin2022cosformer}, Random Feature Attention~\cite{peng2021random}, and Positive Random Features (Performer)~\citep{choromanski2021rethinking}.

We also ablate adding a normalization term, akin to the denominator of the softmax function in standard attention.
We found that this introduced instabilities to most variants, but slightly improved performance for the ReLU activation function $\psi$.

\cref{tab:ablations-kernel-based} also tests more recent proposals to improve linear attention that involve expanding the feature dimension (Based~\citep{arora2024simple} and ReBased~\citep{aksenov2024linear}).
These linear attention extensions aim to appropriate the $\exp$ kernel with a quadratic approximation.
ReBased also proposes to replace the QK activation function with a layer normalization;
from an SSM-centric view we apply a normalization on top of $(B, C)$ before applying the SSM function.
We note that this technique has been independently proposed as the ``QK-Norm'' for softmax attention~\citep{team2024chameleon}
and an ``internal normalization'' for Mamba~\citep{lieber2024jamba}.

Overall, \cref{tab:ablations-kernel} and \cref{tab:ablations-kernel-based} found that the kernel approximation methods we tried did not seem to improve over simple pointwise non-linear activation functions for $\psi$.
Thus our default settings for Mamba-2 used $\psi(x) = \mathsf{Swish}(x)$ to follow Mamba-1,
but we suggest that removing this activation entirely may be a simpler choice that we did not extensively test.

We emphasize however that SSD and vanilla linear attention differ in the inclusion of the 1-semiseparable mask $L$, while the various linear attention methods in the literature were derived to approximate softmax attention without this term; thus, our negative results may be not unexpected.

\begin{figure}[!t]
  \begin{minipage}{.5\linewidth}
    \centering
    \captionsetup{type=table}
    \caption{
      (\textbf{Ablations: Kernel approximations}.)
      We test various proposals for the kernel activation function $\psi$, including linear attention variants aiming to approximate the $\exp$ kernel from standard softmax attention.
    }
    \begin{tabular}{@{}llll@{}}
      \toprule
      \sc{Kernel activation $\varphi$} & \sc{Perplexity} \\
      \midrule
      none                             & $11.58$ \\
      Swish                            & $11.66$ \\
      Exp                              & $11.62$ \\
      ReLU                             & $11.73$ \\
      ReLU + normalization             & $11.64$ \\
      \midrule
      cosFormer                        & $11.97$ \\
      Random Feature Attention         & $11.57$ \\
      Positive Random Features (Performer) & $12.21$ \\
      \bottomrule
    \end{tabular}
    \label{tab:ablations-kernel}
  \end{minipage}
\hfill
\begin{minipage}{.45\linewidth}
  \centering
  \captionsetup{type=table}
  \caption{
    (\textbf{Ablations: Kernel approximations}.)
    We test the (Re)Based methods for linear attention approximations, which involve expanded feature maps.
    (\emph{Top}) 130M models. (\emph{Top}) 380M models with $N=256$.
  }
  \begin{tabular}{@{}lll@{}}
    \toprule
    \sc{Kernel activation $\varphi$} & \sc{Perplexity} \\
    \midrule
    Swish                            & $11.67$         \\
    Swish + Taylor (Based)           & $12.19$         \\
    LayerNorm                        & $11.50$         \\
    LayerNorm + Square (ReBased)     & $11.84$         \\
    \midrule
    Swish                            & $8.58$ \\
    Swish + Taylor (Based)           & $8.71$ \\
    LayerNorm                        & $8.61$ \\
    LayerNorm + Square (ReBased)     & $8.63$ \\
    \bottomrule
  \end{tabular}
  \label{tab:ablations-kernel-based}
\end{minipage}
\end{figure}
}{}

%% file: structure/related.tex
\section{Related Work and Discussion}
\label{sec:related}

The state space duality framework bridges connections between SSMs, structured matrices, and attention.
We discuss in more depth the relations between SSD and these concepts more broadly.
Using ideas from each of the viewpoints, we also suggest some directions that the SSD framework can be extended in future work.

\subsection{State Space Models}
\label{sec:related:ssm}

Structured state space models can be characterized along the axes
\begin{enumerate}[label=(\roman*)]
  \item whether it is time-invariant or time-varying.
  \item the dimensionality of the system.
  \item the structure on the recurrent transitions $A$.
\end{enumerate}
SSD can be described as a selective SSM with SISO dimensions and scalar-identity structure.

\paragraph{Time Variance (Selectivity).}
The original structured SSMs (S4) were linear time-invariant (LTI) systems~\citep{gu2023thesis,gu2022efficiently} motivated by continuous-time online memorization~\citep{gu2020hippo,gu2021combining,gu2023train}.
Many variants of structured SSMs have been proposed~\citep{gupta2022diagonal,gu2022parameterization,smith2023s5,ma2023mega,dao2023hungry}, including several that drop the recurrence and focus on the convolutional representation of LTI SSMs~\citep{li2023makes,poli2023hyena,fu2023simple,qin2023toeplitz}.

SSD is a time-varying structured SSM, also known as a \textbf{selective SSM} introduced in Mamba~\citep{gu2023mamba}.
Selective SSMs are closely related to gating mechanisms of RNNs, including classical RNNs such as the LSTM~\citep{lstm} and GRU~\citep{chung2014empirical} as well as more modern variants such as the QRNN~\citep{bradbury2016quasi}, SRU~\citep{lei2017simple,lei2021attention}, RWKV~\citep{peng2023rwkv}, HGRN~\citep{qin2023hierarchically}, and Griffin~\citep{de2024griffin,botev2024recurrentgemma}.
These RNNs differ in their parameterizations in various ways, most importantly in the lack of a state expansion.

\paragraph{Dimensionality and State Expansion.}
An important characteristic of SSD, shared by previous SSMs in its lineage (S4, H3, Mamba), is that it is a \textbf{single-input single-output (SISO)} system where input channels are processed independently.
This leads to a much larger effective state size of $\mathtt{ND}$ where $\mathtt{N}$ is the SSM state size (also called state expansion factor) and $\mathtt{D}$ is the standard model dimension.
Traditional RNNs either have $\mathtt{N}=1$ or are
multi-input multi-output (MIMO) with dense $B, C$ matrices,
either of which leads to a smaller state.
While MIMO SSMs have been shown to work well in some domains~\citep{smith2023s5,orvieto2023resurrecting,lu2023structured}, Mamba showed that state expansion is crucial for information-dense domains such as language.
One of the main advantages of SSD is allowing for even larger state expansion factors without slowing down the model.
Many subsequent works have since adopted state expansion (\cref{sec:related:concurrent}).

\paragraph{Structure.}
Compared to previous structured SSMs,
the main restriction of SSD is on the expressivity of the state transitions $A_t$.
We note that more general SSMs, such as the case of diagonal $A_t$, have the same theoretical efficiency as SSD, but are less hardware-friendly.
This is because the dual quadratic form loses its attention-like interpretation and becomes more difficult to compute.
Thus compared to Mamba, SSD differs only in a slightly more restrictive form of diagonal $A_t$,
and trades off this expressivity for improved hardware efficiency (and ease of implementation).

We hypothesize that it may be possible to refine our structured matrix algorithms to improve to the general diagonal SSM case as well.

\subsection{Structured Matrices}
\label{sec:related:matrices}

The first viewpoint of the state space duality adopts the viewpoint of these models as \textbf{matrix sequence transformations} or ``matrix mixers'':
sequence transformations (\cref{def:sequence-transformation}) that can be represented as matrix multiplication (by a $\mathtt{T} \times \mathtt{T}$ matrix) along the sequence dimension $\mathtt{T}$.

Several such matrix mixers have been proposed before, where the primary axis of variation is the representation of the matrix.
These include MLP-Mixer~\citep{tolstikhin2021mlp} (unstructured matrix), FNet~\citep{lee2021fnet} (Fourier Transform matrix), M2~\citep{dao2019learning,dao2020kaleidoscope,dao2022monarch,fu2024monarch} (butterfly/monarch matrix),
Toeplitz matrices~\citep{poli2023hyena,qin2023toeplitz}, and even more exotic structures~\citep{de2018two,thomas2018learning}.

An important characterization is that efficient (sub-quadratic) matrix sequence transformations are exactly those which have \emph{structured matrix mixers}.
A core result of the SSD framework is viewing SSMs as matrix mixers with a particular structure -- semiseparable matrices (\cref{sec:ssm}).
The linear vs. quadratic duality then takes the form of structured matrix multiplication vs. naive matrix multiplication.

The structure matrix representation led to our efficient SSD algorithm through block decompositions of particular semiseparable matrices (\cref{sec:efficient}).
We note that semiseparable matrices are well-studied in the scientific computing literature, and incorporating those ideas may be a promising avenue for more improvements to state space models.
We also suggest that focusing on the matrix mixer viewpoint can lead to more fruitful directions for sequence models,
such as designing principled non-causal variants of Mamba, or finding ways to characterize and bridge the gap between softmax attention and sub-quadratic models through analyzing their matrix transformation structure.

\subsection{(Linear) Attention}
\label{sec:related:attention}

Compared to standard (causal) attention, SSD has only two main differences.

First, SSD does not use the softmax activation of standard attention~\citep{bahdanau2015neural,vaswani2017attention}, which is what gives attention its quadratic complexity.
When the softmax is dropped, the sequence can be computed with linear scaling through the linear attention framework~\citep{katharopoulos2020transformers}.

Second, SSD multiplies the logits matrix by an input-dependent 1-semiseparable mask.
Thus this mask can be viewed as replacing the softmax in standard attention.

This semiseparable mask can also be viewed as providing positional information.
The elements $a_t$ act as ``gates'' in the RNN sense, or a ``selection'' mechanism (see discussion in Mamba paper),
and their cumulative products $a_{j:i}$ control how much interaction is allowed between positions $i$ and $j$.
Positional embeddings (e.g.\ sinusoidal~\citep{vaswani2017attention}, AliBi~\citep{press2022train}, and RoPE~\citep{su2021roformer}) are an important component of Transformers that are often viewed as heuristics,
and the 1-SS mask of SSD can be seen as a more principled form of relative positional embeddings.
We note that this view was also posited concurrently by GateLoop~\citep{katsch2023gateloop}.

The second viewpoint of state space duality is a special case of our more general structured masked attention (SMA) framework,
where the duality is revealed as different contraction orderings on a simple 4-way tensor contraction.
SMA is a strong generalization of linear attention that is much more general than SSD as well;
other forms of structured masks may lead to more variants of efficient attention with different properties than SSD.

Beside leading to new models, these connections to attention can lead to other directions for understanding SSMs.
For example, we are curious whether the phenomenon of attention sinks~\citep{darcet2024vision,xiao2024efficient} exist for Mamba models,
and more broadly whether interpretability techniques can be transferred to SSMs~\citep{ali2024hidden}.

Finally, many other variants of linear attention have been proposed~\citep{schlag2021linear,peng2021random,choromanski2021rethinking,qin2022devil,qin2022cosformer,zheng2022linear,zhang2024hedgehog,arora2024zoology,arora2024simple} (see \cref{sec:attention:kernel} for descriptions of several of these),
and we expect that many techniques can be transferred to SSMs\iftoggle{arxiv}{ (e.g. \cref{sec:architecture:kernels})}{}.

We emphasize that SSD \textbf{does not generalize standard softmax attention}, or any other transformation on the attention kernel matrix that does not have a finite feature map $\psi$.
Compared to general attention, SSD's advantage is having a controllable state expansion factor $\mathtt{N}$ that compresses the history, compared to quadratic attention's cache of the entire history scaling with sequence length $\mathtt{T} \gg \mathtt{N}$.
Concurrent work has starting studying the tradeoffs of these representations, for example on copying and in-context learning tasks~\citep{akyurek2024context,jelassi2024repeat,grazzi2024mamba,park2024can}.
We note that Mamba-2 significantly improves on Mamba on some of these capabilities (e.g. as demonstrated by MQAR results in \cref{sec:experiments:mqar}), but more remains to be understood.

\subsection{Related Models}
\label{sec:related:concurrent}

We finally highlight a growing body of recent and concurrent work that have developed sequence models very similar to Mamba and Mamba-2.

\begin{itemize}[leftmargin=*,itemsep=0pt,topsep=0pt]
  \item RetNet~\citep{sun2023retentive} and TransNormerLLM~\citep{qin2023transnormerllm} generalize Linear Attention using decay terms instead of a cumulative sum,
    and propose dual parallel/recurrent algorithms as well as a hybrid ``chunkwise'' mode.
    These algorithms can be seen as an instantiation of SSD where $A_t$ is time-invariant (constant for all $t$);
    in the SMA interpretation, the mask matrix $L$ would be a decay matrix $L_{i,j} = \gamma^{i-j}$.
    These models also differ architecturally in various ways.
    For example, since they were derived from an attention-centric perspective they preserve the multi-head attention (MHA) pattern; since Mamba-2 was derived from an SSM-centric pattern it preserves the multi-value attention (MVA) or multi-expand SSM (MES) pattern, which we show to be better\iftoggle{arxiv}{ (\cref{sec:experiments:ablations})}{}.

  \item GateLoop~\citep{katsch2023gateloop} concurrently proposed using input-dependent decay factors $A_t$, and developed the same dual quadratic form as in SSD which they call a ``surrogate attention'' form.

  \item Gated Linear Attention (GLA)~\citep{yang2024gated} proposed a variant of linear attention with data-dependent gates, along with efficient algorithms to compute a chunkwise mode and hardware-aware implementations.

  \item HGRN~\citep{qin2023hierarchically} introduced an RNN with input-dependent gates, which was improved to incorporate state expansion in HGRN2~\citep{qin2024hgrn2}.

  \item Griffin~\citep{de2024griffin} and RecurrentGemma~\citep{botev2024recurrentgemma} showed that an RNN with input-dependent gating, combined with local attention, can be very competitive with strong modern Transformers.
    Jamba also showed that combining Mamba with a few layers of attention performs very well on language modeling~\citep{lieber2024jamba}.

  \item xLSTM~\citep{beck2024xlstm} improves the xLSTM by adopting the idea of state expansion and other gating, normalization, and stabilization techniques.

  \item RWKV(-4)~\citep{peng2023rwkv} is an RNN based on a different linear attention approximation (the attention-free Transformer~\citep{zhai2021attention}).
    It has recently been improved to the RWKV-5/6 (Eagle and Finch) architectures~\citep{peng2024eagle} by adopting the ideas of selectivity and state expansion.
\end{itemize}

%% file: structure/conclusion.tex
\section{Conclusion}
\label{sec:conclusion}

We proposed a theoretical framework based on well-studied classes of structured
matrices that bridges the conceptual gap between SSMs and attention variants.
This framework yields insights on how recent SSMs (e.g. Mamba) perform as well
as Transformers on language modeling.
Moreover, our theoretical tools provide new ideas to improve SSMs (and
potentially Transformers) by connecting the algorithmic\iftoggle{arxiv}{ and systems}{} advances on
both sides.
As a demonstration, the framework guides our design of a new architecture
(Mamba-2) at the intersection of SSMs and structured attention.

%% file: structure/glossary.tex
\section{Glossary}
\label{sec:glossary}

\begin{table}[!h]
  \caption{
    Glossary of notation and terminology; mnemonics bolded.
    (\emph{Top}) Frequently used tensor dimensions.
    (\emph{Bottom}) Matrices and tensors used in state space models or structured masked attention.
  }
  \centering
  \begin{tabular}{@{}lll@{}}
    \toprule
    Notation       & Description                                                                & Definition \\
    \midrule
    $\mathtt{T}$   & \textbf{Time} axis or \textbf{target} sequence axis                        & \cref{def:sequence-transformation} \\
    $\mathtt{S}$   & \textbf{Source} sequence axis (in attention)                               & \cref{eq:kernel-attention} \\
    $\mathtt{D}$   & Model \textbf{dimension} or $\mathtt{d\_model}$                            & \cref{def:head-pattern} \\
    $\mathtt{N}$   & State/feature dimension or $\mathtt{d\_state}$                             & \cref{eq:s6,eq:kernel-attention} \\
    $\mathtt{P}$   & Head dimension or $\mathtt{d\_head}$                                       & \cref{def:sequence-transformation} \\
    $\mathtt{H}$   & Number of \textbf{heads} or $\mathtt{n\_head}$                             & \cref{def:head-pattern} \\
    \midrule
    $M$            & Sequence transformation \textbf{matrix}                                    & \cref{def:matrix-transformation} \\
    $A$            & Discrete SSM recurrent (state) matrix                                      & \cref{eq:s6} \\
    $B$            & State space model input projection (expansion) matrix                      & \cref{eq:s6} \\
    $C$            & State space model output projection (contraction) matrix                   & \cref{eq:s6} \\
    $X$            & Input matrix (shape $\mathtt{(T, P)}$)                                     & \cref{eq:s6,eq:kernel-attention} \\
    $Y$            & Output matrix (shape $\mathtt{(T, P)}$)                                    & \cref{eq:s6,eq:kernel-attention} \\
    $Q$            & Attention \textbf{query} matrix                                            & \cref{eq:kernel-attention} \\
    $K$            & Attention \textbf{key} matrix                                              & \cref{eq:kernel-attention} \\
    $V$            & Attention \textbf{value} matrix                                            & \cref{eq:kernel-attention} \\
    $G$            & Attention \textbf{Gram} matrix                                             & $QK^{\top}$ (or $CB^{\top}$) \\
    $L$            & (Structured) mask matrix (\textbf{lower}-triangular in the causal setting) & \cref{def:sma} \\
    \bottomrule
  \end{tabular}
  \label{tab:glossary}
\end{table}

%% file: structure/scan.tex
\section{Efficient Algorithms for the Scalar SSM Scan (1-SS Multiplication)}
\label{sec:scan}

In this section we flesh out various algorithms for computing the scalar SSM scan, through the lens of structured matrix decompositions.
The scalar SSM scan is defined as computing the recurrent part of the discrete SSM \eqref{eq:1ss-recurrence},
in the case when $N=1$ (i.e.\ $A$ is a scalar).
This is commonly used to compute SSMs recurrently;
in particular, the case of structured SSMs where $A$ is diagonally structured reduces down to this operation,
such as in the S5~\citep{smith2023s5} and S6~\citep{gu2023mamba} models.

The goal of this section is to support a central theme of this paper that \emph{efficient algorithms for sequence models can be viewed as structured matrix multiplication algorithms}.
The various matrix decomposition ideas we show here are related to ideas used to derive fast SSM algorithms (\cref{sec:efficient}),
as well as directly used as a subroutine.

\subsection{Problem Definition}
Let $a : \mathtt{(D,)}$ and $b : \mathtt{(D,)}$ be sequences of scalars.
The \textbf{scalar SSM scan} is defined as
\begin{equation}%
  \label{eq:ssm-scan}
  h_t = a_t h_{t-1} + b_t
  .
\end{equation}
Here $h_{-1}$ can be an arbitrary value representing the previous \emph{hidden state} to the SSM recurrence;
unless otherwise specified, we assume $h_{-1} = 0$.

We also call equation \eqref{eq:ssm-scan} the \textbf{\texttt{cumprodsum}} (cumulative product sum).
Note that the \texttt{cumprodsum} reduces to the \texttt{cumprod} (cumulative product) when $b = 0$ is the additive identity and it reduces to the \texttt{cumsum} (cumulative sum) when $a=1$ is the multiplicative identity.

Finally, note that in vectorized form we can write
\begin{align*}%
  h &= M b \\
  M &=
  \begin{bmatrix}
    1 & \\
    a_1 & 1 & \\
    a_2a_1 & a_2 & 1 \\
    \vdots & \vdots & \ddots & \ddots \\
    a_{T-1}\dots a_1 & a_{T-1}\dots a_2 & \dots & a_{T-1} & 1 \\
  \end{bmatrix}
\end{align*}
In other words, this is simply the matrix-vector product by a 1-SS matrix $M$.

Therefore we have three ways of viewing this fundamental primitive operation that are all equivalent:
\begin{itemize}
  \item A (scalar) SSM scan.
  \item A \texttt{cumprodsum}.
  \item A 1-SS matrix-vector multiplication .
\end{itemize}

\subsection{Classical Algorithms}

We first describe the two classical ways of computing the SSM scan \eqref{eq:ssm-scan},
previously used by prior work.

\subsubsection{Sequential Recurrence}
\label{sec:scan:recurrence}

The recurrent mode simply computes \eqref{eq:ssm-scan} one timestep $t$ at a time.
From the perspective of 1-SS multiplication, this was also described in \cref{sec:ssm:algorithms:linear}.

\subsubsection{Parallel Associative Scan}
\label{sec:scan:classical:parallel}

Second, an important observation is that this recurrence can be turned into an associative scan~\citep{martin2018parallelizing,smith2023s5}.
This fact is not completely obvious.
For example, S5 defined the correct associative scan operator and then showed associativity of the operator through rote calculation.

A slightly cleaner way to see that this is computable with an associative scan is to turn the multi-term recurrence into a single-term recurrence on a hidden state of size $2$ instead of $1$:
\begin{align*}%
  h_t &= a_t h_{t-1} + b_t
  \\
  \begin{bmatrix} h_t \\ 1 \end{bmatrix}
      &=
  \begin{bmatrix}
    a_t & b_t \\ 0 & 1
  \end{bmatrix}
  \begin{bmatrix} h_{t-1} \\ 1 \end{bmatrix}
  .
\end{align*}
Then computing all the $h_t$ is the same as taking the cumulative products of these $2 \times 2$ matrices.
Since matrix multiplication is associative,
this can be computed with an associative scan.
The associative binary operator is simply matrix multiplication on these particular matrices:
\begin{align*}%
  \begin{bmatrix}
    a_t & b_t \\ 0 & 1
  \end{bmatrix}
  \begin{bmatrix}
    a_s & b_s \\ 0 & 1
  \end{bmatrix}
  =
  \begin{bmatrix}
    a_ta_s & a_tb_s + b_t \\ 0 & 1
  \end{bmatrix}
  .
\end{align*}
Equating the top row yields the same associative scan operator as defined by S5:
\begin{equation}%
  \label{eq:scan:associative-operator}
  (a_t, b_t) \otimes (a_s, b_s) = (a_ta_s, a_tb_s + b_t)
  .
\end{equation}

The reason why associative scans are important is that they can be parallelized using a divide-and-conquer algorithm~\citep{blelloch1990prefix}.
We omit the details of this algorithm, and instead show that the entire associative SSM scan algorithm can be derived from scratch through matrix decompositions (\cref{sec:scan:associative}).

\subsection{Efficient Algorithms via Structured Matrix Decompositions}

We discuss several algorithms for computing the SSM scan, all through the lens of finding structured matrix decompositions of the 1-SS matrix $M$.
These algorithms or computation modes include
\begin{itemize}
  \item A \emph{dilated} mode where information is propagated $1, 2, 4, 8, \dots$ steps at a time.
  \item A \emph{state-passing} mode where information is propagated forward in chunks.
  \item A \emph{fully recurrent} mode that increments one step at a time, which is a special case of the state-passing mode.
  \item A \emph{block decomposition} parallel mode where $M$ is divided into hierarchical blocks.
  \item A \emph{scan} mode where $M$ is divide into equal size blocks and reduced recursively.
\end{itemize}

\subsubsection{Dilated Mode}

This mode factors the 1-SS matrix in a particular way involving increasing ``strides''.
This is best illustrated through a concrete example:
\footnotesize
\begin{align*}%
  M &=
  \begingroup
  \setlength\arraycolsep{2pt}
  \begin{bmatrix}
    a_{0:0} & \\
    a_{1:0} & a_{1:1} & \\
    a_{2:0} & a_{2:1} & a_{2:2} \\
    a_{3:0} & a_{3:1} & a_{3:2} & a_{3:3} \\
    a_{4:0} & a_{4:1} & a_{4:2} & a_{4:3} & a_{4:4} \\
    a_{5:0} & a_{5:1} & a_{5:2} & a_{5:3} & a_{5:4} & a_{5:5} \\
    a_{6:0} & a_{6:1} & a_{6:2} & a_{6:3} & a_{6:4} & a_{6:5} & a_{6:6} \\
    a_{7:0} & a_{7:1} & a_{7:2} & a_{7:3} & a_{7:4} & a_{7:5} & a_{7:6} & a_{7:7} \\
  \end{bmatrix}
  \endgroup
  \\&=
  \begingroup
  \setlength\arraycolsep{2pt}
  \begin{bmatrix}
    a_{0:0} & \\
            & a_{1:1} & \\
            &         & a_{2:2} \\
            &         &         & a_{3:3} \\
    a_{4:0} &         &         &         & a_{4:4} \\
            & a_{5:1} &         &         & & a_{5:5} \\
            &         & a_{6:2} &         & & & a_{6:6} \\
            &         &         & a_{7:3} & & & & a_{7:7} \\
  \end{bmatrix}
  \begin{bmatrix}
    a_{0:0} & \\
            & a_{1:1} & \\
    a_{2:0} &         & a_{2:2} \\
            & a_{3:1} &         & a_{3:3} \\
            &         & a_{4:2} &         & a_{4:4} \\
            &         &         & a_{5:3} &         & a_{5:5} \\
            &         &         &         & a_{6:4} &         & a_{6:6} \\
            &         &         &         &         & a_{7:5} & & a_{7:7} \\
  \end{bmatrix}
  \begin{bmatrix}
    a_{0:0} & \\
    a_{1:0} & a_{1:1} & \\
            & a_{2:1} & a_{2:2} \\
            &         & a_{3:2} & a_{3:3} \\
            &         &         & a_{4:3} & a_{4:4} \\
            &         &         &         & a_{5:4} & a_{5:5} \\
            &         &         &         &         & a_{6:5} & a_{6:6} \\
            &         &         &         &         &         & a_{7:6} & a_{7:7} \\
  \end{bmatrix}
  \endgroup
\end{align*}
\normalsize

Note that this closely resembles the computation of dilated convolutions.

We also note that this factorization shows that 1-SS matrices are a special case of butterfly matrices, another broad and fundamental type of structured matrix \citep{dao2019learning,dao2020kaleidoscope}.

\begin{remark}
  This algoritihm is sometimes described as a ``work-inefficient but more parallelizable'' prefix sum algorithm~\citep{hillis1986data}, becauses it uses $O(T\log(T))$ operations but has half the depth/span as the work-efficient associative scan algorithm.
\end{remark}

\subsubsection{State-Passing (Chunkwise) Mode}

This mode can be viewed as a generalization of the standard recurrent mode where instead of passing forward the recurrent state $h$ one step at a time, we compute the answer on chunks of arbitrary length $k$ and pass the state through the chunk.
This can also be derived from a simple block decomposition of the 1-SS matrix.

\begin{remark}
  While we call this ``state-passing'' to refer to how states are passed from one local segment to another,
  this is related to the ``chunkwise'' algorithms proposed by related models~\citep{sun2023retentive,yang2024gated}.
\end{remark}

Consider computing $h = Mb$ in ``chunks'': for some index $k \in [T]$, we want to compute $h_{0:k}$ or the output up to index $k$, and have a way to reduce the problem to a smaller problem on indices $[k:T]$.

We write $M$ as
\begin{align*}%
  M =
  \begin{bmatrix}
    a_{0:0} & \\
    a_{1:0} & a_{1:1} \\
    \vdots & & \ddots \\
    a_{k-1:0} & \dots & \dots & a_{k-1:k-1} \\
    a_{k:0} & \dots & \dots & a_{k:k-1} & a_{k:k} \\
    \vdots & & & \vdots & \vdots & \ddots \\
    a_{T-1:0} & \dots & \dots & a_{T-1:k-1} & a_{T-1:k} & \dots & a_{T-1:T-1} \\
  \end{bmatrix}
\end{align*}

Let the upper-left triangle be $M_L$, lower-right be $M_R$ (left and right subproblems), and lower-left be $M_C$.
Divide up $b$ into $b_L = b_{0:k}$ and $b_R = b_{k:T}$ in the same way.
Note that
\begin{align*}%
  Mb = \begin{bmatrix} M_L b_L \\ M_R b_R + M_C b_L \end{bmatrix}
\end{align*}
Also, $M_C$ has the rank-1 factorization (this is essentially the defining property of semiseparable matrices)
\begin{align*}%
  M_C =
  \begin{bmatrix} a_{k:k} \\ \vdots \\ a_{T-1:k} \end{bmatrix}
  a_k
  \begin{bmatrix} a_{k-1:0} & \cdots & a_{k-1:k-1} \end{bmatrix}
\end{align*}

Thus
\begin{align*}%
  M_C b_L =
  \begin{bmatrix} a_{k:k} \\ \vdots \\ a_{T-1:k} \end{bmatrix}
  a_k
  \cdot
  (Mb)_{k-1}
  .
\end{align*}
Here we think of $(Mb)_{k-1} = h_{k-1}$ as the ``final state'' of the left chunk,
because the row vector in $M_C$'s factorization is the same as the final row of $M_L$.
Furthermore, note that the column vector in $M_C$'s factorization is the same as the final column of $M_R$.\footnote{Both these facts can be seen from the Woodbury inverse...}
Thus
\begin{align*}%
  M_R b_R + M_C b_L =
  M_R
  \begin{bmatrix} a_k h_{k-1} + b_k \\ b_{k+1} \\ \vdots \\ b_{T-1} \end{bmatrix}
\end{align*}

Finally, we use the observation that $M_L$ and $M_R$ are self-similar to the original matrix $M$; the answers for these two smaller 1-SS matrix multiplications can be performed arbitrarily using any algorithm.
In total, the algorithm proceeds as follows:
\begin{enumerate}
  \item Compute the left half of the answer $h_{0:k}$ using any desired method (i.e.\ any of the methods for 1-SS multiplication from this section).
  \item Compute the final state $h_{k-1}$. %
  \item Increment the state by one step to modify $b_{k}$.
  \item Compute the right half of the answer $h_{k:T}$ using any desired method.
\end{enumerate}

In other words, we compute the left subproblem as a black box, pass its final state on to the right problem, and compute the right subproblem as a black box.

The utility of this method comes from more complicated settings, such as in the general $N$-semiseparable case,
and when the input $b$ has an additional ``batch'' dimension (or in other words this is a matrix-matrix instead of matrix-vector multiplication).
In this case, we can use an alternate algorithm for the chunks (corresponding to MM by $M_L$ and $M_R$) that does not materialize the full hidden states $h$.
Instead, we skip the hidden states and directly compute the final state $h_{k-1}$ in an alternate way, then ``pass'' the state to the next chunk.

\paragraph{Complexity.}
This method can be very work-efficient because steps 2-3 takes only constant time.
Therefore assuming the two subproblems (steps 1 and 4) are linear time,
the whole method takes linear time.

The downside is that this is also sequential.

\subsubsection{Fully Recurrent Mode}

Note that the fully recurrent mode, where the recurrence is evolved one step at a time \eqref{eq:ssm-scan},
is simply an instantiation of the state-passing mode with chunk size $k=1$.

\subsubsection{(Parallel) Block Decomposition Mode}

This uses the same matrix decomposition as the state-passing mode,
but computes subproblems in a different order that trades off computation for parallelization.

As usual, we write $M$ as
\begin{align*}%
  M =
  \begin{bmatrix}
    1 & \\
    a_1 & 1 & \\
    a_2a_1 & a_2 & 1 \\
    \vdots & \vdots & \ddots & \ddots \\
    a_{T-1}\dots a_1 & a_{T-1}\dots a_2 & \dots & a_{T-1} & 1 \\
  \end{bmatrix}
  =
  \begin{bmatrix}
    1 & \\
    -a_1 & 1 & \\
    0 & -a_2 & 1 \\
    \vdots & \vdots & \ddots & \ddots \\
    0 & 0 & \dots & -a_{T-1} & 1 \\
  \end{bmatrix}^{-1}
\end{align*}
The key observation is again that the bottom-left quadrant of $M$ is rank-1.
Aside from inspection, another way to see this is by using the RHS, observing that the bottom-left quadrant of it is a trivial rank-1 matrix (it is all 0 except the top-right corner is $-a_{T/2}$),
and using the Woodbury inversion formula to see that the bottom-left corner of the LHS must also be rank 1.
This also provides a way to deduce the rank-1 factorization,
which can be verified through inspection:
\begin{align*}%
  M_{\text{lower-left-quadrant}}
  &=
  \begin{bmatrix}
    (a_{T/2} \dots a_1) & \dots & a_{T/2} \\
    \vdots & \ddots & \vdots \\
    (a_{T-1} \dots a_{T/2} a_{T/2-1} \dots a_1) & \dots & (a_{T-1} \dots a_{T/2})
  \end{bmatrix}
  \\
  &=
  \begin{bmatrix}
    a_{T/2} \\ \vdots \\ a_{T-1} \dots a_{T/2}
  \end{bmatrix}
  \begin{bmatrix}
    (a_{T/2-1} \dots a_1) & \dots & a_{T/2-1} & 1
  \end{bmatrix}
  .
\end{align*}

A second observation is that \emph{this matrix is self-similar}: any principle submatrix has the same form.
In particular, the top-left and bottom-right quadrants are both 1-SS matrices.

This provides an easy way to perform the matrix multiplication by $M$:
recurse on the two halves (i.e.\ top-left and bottom-right) in parallel,
and then account for the bottom-left submatrix.
This ``combination'' step in the divide-and-conquer algorithm is easy since the submatrix is rank 1.
This leads to a parallel algorithm.

\paragraph{Complexity.}
Like the state-passing algorithm,
this method uses the same block decompositions of the rank-structured semiseparable matrices.
The difference is that we recurse on both subproblems in parallel,
while the state-passing algorithm handles the left and then right subproblems.
This lowers the depth/span of the algorithm from linear to $\log(T)$.
The tradeoff is that the combination step (accounting for the rank-1 bottom-left submatrix) requires linear instead of constant work,
so the total work is $O(T\log(T))$ instead of linear.

Note also that in the recursion,
we can stop at any time and compute the subproblems in any other way.
This is a main idea behind the SSD algorithm (\cref{sec:efficient}),
where we switch to the dual \emph{quadratic attention} formulation on small subproblems.

\subsubsection{Associative Scan Mode}
\label{sec:scan:associative}

The state passing (chunkwise) algorithm has linear work, but also involves sequential operations.

The block matrix reduction and dilated modes are parallelizable: they have $\log(T)$ depth/span. However, they do extra work ($O(T \log(T)$).

As noted in \cref{sec:scan:classical:parallel}, there is an algorithm that achieves both $O(\log T)$ depth and $O(T)$ work by leveraging the associative scan (also called prefix scan) algorithm~\citep{baker1996pade}.
This algorithm is most easily seen from the SSM scan or \texttt{cumprodsum} view, and even then is not obvious: it requires separately deriving an associative operator \eqref{eq:scan:associative-operator},
and then leveraging the parallel/associative/prefix scan algorithm as a black box \citep{blelloch1990prefix}.

Here we show that it is actually possible to derive this parallel scan from leveraging a different matrix decomposition:

\begin{align*}%
M &=
\begin{bNiceArray}{cc|cc|cc|cc}
    a_{0:0} & \\
    a_{1:0} & a_{1:1} & \\
    \hline
    a_{2:0} & a_{2:1} & a_{2:2} \\
    a_{3:0} & a_{3:1} & a_{3:2} & a_{3:3} \\
    \hline
    a_{4:0} & a_{4:1} & a_{4:2} & a_{4:3} & a_{4:4} \\
    a_{5:0} & a_{5:1} & a_{5:2} & a_{5:3} & a_{5:4} & a_{5:5} \\
    \hline
    a_{6:0} & a_{6:1} & a_{6:2} & a_{6:3} & a_{6:4} & a_{6:5} & a_{6:6} \\
    a_{7:0} & a_{7:1} & a_{7:2} & a_{7:3} & a_{7:4} & a_{7:5} & a_{7:6} & a_{7:7} \\
\end{bNiceArray}
\\&=
\begin{bNiceArray}{cc|cc|cc|cc}[cell-space-limits=6pt,columns-width=1.2cm]
    a_{0:0} & \\
    a_{1:0} & a_{1:1} & \\
    \hline
    \Block{2-2}{\begin{bmatrix}a_{2:2}\\a_{3:2}\end{bmatrix}a_{2:1}\begin{bmatrix}a_{1:0}\\a_{1:1}\end{bmatrix}^{\top}} && a_{2:2} \\
                                                                                                                        && a_{3:2} & a_{3:3} \\
    \hline
    \Block{2-2}{\begin{bmatrix}a_{4:4}\\a_{5:4}\end{bmatrix}a_{4:1}\begin{bmatrix}a_{1:0}\\a_{1:1}\end{bmatrix}^{\top}} &&
    \Block{2-2}{\begin{bmatrix}a_{4:4}\\a_{5:4}\end{bmatrix}a_{4:3}\begin{bmatrix}a_{3:2}\\a_{3:3}\end{bmatrix}^{\top}} && a_{4:4} \\
                                                                                                                        &&&& a_{5:4} & a_{5:5} \\
    \hline
    \Block{2-2}{\begin{bmatrix}a_{6:6}\\a_{7:6}\end{bmatrix}a_{6:1}\begin{bmatrix}a_{1:0}\\a_{1:1}\end{bmatrix}^{\top}} &&
    \Block{2-2}{\begin{bmatrix}a_{6:6}\\a_{7:6}\end{bmatrix}a_{6:3}\begin{bmatrix}a_{3:2}\\a_{3:3}\end{bmatrix}^{\top}} &&
    \Block{2-2}{\begin{bmatrix}a_{6:6}\\a_{7:6}\end{bmatrix}a_{6:1}\begin{bmatrix}a_{5:4}\\a_{5:5}\end{bmatrix}^{\top}} && a_{6:6} \\
                                                                                                                        &&&&&& a_{7:6} & a_{7:7} \\
\end{bNiceArray}
\end{align*}

Now we proceed in three stages.

\paragraph{Stage 1.}
First we compute the answers for each of the diagonal blocks in the multiplication $Mb$.
This produces two numbers, but the first element is unchanged.
For example, the second block is going to compute $b_2$ and $a_3 b_2 + b_3$

\paragraph{Stage 2.}
Now consider each of the $2 \times 2$ blocks factored as a rank-1 matrix in the strictly lower triangular part of the matrix.
Note that each of the right side row vectors is the same as the bottom row vector in the diagonal block in its column:
in particular the $[ a_{1:0} \; a_{1:1} ]$, $[ a_{3:2} \; a_{3:3} ]$, and $[ a_{5:4} \; a_{5:5} ]$ rows.

Therefore we already have the answers to these from Stage 1,
which is the second element of all $T/2$ subproblems in Stage 1.
If we call this array of elements $b'$ (of half the size of $b$),
then we need to multiply $b'$ by the 1-SS matrix generated by $a_{3:-1}, a_{3:1}, a_{5:3}, a_{7:5}$.

\paragraph{Stage 3.}
Finally, each of the answers to Stage 2 can be broadcast into two final answers by multiplying by the left-side column vectors:
in particular the $[ a_{2:2} \; a_{3:2} ]^{\top}$, $[ a_{4:4} \; a_{5:4} ]^{\top}$, and $[ a_{6:6} \; a_{7:6} ]^{\top}$ vectors.

Note that this can be slightly modified with some off-by-one shifting of the indices.
An equivalent way to view this algorithm is as the three-step matrix factorization

\footnotesize
\begin{align*}%
  M &=
  \begingroup
  \setlength\arraycolsep{2pt}
  \begin{bmatrix}
    a_{0:0} & \\
    a_{1:0} & a_{1:1} & \\
    a_{2:0} & a_{2:1} & a_{2:2} \\
    a_{3:0} & a_{3:1} & a_{3:2} & a_{3:3} \\
    a_{4:0} & a_{4:1} & a_{4:2} & a_{4:3} & a_{4:4} \\
    a_{5:0} & a_{5:1} & a_{5:2} & a_{5:3} & a_{5:4} & a_{5:5} \\
    a_{6:0} & a_{6:1} & a_{6:2} & a_{6:3} & a_{6:4} & a_{6:5} & a_{6:6} \\
    a_{7:0} & a_{7:1} & a_{7:2} & a_{7:3} & a_{7:4} & a_{7:5} & a_{7:6} & a_{7:7} \\
  \end{bmatrix}
  \endgroup
  \\&=
  \begingroup
  \setlength\arraycolsep{2pt}
  \begin{bmatrix}
    a_{0:0} & \\
            & a_{1:1} & \\
            & a_{2:1} & a_{2:2} \\
            &         &         & a_{3:3} \\
            &         &         & a_{4:3} & a_{4:4} \\
            &         &         &         &         & a_{5:5} \\
            &         &         &         &         & a_{6:5} & a_{6:6} \\
            &         &         &         &         &         & & a_{7:7} \\
  \end{bmatrix}
  \begin{bmatrix}
    a_{0:0}  & \\
             & a_{1:1} & \\
             &         & a_{2:2} \\
             & a_{3:1} & & a_{3:3} \\
             &         & &         & a_{4:4} \\
             & a_{5:1} & & a_{5:3} & & a_{5:5} \\
             &         & &         & &         & a_{6:6} \\
             & a_{7:1} & & a_{7:3} & & a_{7:5} & & a_{7:7} \\
  \end{bmatrix}
  \begin{bmatrix}
    a_{0:0} & \\
    a_{1:0} & a_{1:1} & \\
            &         & a_{2:2} \\
            &         & a_{3:2} & a_{3:3} \\
            &         &         &         & a_{4:4} \\
            &         &         &         & a_{5:4} & a_{5:5} \\
            &         &         &         &         &         & a_{6:6} \\
            &         &         &         &         &         & a_{7:6} & a_{7:7} \\
  \end{bmatrix}
  \endgroup
\end{align*}
\normalsize

Note that Stage 1 and Stage 3 require $O(T)$ work, while Stage 2 reduces to a self-similar problem of half the size.
It is easy to check that this requires $O(T)$ total work and $O(\log T)$ depth/span.

\begin{remark}
  In fact, it is possible to see that the computation graph of this algorithm is identical to that of the associative scan algorithm described in \cref{sec:scan:classical:parallel}.
  The key takeaway is that instead of the steps of (1) recognizing that $M$ defines a recurrence (2) observing that the recurrence can be defined with an associative binary operator;
  there is a completely different perspective of simply finding a structured matrix decomposition algorithm for $M$.
\end{remark}

%% file: structure/theory_details.tex
\section{Theory Details}

\subsection{Extras: Closure Properties of SSMs}
\label{sec:ssm:properties}

We present here some additional properties of semiseparable matrices to illustrate their flexibility and utility.
This section is not necessary to understand our core results.

\begin{proposition}[Semiseparable Closure Properties]
  \label{prop:ss-closure}
  Semiseparable matrices are closed under several primitive operations.
  \begin{itemize}
    \item \textbf{Addition}: The sum of an $N$-SS and $P$-SS matrix is at most ($N+P$)-SS.
    \item \textbf{Multiplication}: The product of an $N$-SS and $P$-SS matrix is ($N+P$)-SS.
    \item \textbf{Inverse}: The inverse of an $N$-SS matrix is at most $(N+1)$-SS.
  \end{itemize}
\end{proposition}
The addition and multiplication properties are easily seen.
The inverse property has many proofs; one approach follows immediately from the Woodbury inversion identity, which has also featured prominently in the structured SSM literature~\citep{gu2022efficiently}.

In turn, these imply closure properties of state space models.

For example, the addition property says that summing two parallel SSM models is still an SSM.
The multiplication property says that sequentially composing or chaining two SSMs can still be viewed as an SSM, whose total state size is additive--a somewhat nontrivial fact.

Finally, the inverse property can let us relate SSMs to other types of models. For example,
one can notice that banded matrices are semiseparable, so their inverses are semiseparable.
(In fact, the semiseparable family of structure is often motivated by taking inverses of banded matrices~\citep{vandebril2005bibliography}).
Moreover, the fast recurrence properties of semiseparable matrices can be viewed as a consequence of their inverse being banded.

\begin{remark}
  The fact that 1-SS matrices are simple recurrences \eqref{eq:1ss-recurrence} are equivalent to the fact that the inverse of a 1-SS matrix is a 2-banded matrix:
\begin{align*}
  M =
  \begin{bmatrix}
    1 & \\
    a_1 & 1 & \\
    a_2a_1 & a_2 & 1 \\
    \vdots & \vdots & \ddots & \ddots \\
    a_{T-1}\dots a_1 & a_{T-1}\dots a_2 & \dots & a_{T-1} & 1 \\
  \end{bmatrix}
  =
  \begin{bmatrix}
    1 & \\
    -a_1 & 1 & \\
    0 & -a_2 & 1 \\
    \vdots & \vdots & \ddots & \ddots \\
    0 & 0 & \dots & -a_{T-1} & 1 \\
  \end{bmatrix}^{-1}
\end{align*}
Thus $y = Mx \leftrightarrow M^{-1}y = x$, or
\begin{align*}
  \begin{bmatrix}
    1 & \\
    -a_1 & 1 & \\
    0 & -a_2 & 1 \\
    \vdots & \vdots & \ddots & \ddots \\
    0 & 0 & \dots & -a_{T-1} & 1 \\
  \end{bmatrix}
  y = x
  .
\end{align*}
Or elementwise,
\begin{align*}
  & y_t - a_t y_{t-1} = x_t \\
  & y_t = a_t y_{t-1} + x_t
  .
\end{align*}
\end{remark}
Conversely,
we also use these closure results to prove that autoregressive structured attention (under certain assumptions)
must be SSMs,
allowing us to show that more general families of efficient sequence models including attention variants can be reduced to state space models (\cref{sec:theory-details:ssm-sma}).

\subsection{Autoregressive Masked Attention is Semiseparable-Structured Attention}
\label{sec:theory-details:ssm-sma}

We prove \cref{thm:ss-sma} from \cref{sec:ssd:1ss-sma}.
In \cref{sec:structured-attention} we defined structured attention as a broad generalization of masked attention,
where the property of efficiency (i.e.\ a linear-time form for the kernel attention) is abstracted into the efficiency of structured matrix multiplication.
However, beyond computational efficiency, standard linear attention~\citep{katharopoulos2020transformers} also has two important properties.
First, it is \emph{causal}, which is required for settings such as autoregressive modeling.
Moreover, it has \emph{efficient autoregressive generation}.
In other words, the cost of an autoregressive step -- i.e.\ the incremental cost of computing the output $y_T$ upon seeing $x_T$, given that $x_{0:T}$ has already been seen and preprocessed --
requires only constant time.

Here we characterize which instances of SMA have efficient autoregression.

In the framework of SMA,
causality is equivalent to the constraint that the mask $L$ is a \emph{lower-triangular} matrix.

Characterizing the space of $L$ matrices that have efficient autoregression is more difficult.
We will use a narrow technical definition of autoregressive processes, in the spirit of classical definitions from the time series literature (e.g.\ ARIMA processes~\citep{box2015time}).%
\begin{definition}
We define an autoregressive transformation $x \in \R^T \mapsto y \in \R^T$ of order $k$ as one where each output $y_t$ depends only on the current input and last $k$ outputs:
\begin{equation}
  \label{eq:efficient-ar}
  y_t = \mu_t x_t + \ell_{t1} y_{t-1} + \dots + \ell_{tk} y_{t-k}.
\end{equation}
\end{definition}
Note that the case where $L$ is the cumsum matrix is a special case with $k=1$ and thus $y_t = x_t + y_{t_1}$.
With this definition, characterizing the space of efficient autoregressive linear transforms follows from the properties of semiseparable matrices.
\cref{thm:ss-sma-formal} formalizes and proves \cref{thm:ss-sma}.

\begin{theorem}
  \label{thm:ss-sma-formal}
  Let $L \in \R^{T \times T}$ be an efficient autoregressive transformation of order $k$.
  Then $L$ is a state space model of order $k+1$.
\end{theorem}
\begin{proof}
  Let $(x, y)$ be input and output sequences, so that $y = Lx$.
  Rearranging the definition \eqref{eq:efficient-ar},
  \begin{align*}
    y_t - \ell_{t1} y_{t-1} - \dots - \ell_{tk} y_{t-k} = \mu_t x_t
    .
  \end{align*}
  Vectorizing over $t$, this can be expressed as a matrix transformation
  \begin{align*}%
    \begin{bmatrix}
      1 & \\
      -\ell_{t1} & 1 \\
      \vdots & \ddots & \ddots \\
      -\ell_{tk} & \dots & -\ell_{t1} & 1 \\
      \vdots & \ddots & \vdots & \ddots & \ddots \\
      0 & \dots & -\ell_{T-1,k} & \dots & -\ell_{T-1,1} & 1 \\
    \end{bmatrix}
    \begin{bmatrix}
      y_0 \\
      y_1 \\
      \vdots \\
      y_k \\
      \vdots \\
      y_{T-1} \\
    \end{bmatrix}
    =
    \begin{bmatrix}
      \mu_0 \\
      & \mu_1 \\
      & & \ddots \\
      & & & \mu_k \\
      & & & & \ddots \\
      & & & & & \mu_{T-1} \\
    \end{bmatrix}
    \begin{bmatrix}
      x_0 \\
      x_1 \\
      \vdots \\
      x_k \\
      \vdots \\
      x_{T-1} \\
    \end{bmatrix}
    .
  \end{align*}
  The $\mu$ diagonal matrix can be moved to the left and folded into the matrix of $\ell$ coefficients,
  which remains a $k+1$-band lower-triangular matrix.
  But we also have $L^{-1} y = x$, so $L$ is the inverse of this matrix.

  Next, note that $k+1$-band matrices are $k+1$-semiseparable by the rank characterization of semiseparability (\cref{def:semiseparable-rank}).
  By \cref{prop:ss-closure}, the inverse $L$ is therefore at most $k+2$-semiseparable.
  A slightly stronger bound of $k+1$ can be obtained because of the additional structure of banded matrices.
  Finally, the characterization of $L$ as an order-$k+1$ state space model follows from \cref{thm:ssm-sss}.
\end{proof}

In other words, efficient autoregressive attention is \textbf{semiseparable SMA}.

%% file: structure/experiment_details.tex
\section{Experimental Details}
\label{sec:exp-details}

\subsection{MQAR Details}
\label{sec:exp-details:mqar}

We use a harder version of the task introduced in Based \citep{arora2024simple} where tokens that are not query/key/values are replaced with random tokens.
We also use more key-value pairs, longer sequences, and smaller model sizes than the usual variant of MQAR used by prior work, all of which make the task harder.

For each sequence length $T \in \{256, 512, 1024\}$, we use $T/4$ key-value pairs. %
The total vocab size is 8192.

We use a form of curriculum training where training cycles through datasets using $(T/32, T/16, T/8, T/4)$ key-value pairs, where each dataset has $2^{18} \approx 250000$ examples, for a total of $8$ epochs through each dataset (total of $2^{28} \approx 270M$ examples).
The total batch size is $2^{18} \approx 0.25M$ tokens (e.g. for $T=1024$, the batch size is $256$).

All methods use $2$ layers with default settings; the attention baseline additionally receives positional embeddings.
For each method, we sweep over model dimensions $\mathtt{D} = \{32, 64, 128, 256\}$ and learning rates $\{10^{-3.5}, 10^{-2}, 10^{-2.5}\}$.
We use a linear decay schedule that drops on every epoch (e.g. the last epoch would have a learning rate $1/8$ of the maximum/starting learning rate).

\subsection{Scaling Law Details}
\label{sec:exp-details:lm:scaling}

All models were trained on the Pile.
For the scaling law experiments, we use the GPT2 tokenizer.

\paragraph{Model Sizes.}

\cref{tab:gpt3} specifies the model sizes we use for scaling laws following GPT3~\citep{brown2020language},
First, we changed the batch size of the 1.3B model from 1M tokens to 0.5M tokens
for uniformity.
Second, we changed the number of training steps and total tokens to roughly match Chinchilla scaling laws~\citep{hoffmann2022empirical}, which specify that training tokens should increase proportionally to model size.

\begin{table}
  \caption{
    (\textbf{Scaling Law Model Sizes}.)
    Our model sizes and hyperparameters for scaling experiments.
    (Model dimension and number of heads applies only to Transformer models.)
  }
  \centering
  \small
  \begin{tabular}{@{}llllllll@{}}
    \toprule
    \sc{Params} & $\mathtt{n\_layers}$ & $\mathtt{d\_model}$ & $\mathtt{n\_heads}$ / $\mathtt{d\_head}$ & \sc{Training steps}  & \sc{Learning Rate}  & \sc{Batch Size} & \sc{Tokens} \\
    \midrule
    125M   & 12                   & 768                 & 12 / 64                                  & 4800            & 6e-4           & 0.5M tokens & 2.5B       \\
    350M   & 24                   & 1024                & 16 / 64                                  & 13500           & 3e-4           & 0.5M tokens & 7B         \\
    760M   & 24                   & 1536                & 16 / 96                                  & 29000           & 2.5e-4         & 0.5M tokens & 15B       \\
    1.3B   & 24                   & 2048                & 32 / 64                                  & 50000           & 2e-4           & 0.5M tokens & 26B       \\
    \bottomrule
  \end{tabular}
  \label{tab:gpt3}
\end{table}

\paragraph{Training Recipes.}

All models used the AdamW optimizer with
\begin{itemize}
  \item gradient clip value $1.0$
  \item weight decay $0.1$
  \item no dropout
  \item linear learning rate warmup with cosine decay
\end{itemize}
By default, the peak learning rate is the GPT3 specification.

Compared to GPT3 recipe, we use an ``improved recipe'', inspired by changes adopted by popular large language models such as PaLM~\citep{chowdhery2022palm} and LLaMa~\citep{touvron2023llama}.
These include:
\begin{itemize}
  \item linear learning rate warmup with cosine decay to $1e-5$, with a peak value of $5\times$ the GPT3 value
  \item no linear bias terms
  \item RMSNorm instead of LayerNorm
  \item AdamW hyperparameter $\beta=(.9, .95)$ (the GPT3 value) instead of the PyTorch default of $\beta=(.9, .999)$
\end{itemize}

\subsection{Downstream Evaluation Details}

To evaluate downstream performance of fully trained, we train Mamba-2 on 300B tokens on the Pile, using the GPT-NeoX~\citep{black2022gpt} tokenizer.

We use the same hyperparameters as the scaling experiments, except with batch
size 1M for the 1.3B and 2.7B model.
For the 2.7B model, we also follow GPT3 specification (32 layers, dimension 2560).

For all models, we use 5x the learning rate of the corresponding GPT3 model.

For downstream evaluation, we use the LM evaluation harness from EleutherAI~\citep{eval-harness}, on the same tasks as Mamba~\citep{gu2023mamba} with one additional one:
\begin{itemize}
  \item LAMBADA~\citep{paperno2016lambada}
  \item HellaSwag~\citep{zellers2019hellaswag}
  \item PIQA~\citep{bisk2020piqa}
  \item ARC-challenge~\citep{clark2018think}
  \item ARC-easy: an easy subset of ARC-challenge
  \item WinoGrande~\citep{sakaguchi2021winogrande}
  \item OpenBookQA~\citep{mihaylov2018can}
\end{itemize}

\begin{table*}[ht]
  \small
  \centering
  \captionsetup{font=small}
  \caption{
    (\textbf{Zero-shot Evaluations}.) Best results for each size in bold, second best unlined.
    We compare against open source LMs with various tokenizers, trained for up to 300B tokens.
    Pile refers to the validation split, comparing only against models trained on the same dataset and tokenizer (GPT-NeoX-20B).
    For each model size, Mamba-2 outperforms Mamba, and generally matches Pythia at twice the model size.
  }
  \resizebox{0.99\linewidth}{!}
  {
    \begin{tabular}{@{}llllllllllll@{}}
      \toprule
      \sc{Model}            & \sc{Token.} & \sc{Pile}                  & \sc{LAMBADA}               & \sc{LAMBADA}              & \sc{HellaSwag}             & \sc{PIQA}                 & \sc{Arc-E}                & \sc{Arc-C}                            & \sc{WinoGrande}                         & \sc{OpenbookQA}                         & \sc{Average} \\
                            &             & \sc{ppl $\downarrow$}      & \sc{ppl $\downarrow$}      & \sc{acc $\uparrow$}       & \sc{acc $\uparrow$}        & \sc{acc $\uparrow$}       & \sc{acc $\uparrow$}       & \sc{acc $\uparrow$}                   & \sc{acc $\uparrow$}                     & \sc{acc $\uparrow$}                     & \sc{acc $\uparrow$} \\
                                         \midrule
      Hybrid H3-130M        & GPT2        & ---                        & 89.48                      & 25.8                      & 31.7                       & 64.2                      & 44.4                      & \textbf{24.2}                         & 50.6                                    & 27.0                                    & 38.2 \\
      Pythia-160M           & NeoX        & 29.64                      & 38.10                      & 33.0                      & 30.2                       & 61.4                      & 43.2                      & 24.1                                  & \underline{51.9}                        & \underline{29.2}                        & 39.0 \\
      Mamba-130M            & NeoX        & \underline{10.56}          & \textbf{16.07}             & \textbf{44.3}             & \underline{35.2}           & \underline{64.5}          & \textbf{48.0}             & \textbf{24.2}                         & \underline{51.9}                        & 28.8                                    & \underline{42.4} \\
      \textbf{Mamba-2-130M} & NeoX        & \textbf{10.48}             & \underline{16.86}          & \underline{43.9}          & \textbf{35.3}              & \textbf{64.9}             & \underline{47.4}          & \textbf{24.2}                         & \textbf{52.1}                           & \textbf{30.6}                           & \textbf{42.6} \\
      \midrule
      Hybrid H3-360M        & GPT2        & ---                        & 12.58                      & 48.0                      & 41.5                       & 68.1                      & 51.4                      & 24.7                                  & 54.1                                    & \underline{31.6}                        & 45.6 \\
      Pythia-410M           & NeoX        & 9.95                       & 10.84                      & 51.4                      & 40.6                       & 66.9                      & 52.1                      & 24.6                                  & 53.8                                    & 30.0                                    & 45.6 \\
      Mamba-370M            & NeoX        & \underline{8.28}           & \underline{8.14}           & \underline{55.6}          & \underline{46.5}           & \underline{69.5}          & \textbf{55.1}             & \textbf{28.0}                         & \underline{55.3}                        & 30.8                                    & \underline{48.7} \\
      \textbf{Mamba-2-370M} & NeoX        & \textbf{8.21}              & \textbf{8.02}              & \textbf{55.8}             & \textbf{46.9}              & \textbf{70.5}             & \underline{54.9}          & \underline{26.9}                      & \textbf{55.7}                           & \textbf{32.4}                           & \textbf{49.0} \\
      \midrule
      Pythia-1B             & NeoX        & 7.82                       & 7.92                       & 56.1                      & 47.2                       & 70.7                      & 57.0                      & 27.1                                  & 53.5                                    & 31.4                                    & 49.0 \\
      Mamba-790M            & NeoX        & \underline{7.33}           & \underline{6.02}           & \textbf{62.7}             & \textbf{55.1}              & \textbf{72.1}             & \textbf{61.2}             & \textbf{29.5}                         & \underline{56.1}                        & \underline{34.2}                        & \underline{53.0} \\
      \textbf{Mamba-2-780M} & NeoX        & \textbf{7.26}              & \textbf{5.86}              & \underline{61.7}          & \underline{54.9}           & \underline{72.0}          & \underline{61.0}          & \underline{28.5}                      & \textbf{60.2}                           & \textbf{36.2}                           & \textbf{53.5} \\
      \midrule
      GPT-Neo 1.3B          & GPT2        & ---                        & 7.50                       & 57.2                      & 48.9                       & 71.1                      & 56.2                      & 25.9                                  & 54.9                                    & 33.6                                    & 49.7 \\
      Hybrid H3-1.3B        & GPT2        & ---                        & 11.25                      & 49.6                      & 52.6                       & 71.3                      & 59.2                      & 28.1                                  & 56.9                                    & 34.4                                    & 50.3 \\
      OPT-1.3B              & OPT         & ---                        & 6.64                       & 58.0                      & 53.7                       & 72.4                      & 56.7                      & 29.6                                  & 59.5                                    & 33.2                                    & 51.9 \\
      Pythia-1.4B           & NeoX        & 7.51                       & 6.08                       & 61.7                      & 52.1                       & 71.0                      & 60.5                      & 28.5                                  & 57.2                                    & 30.8                                    & 51.7 \\
      RWKV4-1.5B            & NeoX        & 7.70                       & 7.04                       & 56.4                      & 52.5                       & 72.4                      & 60.5                      & 29.4                                  & 54.6                                    & 34.0                                    & 51.4 \\
      Mamba-1.4B            & NeoX        & \underline{6.80}           & \underline{5.04}           & \underline{65.0}          & \underline{59.1}           & \textbf{74.2}             & \textbf{65.5}             & \underline{32.8}                      & \textbf{61.5}                           & \underline{36.4}                        & \textbf{56.4} \\
      \textbf{Mamba-2-1.3B} & NeoX        & \textbf{6.66}              & \textbf{5.02}              & \textbf{65.7}             & \textbf{59.9}              & \underline{73.2}          & \underline{64.3}          & \textbf{33.3}                         & \underline{60.9}                        & \textbf{37.8}                           & \textbf{56.4} \\
      \midrule
      GPT-Neo 2.7B          & GPT2        & ---                        & 5.63                       & 62.2                      & 55.8                       & 72.1                      & 61.1                      & 30.2                                  & 57.6                                    & 33.2                                    & 53.2 \\
      Hybrid H3-2.7B        & GPT2        & ---                        & 7.92                       & 55.7                      & 59.7                       & 73.3                      & 65.6                      & 32.3                                  & 61.4                                    & 33.6                                    & 54.5 \\
      OPT-2.7B              & OPT         & ---                        & 5.12                       & 63.6                      & 60.6                       & 74.8                      & 60.8                      & 31.3                                  & 61.0                                    & 35.2                                    & 55.3 \\
      Pythia-2.8B           & NeoX        & 6.73                       & 5.04                       & 64.7                      & 59.3                       & 74.0                      & 64.1                      & 32.9                                  & 59.7                                    & 35.2                                    & 55.7 \\
      RWKV4-3B              & NeoX        & 7.00                       & 5.24                       & 63.9                      & 59.6                       & 73.7                      & 67.8                      & 33.1                                  & 59.6                                    & 37.0                                    & 56.4 \\
      Mamba-2.8B            & NeoX        & \underline{6.22}           & \underline{4.23}           & \underline{69.2}          & \underline{66.1}           & \underline{75.2}          & \textbf{69.7}             & \underline{36.3}                      & \underline{63.5}                        & \textbf{39.6}                           & \underline{59.9} \\
      \textbf{Mamba-2-2.7B} & NeoX        & \textbf{6.09}              & \textbf{4.10}              & \textbf{69.7}             & \textbf{66.6}              & \textbf{76.4}             & \underline{69.6}          & \textbf{36.4}                         & \textbf{64.0}                           & \underline{38.8}                        & \textbf{60.2} \\
      \midrule
      GPT-J-6B              & GPT2        & --                         & 4.10                       & 68.3                      & 66.3                       & 75.4                      & 67.0                      & 36.6                                  & 64.1                                    & 38.2                                    & 59.4 \\
      OPT-6.7B              & OPT         & --                         & 4.25                       & 67.7                      & 67.2                       & 76.3                      & 65.6                      & 34.9                                  & 65.5                                    & 37.4                                    & 59.2 \\
      Pythia-6.9B           & NeoX        & 6.51                       & 4.45                       & 67.1                      & 64.0                       & 75.2                      & 67.3                      & 35.5                                  & 61.3                                    & 38.0                                    & 58.3 \\
      RWKV4-7.4B            & NeoX        & 6.31                       & 4.38                       & 67.2                      & 65.5                       & 76.1                      & 67.8                      & 37.5                                  & 61.0                                    & 40.2                                    & 59.3 \\
      \bottomrule
    \end{tabular}
  }
  \label{table:downstream_zeroshot_full}
\end{table*}

\iftoggle{arxiv}{
\subsection{Ablation Details}

\paragraph{(Re)Based Details.}
Our ablations in \cref{sec:experiments:ablations:kernels} considered the Based~\citep{arora2024simple} and ReBased~\citep{aksenov2024linear} models.

Based approximates the $\exp$ kernel with a quadratic Taylor expansion $\exp(x) \approx 1 + x + x^2/2$,
which can be accomplished by the feature map
$$\psi_{\text{Taylor}}(x) = \mathsf{concatenate}(1, x, 1/\sqrt{2} x \otimes x).$$
ReBased proposes to use the simpler feature map $\psi_{\text{Quadratic}}(x) = x \otimes x$ corresponding to the kernel transformation $x^2$, but also applies a layer normalization beforehand.
We view the layer normalization as an alternative non-linear activation to our default Swish activation,
and ablate combinations of these.

Note that because these expand the feature dimension, we must project to smaller $B, C$ dimensions;
in \cref{tab:ablations-kernel-based}, use state size $N=64$ for 130M models and $N=256$ for 380M models.
For the (Re)Based methods, we project to 8 and 16 dimensions respectively before applying their feature maps;
this results in a total state size of $8^2 = 64$ for ReBased and $1+8+8^2=73$ for Based in the 130M model case.
Because the $B$ and $C$ projections are smaller, these methods use fewer parameters, and we adjust the layer count appropriately.
}{}